\documentclass[twoside]{article}

\usepackage[accepted]{aistats2022}
\usepackage[round]{natbib}

\usepackage[utf8]{inputenc} % allow utf-8 input
\usepackage[T1]{fontenc}    % use 8-bit T1 fonts
\usepackage{url}            % simple URL typesetting
\usepackage{booktabs}       % professional-quality tables
\usepackage{amsfonts}       % blackboard math symbols
\usepackage{nicefrac}       % compact symbols for 1/2, etc.
\usepackage{microtype}      % microtypography
\usepackage{xcolor}         % colors

\usepackage[normalem]{ulem}

\usepackage{graphicx}
\usepackage{subfigure}
\usepackage{url}
\usepackage{multicol}
\usepackage{multirow}
\usepackage{wrapfig}
\usepackage{float}
\usepackage{tabularx}
\usepackage{caption}
\usepackage{enumitem}
\usepackage{soul}
\usepackage{amsmath,amsfonts,bm}

\def\eqref#1{equation~\ref{#1}}
\def\1{\bm{1}}

\DeclareMathAlphabet{\mathsfit}{\encodingdefault}{\sfdefault}{m}{sl}
\SetMathAlphabet{\mathsfit}{bold}{\encodingdefault}{\sfdefault}{bx}{n}

\newcommand{\E}{\mathbb{E}}

\newcommand{\R}{\mathbb{R}}

\newcommand{\bx}{{\bf x}}
\newcommand{\bw}{{\bf w}}
\newcommand{\bu}{{\bf u}}
\newcommand{\buh}{\hat{\bf u}}
\newcommand{\bwh}{\hat{\bf w}}
\newcommand{\by}{{\bf y}}

\newcommand{\bU}{{\bf U}}

\newcommand{\bV}{{\bf V}}
\newcommand{\bv}{{\bf v}}

\newcommand{\Uh}{\hat{U}}

\newcommand{\uh}{\hat{u}}

\newcommand{\bUh}{{\bf \hat{U}}}

\newcommand{\cC}{\mathcal{C}}
\newcommand{\cU}{\mathcal{U}}
\newcommand{\cUh}{\hat{\mathcal{U}}}

\usepackage{array}
\newcolumntype{P}[1]{>{\centering\arraybackslash}p{#1}}
\newcommand{\centered}[1]{\begin{tabular}{c} #1 \end{tabular}}
\usepackage{hyperref}
\usepackage{algorithmic}
\usepackage{algorithm}% http://ctan.org/pkg/algorithms

\usepackage{amsthm}
\newtheorem{thm}{Theorem}

\newtheorem{lem}{Lemma}
\newtheorem{rem}{Remark}

\newtheorem{definition}{Definition}

\iftrue
\newcommand{\berivan}[1]{{\color{red}\textbf{BI}: #1}}
\newcommand{\noa}[1]{{\color{green!50!black}\textbf{AN}: #1}}
\newcommand{\wip}[1]{{\color{blue} #1}}

\fi

\begin{document}



\twocolumn[

\aistatstitle{An Information-Theoretic Justification for Model Pruning}

\aistatsauthor{ Berivan Isik \And Tsachy Weissman \And  Albert No }

\aistatsaddress{Stanford University \\ berivan.isik@stanford.edu \And Stanford University \\ tsachy@stanford.edu \And Hongik University \\ albertno@hongik.ac.kr } ]

\begin{abstract}
    \label{abstract}
We study the neural network (NN) compression problem, viewing the tension between the compression ratio and NN performance through the lens of rate-distortion theory. We choose a distortion metric that reflects the effect of NN compression on the model output and derive the tradeoff between rate (compression) and distortion. In addition to characterizing theoretical limits of NN compression, this formulation shows that \emph{pruning}, implicitly or explicitly, must be a part of a good compression algorithm. This observation bridges a gap between parts of the literature pertaining to NN and data compression, respectively, providing insight into the empirical success of model pruning. Finally, we propose a novel pruning strategy derived from our information-theoretic formulation and show that it outperforms the relevant baselines on CIFAR-10 and ImageNet datasets.

\end{abstract}
\section{Introduction}
\label{intro}

The recent success of NNs in various machine learning applications has come with their over-parameterization. Deployment of such over-parameterized models on edge devices is challenging as these devices have limited storage, computation, and power resources. Motivated by this, there has been significant interest in NN compression by the research community. The most established NN compression techniques can be broadly grouped into five categories: quantization \citep{  quant_1, scalableQuant, jacob2018quantization, jung2019learning,  quant_3, choi2020universal, young2020transform, idelbayev2021optimal} and coding \citep{DeepCABAC, zhe2021rate} of NN parameters, pruning \citep{deep_compression, molchanov2016pruning, carreira2018learning, liu2018rethinking, yu2018nisp,   lin2019towards, peng2019collaborative, xiao2019autoprune, zhao2019variational, blalock2020state,   elsen2020fast, park2020lookahead,  renda2020comparing}, Bayesian compression \citep{federici2017improved, louizos2017bayesian, louizos2017learning, molchanov2017variational, compress_var_info}, distillation \citep{distillation, polino2018model, wang2019private}, and low-rank matrix factorization \citep{low_rank1, low_rank2, Idelbayev_low_rank3}. The success of these techniques in compressing NN models without a significant performance loss brings a theoretical question: \emph{what is the fundamental limit of NN compression while maintaining a target performance?}

%highest achievable compression ratio given a target performance for the compressed model?} 
%\noa{instead of using the term ``compression ratio'' explicitly, how about
%``what is the fundamental limit of NN compression while maintaining a target performance?''}

\looseness=-1
A similar question arises in the classical data compresssion problem as well \citep{salomon2004data}.
\cite{shannon2001mathematical} introduced the mathematical formulation of the data compression problem, where the goal is to describe a source sequence with the minimum number of bits. In an information-theoretic sense, entropy is the limit of how much a source sequence can be losslessly compressed.
However, \emph{in practice}, there are many sources such as image, video, and audio, where lossless compression cannot achieve a high enough compression rate.  
In such cases, we need to compress the source sequence in a \emph{lossy} manner allowing some \emph{distortion} between the source and reconstruction. This is where rate-distortion theory comes into the picture. For lossy compression, rate-distortion theory gives the limit of how much a source sequence can be compressed without exceeding a target distortion level \citep{berger2003rate}.

\looseness=-1
In this work, we connect these two lines of research and study the theoretical limits of lossy NN compression via rate-distortion theory. In particular, we consider a classical lossy compression problem to compress NN weights while minimizing the perturbation in the NN output space. We first (1) define a distortion metric that upper bounds the output perturbation due to compression, then (2) find a probability distribution that fits NN parameters, and finally (3) derive the rate-distortion function for the chosen distortion metric and distribution. This function describes the theoretical tradeoff between rate (compression ratio) and NN output perturbation, thus provides insight into how compressible NN models are. Furthermore, our findings indicate that the compressed model that reaches the optimal achievable compression ratio must be sparse. This suggests that a good NN compression algorithm must, implicitly or explicitly, involve a pruning step,
complimenting the empirical success of pruning strategies \citep{stateofsparsity}. Therefore, we provide theoretical support for pruning as a rate-distortion theoretic compression scheme that maintains the model output. 

\looseness=-1
Inspired by this observation, we propose a practical lossy compression algorithm for NN models. The reconstruction of our algorithm is a sparse model, which naturally induces a novel pruning strategy. Our algorithm is based on \emph{successive refinability} -- a property that often helps to reduce the complexity of lossy compression algorithms \citep{equitz1991successive}. Our strategy differs from previous score-based pruning methods as it relies solely on an information-theoretic approach to a data compression problem with additional practical benefits that we cover in Section~\ref{experiments}. We also prove that the proposed algorithm is sound from a rate-distortion theoretic perspective. We demonstrate the efficacy of our pruning strategy on CIFAR-10 and ImageNet datasets. Lastly, we show that our strategy provides a tool for compressing NN gradients as well,
an important objective in communication-efficient federated learning (FL) settings \citep{kairouz2019advances}.
The contributions of our paper can be summarized as: 
\looseness=-1
\begin{itemize}
    \item We take a step in bridging the gap between NN compression  
    and data compression. 
    \item We present the rate-distortion theoretical limit of achievable NN compression given a target distortion level and show that pruning is an essential part of a good compression algorithm.
    \item We propose a novel pruning strategy derived from our findings, which outperforms relevant baselines.% in the literature. 
\end{itemize}
\looseness=-1

\section{Related Work}
\label{related}

This section is devoted to prior work on NN compression that has the same flavor as ours, in particular, we touch on (a) data compression approaches to NN compression and (b) pruning. We cover related works in classical data compression as we go through the methodology in Sections~\ref{prelims}, \ref{method1_rate_distortion}, and \ref{method2_SuRP}.

\looseness=-1
\paragraph{From Data Compression to NN Compression.}

To date, several works have proposed to minimize the bit-rate (compressed size) of NNs with quantization techniques \citep{  quant_3, idelbayev2021optimal, QuantNoise}.
Some recent work has shown promising results to go beyond quantization using tools from data compression.
For instance, \cite{havasi2018minimal} and \cite{oktay2019scalable} have trained a model to jointly optimize compression and performance of the model using tools from minimum description length principle \citep{grunwald2007minimum} and a recently advanced image compression framework \citep{balle2016end}, respectively. 
%have used tools from minimum description length principle \cite{grunwald2007minimum} and a recently advanced image compression framework \cite{balle2016end}, respectively, to train compressible NN models. 
While we share the same goal with these papers, our focus is on compressing NN models \emph{post-training}. With this distinction, our work is most related to \citep{gao2019rate}, where the authors have put the first attempt to approach NN compression from a rate-distortion theoretic perspective. Although they have shown achievability results on one-layer networks, their results do not generalize to deeper networks without first-order Taylor approximations. Moreover, their formulation relies on the assumption that NN weights follow Gaussian distribution, which currently lacks empirical evidence. On the other hand, we show achievable compression ratios generalized to multi-layer networks without making linear approximations and provide strong empirical evidence for our choice of \emph{Laplacian} distribution for %\noa{(normalized)}
NN weights. 

\looseness=-1
\paragraph{Pruning.} The overparameterized nature of NNs has motivated researchers to explore ways to find and remove redundant parameters \citep{OBD, OBS}. The idea of iterative magnitude pruning was shown to be remarkably successful in deep NNs first by \cite{deep_compression}, and since then, NN pruning research has accelerated. To improve upon the iterative magnitude pruning scheme of \citep{deep_compression}, researchers have looked for different ways to adjust the pruning ratios across layers. For instance, \cite{zhu2017prune} have suggested pruning the parameters uniformly across layers. \cite{stateofsparsity}, on the other hand, have shown better results when the first convolutional layer is excluded from the pruning and the last fully-connected layer is not pruned more than $80 \%$. Layerwise pruning ratio has also been investigated for NNs pruned at initialization since the explosion of the Lottery Ticket Hypothesis \citep{frankle2018lottery, morcos2019one}. \cite{evci2020rigging} have shown promising results on NNs pruned at initialization where the pruning ratio across layers is adjusted by Erd\H{o}s-R\'enyi kernel method, as introduced by \cite{mocanu2018scalable}. More recently, \cite{lee2020deeper} have proposed adjusting the pruning threshold for each layer based on the norm of the weights at that layer.
We follow a similar methodology in \citep{lee2020deeper} to normalize the parameters prior to applying our \emph{novel} pruning algorithm. Unlike other pruning strategies, our algorithm outputs a pruned (sparse) model,
without an explicit score-based pruning step. Instead, our reconstruction goes from the coarsest (sparsest) to the finest representation of the model. %\noa{we can remove this} %This is similar to the progressive or hierarchical image compression techniques such as JPEG \cite{mallat2009theory, lewis1992image, rabbani2002jpeg2000}. %we develop the algorithm from a theoretical formulation of the NN compression problem,
%where our reconstruction goes from the coarsest (sparsest) to the finest representation of the model.
Parallel to our work, a recent study has proposed a heuristic bottom-up approach as opposed to the common top-down pruning approach
and provided promising empirical results \citep{chen2021long}.
To the best of our knowledge, our work is the first to provide a rate-distortion theoretic justification for pruning. 
\looseness=-1

\section{Preliminaries}
\label{prelims}
In this section, we present the problem setup and briefly introduce the rate-distortion theory and the successive refinement concept.

\looseness=-1
\subsection{Problem Statement} 
We study a NN compression problem where the network $\by = f(\bx;\bw)$ characterizes a prediction from the input space $\mathcal{X}$ to the output space $\mathcal{Y}$, parameterized by weights $\bw$. Our goal is to minimize the difference between $\by=f(\bx;\bw)$ and $\hat{\by}=f(\bx; \bwh)$, where $\bwh$ is a compressed version of the trained parameters $\bw$. In Section~\ref{distortion}, we define an appropriate distortion function $d(\bw, \bwh)$ that reflects the perturbation in the output space $\|f(\bx;\bw) - f(\bx;\bwh)\|_1$. This is a lossy compression problem where the distortion is a measure of the distance between the original model and the compressed model,
and the rate is the number of bits required to represent one weight.
In information-theoretic term, rate distortion theory characterizes the minimum achievable rate given the target distortion.

\looseness=-1
\subsection{Notation} 
Throughout the paper, $\bw \in \R^n$ is the weights of a trained model. Logarithms are natural logarithms.
Rate is defined as nats (the unit of information obtained from natural logarithm) per symbol (weight in our case). We use lower case $u$ to denote the realization of a scalar random variable $U$ and $\bu=u^n= (u_1,\dots, u_n)$ to denote the realization of a random vector $\bU = U^n = (U_1, \dots, U_n)$. We use the term ``perturbation'' for the change in the model output due to compression, whereas ``distortion'' $d(\bw, \bwh)$ refers to the change in the \emph{parameter} space. Lastly, $d(u^n, \uh^n) = \frac{1}{n}\sum_{i=1}^n d(u_i, \uh_i)$ is the regular extension of the distortion function for an $n$ dimensional vector.

%%%%%%%%%%%%%%%%%%%%%%%%%%%%%%%%%%%%%%%%%%%%%%%%%%
%% Rate Distortion
%%%%%%%%%%%%%%%%%%%%%%%%%%%%%%%%%%%%%%%%%%%%%%%%%%
\looseness=-1
\subsection{Rate-Distortion Theory} 
Let $U_1, \dots, U_n\in\cU$ be a source sequence generated by i.i.d.\ $\sim p(u)$ where $p(u)$ is a probability density function and $\cU=\R$.
The encoder $f_e: \cU^n \rightarrow  \{0,1\}^{nR}$ describes this sequence in $nR$ bits, where this binary representation is called a ``message'' $m$.
The decoder $f_d: \{0,1\}^{nR} \rightarrow \cUh^n$ reconstructs an estimate $\buh = \uh^{n} \in \cUh^n$ based on $m\in\{0,1\}^{nR}$
where $\cUh=\R$ as well. This process, summarized in Figure~\ref{fig:enc_dec_diagram}(a), is called lossy source coding. The number of bits per source symbol  ($\frac{nR}{n}=R$ in this case) and the ``distance''
$d(\bu, \buh) = d(u^n, \uh^n)=\frac{1}{n}\sum_{i=1}^nd(u_i,\uh_i)$ between $\bu$ and $\buh$ are named as rate and distortion, respectively.
Ideally, we would like to keep both rate and distortion low, but there is a tradeoff between these two quantities,
which is characterized by the rate-distortion function \citep{shannon2001mathematical, berger2003rate, elements_of_it} as:
\begin{equation}
\label{rd}
    R(D) = \min_{p(\uh|u): \E [d(u, \uh)] \leq D} I(U; \Uh)
\end{equation}
where $I(U;\Uh)$ is the mutual information between $U$ and $\Uh$,
and $d(\cdot, \cdot)$ is a predefined distortion metric, e.g. $\ell_2$ distance.
The rate-distortion function $R(D)$ in Eq.~\ref{rd} is the minimum achievable rate at distortion $D$,
and the conditional distribution $p(\uh|u)$ that achieves $I(U; \Uh) = R(D)$ explains
how an optimal encoder-decoder pair should operate for the source $p(u)$.
We can also define the inverse, namely the distortion-rate function $D(R)$, which is the minimum achievable distortion at rate $R$.
Clearly, source distribution has a critical role in the solution of the rate distortion problem.
We discuss possible assumptions for the distribution of NN weights in Section~\ref{RD_function}. 

\begin{figure}[h!]
        \centering 
        \subfigure[Lossy Source Coding.]{\includegraphics[width=.45\textwidth]{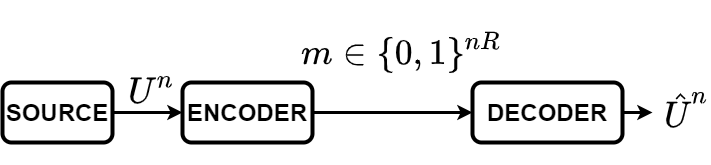}}
        \subfigure[Successive Refinement. ]{\includegraphics[width=.45\textwidth]{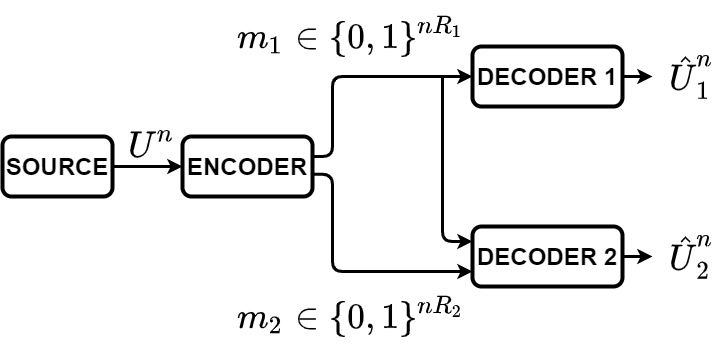}}
       \caption{(a) Source Coding, (b) Successive Refinement with $2$ Decoders.}
       \label{fig:enc_dec_diagram}
       \end{figure}

%%%%%%%%%%%%%%%%%%%%%%%%%%%%%%%%%%%%%%%%%%%%%%%%%%
%% Successive Refinement
%%%%%%%%%%%%%%%%%%%%%%%%%%%%%%%%%%%%%%%%%%%%%%%%%%
\looseness=-1
\subsection{Successive Refinement} 
In the successive refinement problem, summarized in Figure~\ref{fig:enc_dec_diagram}(b), the encoder wants to describe the source to two decoders, where each decoder has its own target distortion, $D_1$ and $D_2$.
Instead of having separate encoding schemes for each decoder,
the successive refinement encoder encodes a message $m_1$ for Decoder 1 (with higher target distortion, $D_1$),
and encodes an extra message $m_2$ where the second decoder gets both $m_1$ and $m_2$.
Receiving both $m_1$ and $m_2$, Decoder 2 reconstructs $\bUh_2$ with distortion $D_2$.
Since the message $m_1$ is re-used, the performance of successive refinement encoder is sub-optimal in general.
However, in some cases, the successive refinement encoder achieves the optimum rate-distortion tradeoff as if dedicated encoders were used separately.
In such a case, we call the source (distribution) and the distortion pair successively refinable \citep{koshelev1980hierarchical, equitz1991successive}.
In Section~\ref{sec:fixed_successive_refinement}, we discuss how to achieve low complexity via successive refinement.

\renewcommand{\Pr}[1]{\mbox{{Pr}}\left[#1\right]}

\section{Rate-Distortion Theory for Neural Network Parameters}
\label{method1_rate_distortion}
In this section, we first derive the distortion metric to be used in the rate-distortion function, then we estimate the source distribution (probability density of NN weights), and finally, we present the rate-distortion function associated with the chosen distortion metric and the source distribution. 
   
%%%%%%%%%%%%%%%%%%%%%%%%%%%%%%%%%%%%%%%%%%%%%%%%%%
%% Distortion Metric
%%%%%%%%%%%%%%%%%%%%%%%%%%%%%%%%%%%%%%%%%%%%%%%%%%
\looseness=-1
\subsection{Distortion Metric}
\label{distortion}
\looseness=-1
Our objective is to minimize the difference between the output of the original NN model and the compressed model.
Formally, we would like to keep the output perturbation $\|f(\bx;\bw) - f(\bx;\bwh)\|_1$ small.
Since the effect of a weight distortion on the output space $f(\bx;\bw)$ is intractable for deep NNs,
we seek to find a distortion function on parameter space that upper bounds $\|f(\bx;\bw) - f(\bx;\bwh)\|_1$.

Prior work has derived an upper bound for the $\ell_2$ norm of the output perturbation as the Frobenius norm of the difference between $\bw$ and $\bwh$
when only a single layer is compressed \citep{lee2020deeper}.
%More precisely, consider a NN model with $d$ layers.
More precisely, consider a fully connected NN model with $d$ layers and ReLU activation. Let $\bw$ be the weights of the original trained model and $\bwh$ be a compressed version of $\bw$ where $\bwh$ is the same with $\bw$ except in the $l$-th layer.
In such a case, i.e., when only a single layer is compressed, the output perturbation is bounded by
\vspace{-1.5mm}
\begin{align}
\begin{aligned}
&\sup_{\|\bx\|_2\leq 1} \|f(\bx;\bw) - f(\bx;\bwh)\|_2 \\
 & \leq \frac{\|\bw^{(l)} - \bwh^{(l)}\|_F}{\|\bw^{(l)}\|_F} \cdot \left(\prod_{k=1}^d \|\bw^{(k)}\|_F\right)
\label{eq:lamp bound}
\end{aligned}
\end{align}

where $\bw^{(l)}$ indicates the weights of the $l$-th layer. Inspired by Eq.~\ref{eq:lamp bound}, \cite{lee2020deeper} have introduced Layer-Adaptive Magnitude-based Pruning (LAMP) score $(\bw^{(l)}_i)^2/ \left (\sum_{j} ( \bw^{(l)}_j )^2 \right )$ to measure the importance of the weight $\bw^{(l)}_i$ for pruning. Notice that Eq.~\ref{eq:lamp bound} holds only when a single layer is pruned. 

In this work, we follow a similar strategy to relate the ``$\ell_1$ norm of perturbation on the output space'' to 
``$\ell_1$ norm of the weight distortion after compression'', but not limited to single-layer compression.

%\noa{the following theorem needs an additional assumption $\|\bwh^{(l)}\|_1\leq \|\bw^{(l)}\|_1$ for all $1\leq l\leq d$ which we mention later}
\begin{thm}\label{thm:l1bound}
Suppose $f(\cdot;\bw)$ is a fully-connected NN model with $d$ layers and 1-Lipschitz activations $\sigma(\cdot)$ such that $\sigma(0)=0$, e.g., ReLU. Let $\bwh$ be the reconstructed weights (after compression) where all layers are subject to compression. If $\|\bw^{(l)}\|_1 \geq \|\bwh^{(l)}\|_1$ for all $1\leq l \leq d$ \footnote{We provide a symmetric version of Theorem~\ref{thm:l1bound} in Appendix~\ref{sec:modifiedTheorem1}, which essentially implies the same upper bound on the output perturbation without requiring the additional condition of $\|\bw\|_1 \geq \|\bwh\|_1$.}, then, we have the following bound on the output perturbation:
\begin{align}
\begin{aligned}
&\sup_{\|x\|_1\leq 1} \|f(\bx, \bw) - f(\bx, \bwh)\|_1
\\ &\leq \left(\sum_{l=1}^d \frac{\|\bw^{(l)} - \bwh^{(l)}\|_1}{\|\bw^{(l)}\|_1}\right) \left(\prod_{k=1}^d \|\bw^{(k)}\|_1\right)
\label{eq:l1bound}
\end{aligned}
\end{align}
%\noa{or we can say
%\begin{align}
%\begin{aligned}
%\sup_{\|x\|_1\leq 1} \frac{ \|f(\bx, \bw) - f(\bx, %\bwh)\|_1}{\prod_{k=1}^d \|\bw^{(k)}\|_1}
%\\ \leq \left(\sum_{l=1}^d \frac{\|\bw^{(l)} - %\bwh^{(l)}\|_1}{\|\bw^{(l)}\|_1}\right) 
%\label{eq:l1bound}
%\end{aligned}
%\end{align}
%which implies ``normalized output perturbation'' is bounded by %``normalized wight distortion''.
%Then, we do not have to explain why we ignore the last term in %Eq.~\ref{eq:l1bound}.
%}
i.e., the output perturbation is bounded by the $\ell_1$ distortion of the normalized weights.
\end{thm}
The matrix norm $\|\cdot\|_1$ is an induced norm by $\ell_1$ vector norm.
The proof is given in Appendix~\ref{distortion_appendix}. In Section~\ref{sec:surp} (Remark~\ref{rem_add_constraint}), we show that the proposed compression algorithm satisfies the additional assumption 
$\|\bw^{(l)}\|_1 \geq \|\bwh^{(l)}\|_1$ for all $1\leq l \leq d$. 
Since the last term in Eq.~\ref{eq:l1bound}, $\left(\prod_{k=1}^d \|\bw^{(k)}\|_1\right)$, is independent of the compression,
we do not include this term in our weight distortion function.
Then, one distortion function that naturally arises from Theorem~\ref{thm:l1bound} is $d(\bw, \bwh) = \sum_{l=1}^d \frac{\|\bw^{(l)} - \bwh^{(l)}\|_1}{\|\bw^{(l)}\|_1}$.
By changing the notation slightly, we would like to minimize the following distortion function
\begin{align}
d(\bu, \buh) = \frac{1}{n}\sum_{i =1}^n |u_i - \uh_i|
\label{eq:distortion_l1}
\end{align}
where $\bu$ is the normalized weights arisen from the normalization in Eq.~\ref{eq:l1bound},
i.e., $\bu^{(l)} = \frac{\bw^{(l)}}{\|\bw^{(l)} \|_1}$ for $l = 1, \dots, d$.
In the next section, we derive the rate-distortion function with the distortion metric in Eq.~\ref{eq:distortion_l1},
which approximates the perturbation ($\ell_1$ loss) on the output space due to compression.

%%%%%%%%%%%%%%%%%%%%%%%%%%%%%%%%%%%%%%%%%%%%%%%%%%
%% Rate-distortion function - optimality of pruning
%%%%%%%%%%%%%%%%%%%%%%%%%%%%%%%%%%%%%%%%%%%%%%%%%%
\subsection{Rate-Distortion Function for Neural Network Parameters}
\label{RD_function}
   
Since we define our distortion function as the $\ell_1$ distortion between $\bu$ and $\buh$ as in Eq.~\ref{eq:distortion_l1}, where $\bu$ is the normalized NN weights, we can formulate the compression problem as a lossy compression of the normalized NN weights.
Before deriving the rate-distortion function, we need a source distribution that fits the normalized weights $\bu$.
Figure~\ref{fig:density_weight} shows that Laplacian distribution is a good fit for pretrained NN weights after normalization 
%(we show that it is a good fit without the normalization as well in Appendix~\ref{density_appendix})
as opposed to the common Gaussian assumption in the prior work \citep{gao2019rate}. %The weights in Figure~~\ref{fig:density_weight} are loaded from PyTorch's pretrained models and no further training performed.  
For Figure~\ref{fig:density_weight}, we use PyTorch's pretrained models with no further training.
   
Now that we have a distortion metric and a source distribution, suitable for NN compression problem, we can finally derive the rate-distortion function.
We consider i.i.d.\ Laplacian source sequence $u_1, \dots, u_n$ distributed according to $f_L(u; \lambda) = \frac{\lambda}{2} e^{-\lambda|u|}$  with zero-mean and scale factor of $\lambda$, reconstructed sequence $v_1, \dots, v_n$, and $\ell_1$ distortion given in Eq.~\ref{eq:distortion_l1} with $\buh = \bv$. The rate-distortion function, which is the minimum achievable rate given the target distortion $D$ follows by:

\begin{lem}[\cite{berger2003rate}] The rate-distortion function for a Laplacian source with $\ell_1$ distortion is given by
\begin{align}
    R(D) = \begin{cases} -\log(\lambda D), & \ 0 \leq D \leq \frac{1}{\lambda} \\
    0, &  D > \frac{1}{\lambda}\end{cases}
\label{eq_rd_lap}
\end{align}
with the following optimal conditional probability distribution that achieves the minimum rate: 
\begin{align}
f_{\bU|\bV}(u|v) = \frac{1}{2D} e^{-|u-v|/D}.
\end{align}
Moreover, the marginal distribution of $\bV$ for the optimal reconstruction is %$f_{\bV}(v) = \lambda^2D^2 \cdot \delta (v) + (1- \lambda^2 D^2) \cdot \frac{\lambda}{2} e^{-\lambda|v|},$
\begin{align}
  f_{\bV}(v) = \lambda^2D^2 \cdot \delta (v) + (1- \lambda^2 D^2) \cdot \frac{\lambda}{2} e^{-\lambda|v|},  
\label{marginal}
\end{align}
where $\delta(v)$ is a Dirac measure. 
\label{rd_lap}
\end{lem}

\begin{figure*}[t]
    \centering %
        \subfigure[ResNet-18. ]{\includegraphics[width=.245\textwidth]{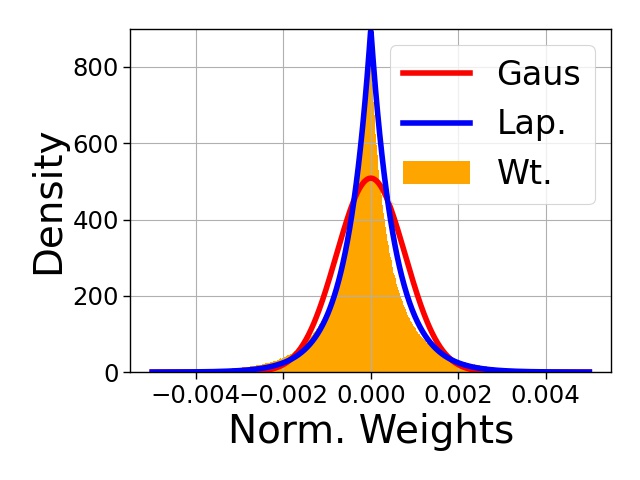}}
        \subfigure[ResNet-50.]{\includegraphics[width=.245\textwidth]{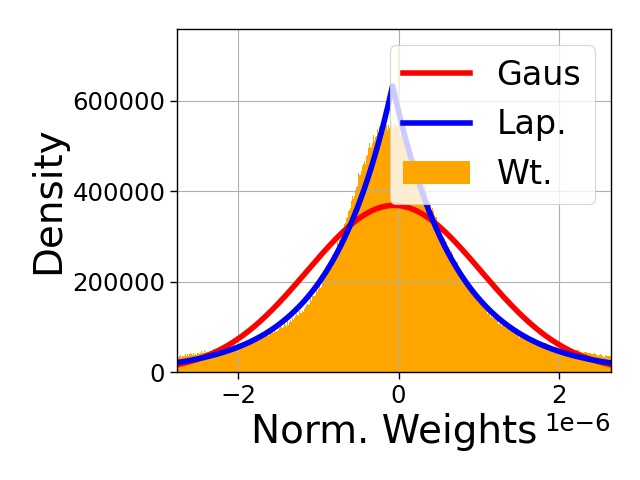}}
        \subfigure[ResNet-152.]{\includegraphics[width=.245\textwidth]{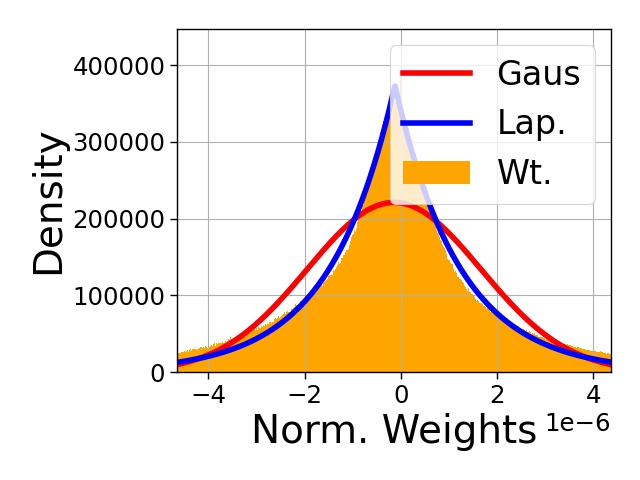}}
         \subfigure[Wide ResNet-50.]{\includegraphics[width=.245\textwidth]{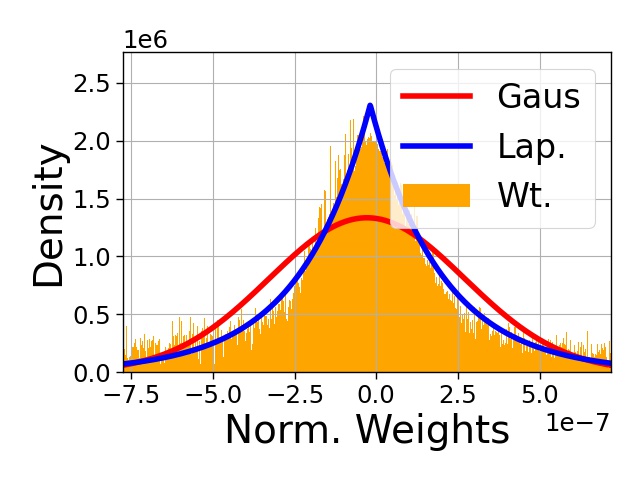}}
    \caption{Density of normalized weights. (a) ResNet-18, (b) ResNet-50, (c) ResNet-152, and (d) Wide ResNet-50. Gaus: Gaussian, Lap.: Laplacian, Wt.: Normalized NN weights. We use PyTorch's pretrained models with no further training.}\label{fig:density_weight}
    \vspace{-2mm}
   \end{figure*}
   
The proof of Lemma~\ref{rd_lap} is given in Appendix~\ref{lemma_proof_appendix}. The rate-distortion function in Eq.~\ref{eq_rd_lap} describes the tradeoff between NN compression ratio and weight distortion $D$ -- which upper bounds the \emph{output} perturbation. Lemma~\ref{rd_lap} further indicates that:
\begin{enumerate}

    \item[(1)] The rate-distortion theoretic optimal encoder-decoder pair makes the reconstruction sparse as the optimal marginal distribution in Lemma~\ref{rd_lap} is a sparse Laplacian distribution with sparsity $\lambda^2 D^2$. %More concretely, the reconstructed NN model must be sparse to satisfy the conditions for the optimal compression scheme. 
    Therefore, unless a compression scheme involves an implicit or explicit pruning step (to make the reconstruction sparse), the reconstruction does not follow the optimal marginal distribution. This would leave a sub-optimal %make the compression scheme sub-optimal 
    compression scheme since the mutual information $I(U; \Uh)$ between the source and reconstruction would be strictly larger than the rate-distortion function.
    %\noa{I like this paragarph, but I think we can compress this to meet length requirement}
    
    \item[(2)] Once $\bV$ is reconstructed at the decoder, the error term on the encoder side, $\bU-\bV$, follows a Laplacian distribution with parameter $1/D$ (see the conditional distribution in Lemma~\ref{rd_lap}). This allows for a practical coding scheme with low complexity based on successive refinement. That is, we can iteratively \footnote{The term ``iterative'' in our proposed algorithm is different from the ``iterative'' magnitude pruning concept.} describe NN weights with reasonable complexity.% at each iteration, which we elaborate more on in Section~\ref{method2_SuRP}. 
\end{enumerate}

In Theorem~\ref{thm:l1bound}, we add another constraint that the norm of the reconstructed weights at each layer is smaller than the norm of the original weights at the same layer ($\|\bw^{(l)}\|_1 \geq \|\bwh^{(l)}\|_1$). This is mainly because (1) sign change in the NN weights can significantly affect the NN output, hence sign bits must be protected to maintain the performance \citep{isiknoisy};
and (2) this inequality ($\|\bw^{(l)}\|_1 \geq \|\bwh^{(l)}\|_1$) is necessary to apply the iterative compression algorithm based on successive refinement (to be discussed in Section~\ref{method2_SuRP}).

In the next section, we develop a NN compression algorithm merging (i) our theoretical findings in Lemma~\ref{rd_lap} for \emph{optimality}  and (ii) successive refinement property for \emph{practicality}.

%%%%%%%%%%%%%%%%%%%%%%%%%%%%%%%%%%%%%%%%%%%%%%%%%%
%% The algorithm
%%%%%%%%%%%%%%%%%%%%%%%%%%%%%%%%%%%%%%%%%%%%%%%%%%
%%% Successive Refinement with Fixed Number of Decoders
%%%%%%%%%%%%%%%%%%%%%%%%%%%%%%%%%%%%%%%%%%%%%%%%%%%%%%%%%

\section{Successive Refinement for Pruning %(SuRP)
}
\label{method2_SuRP}
Rate-distortion theory, although, gives the limit of lossy compression and suggests that pruning must be a part of a good compression algorithm, does not explicitly give the optimal compression algorithm.
In \emph{theory}, a compression algorithm could be designed by letting the encoder pick the closest codeword from a random codebook
generated according to the marginal distribution of $\bV$ in Lemma~\ref{rd_lap}, as suggested by \cite{shannon2001mathematical}.
However, such a compressor would not be practical due to the size of the randomly generated codebook $|\cC| = 2^{nR(D)}$ (exponential in $n$ -- number of weights in our case). While designing practical compression algorithms without sacrificing the optimality is a fundamental dilemma in data compression, recent studies have shown that it is possible to design theoretically optimal schemes with low complexity for certain source distributions. In particular, for a successively refinable source, an optimal compression algorithm can also be practical \citep{no2016strong}.
We exploit this idea for the Laplacian source and develop a practical iterative compression algorithm that is rate-distortion theoretically optimal. We call it Successive Refinement for Pruning (SuRP) since it also outputs a sparse model, which can be viewed as a pruned model (although we do not explicitly prune the model).
We first present the successive refinement scheme for Laplacian source that shows the core idea to achieve lower complexity, but still impractical. We then push further to provide the practical algorithm and prove the optimality in a rate-distortion theoretic sense.

\looseness=-1
\subsection{Successive Refinement with Randomly Generated Codebooks}
\label{sec:fixed_successive_refinement}
Instead of a successive refinement scheme with two decoders as described in Section \ref{prelims},
we consider successive refinement with $L$ decoders. Let $\lambda = \lambda_1 < \dots < \lambda_L$ where $D_t = 1/\lambda_{t+1}$ is the target distortion at the $t$-th decoder. This is because the error term at iteration $t$ has a Laplacian distribution with parameter $\lambda_{t+1} = 1/D_t$ in an optimal compression scheme (see Lemma~\ref{rd_lap}).
We begin by setting $\bU^{(1)} = u^n$.
At the $t$-th iteration, the encoder finds $\bV^{(t)}$ that minimizes the distance $d(\bU^{(t)}, \bV^{(t)})$ from a codebook $\cC^{(t)}$,
then computes the residual $\bU^{(t+1)} = \bU^{(t)} - \bV^{(t)}$.
The $t$-th codebook $\cC^{(t)}$ consists of $2^{nR/L}$ codewords generated by the marginal distribution in Lemma~\ref{rd_lap}:
\begin{align*}
f_{\bV^{(t)}}(v) = \frac{\lambda_{t}^2}{\lambda_{t+1}^2} \cdot \delta (v)
+ \left(1-\frac{\lambda_{t}^2}{\lambda_{t+1}^2}\right) \cdot \frac{\lambda_{t}}{2}e^{-\lambda_{t} |v|}
\end{align*}
Since $\bU^{(t+1)}$ is again an i.i.d.\ Laplacian random sequence with parameter $\lambda_{t+1}=1/D_{t}$ (from the conditional probability in Lemma~\ref{rd_lap}),
the encoder can keep applying the same steps for Laplacian sources at each iteration. In summary, for $1\leq t\leq L-1$, the information-theoretic successive refinement encoder performs the following steps iteratively: (1) find $\bV^{(t)}\in\cC^{(t)}$ that minimizes $d(\bU^{(t)}, \bV^{(t)})$; and (2) update $\bf U$ as ${\bf U}^{(t+1)} = \bU^{(t)} - \bV^{(t)}$. The decoder, on the other hand, reconstructs $\bUh^{(t)} = \sum_{\tau=1}^t \bV^{(\tau)}$ at iteration $t$. This scheme has a complexity of $L\cdot 2^{nR/L}$ (the total size of the codebooks in $L$ iterations), which is lower than the naive random coding strategy ($2^{nR}$ at once). At the same time, it still achieves the rate-distortion limit, i.e., does not sacrifice the optimality, thanks to successive refinability of Laplacian source. 
However, the complexity is still exponential in $n$, which is impractical. We fix this in the next section.

%%% SuRP Algorithm
%%%%%%%%%%%%%%%%%%%%%%%%%%%%%%%%%%%%%%%%%%%%%%%%%%%%%%%%%
\subsection{SuRP Algorithm}\label{sec:surp}
The algorithm in Section~\ref{sec:fixed_successive_refinement} is rate-distortion theoretic optimal with lower complexity thanks to successive refinability, but still impractical due to the exponential size of the codebooks. In this section, we develop a new algorithm SuRP, that enjoys both practicality and optimality. Concretely, SuRP does not require a random codebook
%We saw in Section~\ref{sec:fixed_successive_refinement} that the information-theoretic approach is rate-distortion theoretic optimal with lower complexity due to successive refinability, but still impractical with the exponential size of codebooks. In this section, we develop a new algorithm SuRP, that enjoys both practicality and optimality. Concretely, SuRP does not require a random codebook
or a search for the nearest codeword $\bV^{(t)}$ from $\bU^{(t)}$ at each iteration, yet still rate-distortion theoretic optimal.
With the same initialization $\bU^{(1)}= u^n$ and $\lambda_1=\lambda$, new iterative coding scheme for $1\leq t\leq L-1$ is as follows:
\begin{enumerate}
    \item[(1)] Find indices $(i,j)$ such that $\bU^{(t)}_i \geq \frac{1}{\lambda_{t}}\log \frac{n}{2\beta}$
    and $\bU^{(t)}_j \leq -\frac{1}{\lambda_{t}}\log \frac{n}{2\beta}$.
    If there are more than one such indices, pick $(i,j)$ randomly. Encode $(i, j)$ as $m_t$.
    
\item[(2)] Let $\bV^{(t)}$ be an $n$-dimensional all-zero vector except $\bV^{(t)}_i = \frac{1}{\lambda_{t}}\log\frac{n}{2\beta}$
    and $\bV^{(t)}_j = -\frac{1}{\lambda_{t}}\log \frac{n}{2\beta}$.
    
\item[(3)] Let $\bU^{(t+1)} = \bU^{(t)} - \bV^{(t)}$.

\item[(4)] Set $\lambda_{t+1} = \frac{n}{n-2\log \frac{n}{2\beta}} \cdot \lambda_{t}$. 

\end{enumerate}
Here, $\beta>1$ is a tunable parameter. Similar to the algorithm in Section~\ref{sec:fixed_successive_refinement}, the reconstruction at $t$-th iteration would be $\bUh^{(t)} = \sum_{\tau=1}^t \bV^{(\tau)}$. We note that the encoder still communicates a sparse vector $\bV^{(t)}$ with two nonzero entries by sending $m_t=(i,j)$.  We give the pseudocode for SuRP in Appendix~\ref{algorithms_exponential_appendix}.

This coding scheme is equivalent to the original scheme in Section~\ref{sec:fixed_successive_refinement}, where $\frac{\lambda_{t}}{\lambda_{t+1}} = \frac{n-2\log \frac{n}{2\beta}}{n}$ for $1\leq t\leq L-1$ except the fact that the encoder does not do a search over a randomly generated codebook with exponential size, i.e., SuRP is practical. However, there is still an \emph{implicit} codebook $\cC^{(t)}$ at every iteration $t$, which consists of $n$-dimensional all-zero vectors except for two nonzero elements of values $\pm\frac{1}{\lambda_{t-1}}\log \frac{n}{2\beta}$.
The size of this codebook is $n(n-1)$ (not exponential anymore). Since these \emph{implicit} codebooks are not directly generated from the optimal marginal distribution in Lemma~\ref{rd_lap}, it is not obvious that SuRP is rate-distortion theoretic optimal. However, we prove the optimality under some criteria  in Section~\ref{zero_rate_opt}. % that it is rate-distortion theoretic optimal under certain criteria. 
   
We highlight that our scheme follows a bottom-up approach, in that sparsity in the reconstructed weights starts from $100\%$ at the first iteration and it decreases as the decoder receives new indices from the encoder (see Figure~\ref{fig:SURP_imagenet}(a)). This is similar to the progressive/hierarchical image compression techniques \citep{lewis1992image, rabbani2002jpeg2000}. Similarly, from Figure~\ref{fig:SURP_imagenet}(b), accuracy increases through the iterations. 

As a practical issue, when there is no index $i$ or $j$ such that $\bU_i^{(t)}\geq \frac{1}{\lambda_{t-1}} \log{\frac{n}{2 \beta}}$ or $\bU_j^{(t)}\leq - \frac{1}{\lambda_{t-1}} \log{\frac{n}{2 \beta}}$, the encoder re-estimates $\lambda$ and sends a refreshed value to the decoder. Obviously, these refreshments must be avoided to preserve optimality. We have seen in our experiments that this is a rare situation ($20$ refreshments in 20M iterations) and hence has a negligible effect on the overall optimality. In fact, we control the probability of this undesired situation (when there is need for refreshment) with the tunable parameter $\beta$. More precisely, the probability that all Laplacian random variables are smaller than $\frac{1}{\lambda}\log \frac{n}{2\beta}$ in magnitude (i.e., no index $i$ or $j$ found) is
\begin{align}
\begin{aligned}
\label{eq:prob}
& P \left [\max X_i < \frac{1}{\lambda}\log \frac{n}{2\beta} \mbox{ or }  \min X_i > -\frac{1}{\lambda}\log \frac{n}{2\beta} \right ]\\ 
& \leq P  \left [\max X_i < \frac{1}{\lambda}\log \frac{n}{2\beta}  \right ] +P  \left[\min X_i > -\frac{1}{\lambda}\log \frac{n}{2\beta}  \right] \\
& = 2\left(1 - \frac{1}{2}\frac{2\beta}{n}\right)^n \approx 2e^{-\beta}.
\end{aligned}
\end{align}
We set $\beta = \log n$ to bound this probability by $\frac{2}{n}$, which converges to $0$ as $n$ increases. We discuss the choice of $\beta$ in more detail with empirical results in Appendix~\ref{appedix_beta}. %\berivan{ablation study}
\begin{figure}[t]
%\vspace{-6mm}
    \centering %
        \subfigure[Sparsity. ]{\includegraphics[width=.23\textwidth]{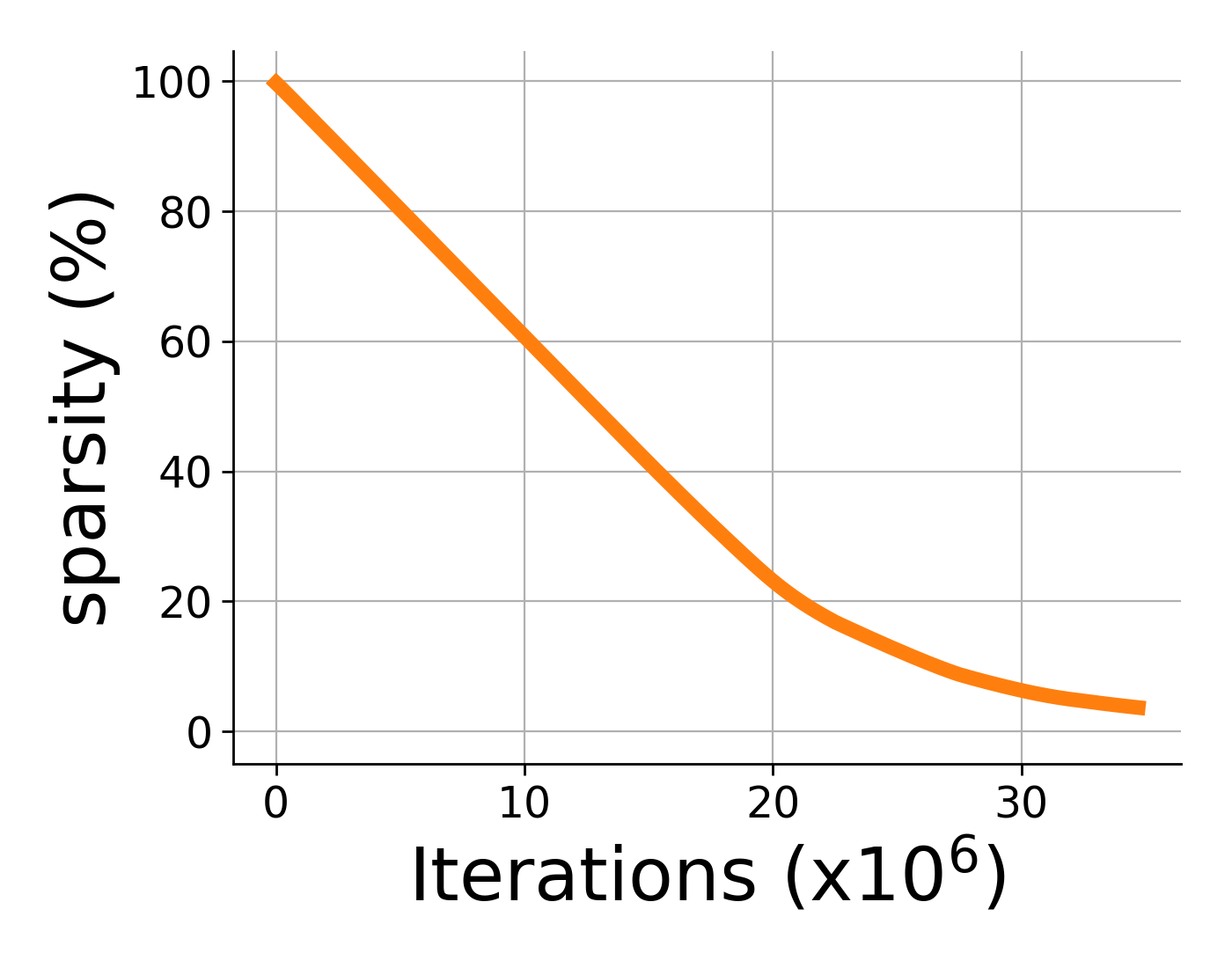}}
        \subfigure[Accuracy. ]{\includegraphics[width=.23\textwidth]{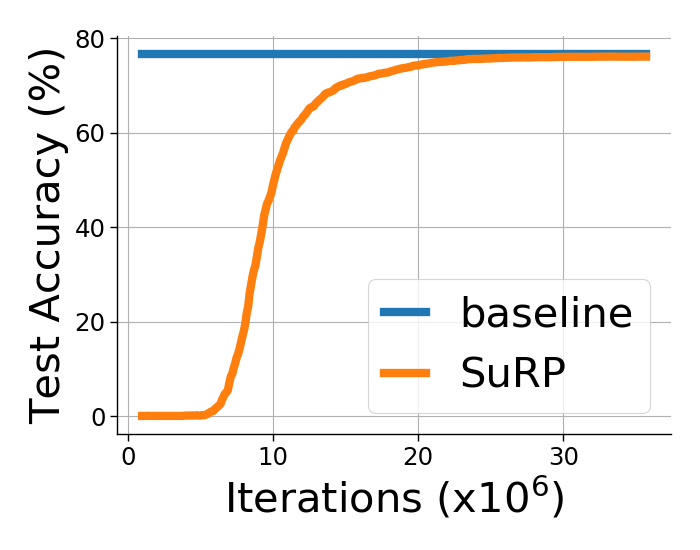}}
    \caption{Sparsity and accuracy of the reconstructed ResNet-50 on ImageNet during one cycle of SuRP. Iterations correspond to iterations running inside SuRP.}\label{fig:SURP_imagenet}
    %\vspace{-12mm}
   \end{figure}

\begin{rem}
From the extreme value theory, the maximum of $n$ Laplacian random variables concentrates near $\frac{1}{\lambda}\log \frac{n}{2}$,
which is the case of $\beta=1$.
Thus, one iteration of SuRP can be viewed as finding a near-maximum (and minimum) element. From this perspective, magnitude pruning can be viewed as a special case of SuRP. 
\end{rem}

\begin{rem}
SuRP guarantees $\|\bU^{(t)}\|_1 \geq \|\bV^{(t)}\|_1$ for all $t$. This implies that the magnitude of weights in $\bw$ is always larger than the magnitude of weights in $\bwh$. From Theorem~\ref{thm:l1bound}, we can say that the $\ell_1$ weight distortion of SuRP algorithm
is an upper bound to the NN model's output perturbation.
\label{rem_add_constraint}
\end{rem}
%%%%%%%%%%%%%%%%%%%%%%%%%%%%%%%%%%%%%%%%%%%%%%%%%%
%% Zero-rate optimality of SuRP
%%%%%%%%%%%%%%%%%%%%%%%%%%%%%%%%%%%%%%%%%%%%%%%%%%

\subsection{Zero-Rate Optimality of SuRP}
\label{zero_rate_opt}

In this section, we prove that SuRP is a zero-rate optimal compression algorithm. Given that SuRP uses an \emph{implicit} codebook of size $n(n-1)$ at each iteration, the rate is found as $R_n = \frac{\log n(n-1)}{n}$. We note that $R_n$ gets arbitrarily close to zero as $n$ increases. Moreover, the decrement in the distortion at each iteration is given as $D_n = \frac{2}{n \lambda } \log \frac{n}{2\beta_n}$, where $\beta_n=\beta$ as before. We start with the definition of \emph{zero-rate optimality}, which states that a sub-linear number of bits (in our case $\log n(n-1)$ nats) is being used optimally in the rate-distortion theoretic sense. 

\begin{definition}[Zero-rate optimality] A scheme with rate $R_n$, distortion decrement $D_n$, and distortion-rate function $D(\cdot)$, is zero-rate optimal if $\lim_{n\rightarrow\infty} R_n = 0$ and $ \lim_{n\rightarrow \infty} \frac{D_n}{R_n} = D'(0)$. 
\end{definition}
This implies that a zero-rate optimal scheme achieves the ``slope'' of the distortion-rate function at zero rate $R=0$. In the case of Laplacian source, this slope is $D'(0) = -\frac{1}{\lambda}$ since the distortion-rate function is $D(R) = \frac{1}{\lambda} e^{-R}$, which can be derived from the rate-distortion function in Lemma~\ref{rd_lap}. Finally, the following theorem states that a single iteration of SuRP is zero-rate optimal.

\setlength{\tabcolsep}{3pt}
\begin{table*}[!h]
%\vspace{-5mm}
\centering
\caption{Accuracy of VGG-16, ResNet-20, and DenseNet-121 on CIFAR-10. Results are averaged over five runs. %Global \citep{morcos2019one}, Unif. \citep{zhu2017prune}, Adap. \citep{stateofsparsity}, RiGL \citep{evci2020rigging}, LAMP \citep{lee2020deeper}. 
}
\resizebox{1.0\textwidth}{!}{
\begin{tabular}{clcccccccccc}
\toprule
&  \textbf{Pruning Ratio:}  & $93.12\%$   & $95.60\%$      & $97.19\%$    & $98.20\%$    & $98.85\%$    & $99.53\%$   & $99.70\%$   & $99.81\%$ & $99.88\%$ \\ \midrule
\centered{VGG-16}& \centered{Global  \\ Uniform  \\ Adaptive  \\RiGL  \\ LAMP  \\ SuRP (ours)   }
& \centered{$91.30$ \\ $91.47$  \\ $91.54$ \\ $92.34$ \\ $92.24$ \\ $\textbf{92.55}$} 
& \centered{$90.80$  \\ $90.78$ \\ $91.20$  \\ $91.99$  \\ $92.06$   \\ $\mathbf{92.13}$}
& \centered{$89.28$ \\ $88.61$ \\  $90.16$ \\ $91.66$ \\ $91.71$ \\ $\mathbf{91.95}$ }
&\centered{$85.55$ \\ $84.17$ \\  $89.44$ \\ $91.15$ \\ $91.66$ \\ $\mathbf{91.72}$ }
&\centered{$81.56$ \\ $55.68$ \\ $87.85$  \\ $90.55$ \\ $91.07$ \\ $\mathbf{91.21}$}
%& \centered{$54.58$ \\ $38.51$ \\ $86.53$ \\ $89.51$ \\ $90.49$ \\  $\mathbf{90.73}$}
& \centered{$41.91$ \\ $26.41$ \\ $84.84$ \\ $88.21$ \\ $89.64$ \\$\mathbf{90.65}$ }
& \centered{$31.93$ \\ $16.75$ \\ $82.41$ \\ $86.73$ \\ $88.75$   \\ $\mathbf{89.70}$}
&\centered{$21.87$  \\  $11.58$ \\ $74.54$ \\ $84.85$ \\  $87.07$ \\ $\mathbf{87.28}$}
&\centered{$11.72$ \\ $9.95$ \\$24.46$ \\$81.50$ \\ $84.90$ \\$\mathbf{85.04}$}
\\ \midrule
&  \textbf{Pruning Ratio:}  & $79.03\%$   & $86.58 \%$      & $91.41\%$   & $94.50 \%$    & $96.48\%$  & $97.75\%$    & $98.56\%$   & $99.41\%$  & $99.62\%$ \\ \midrule
\centered{ResNet-20} &\centered{Global  \\Uniform  \\ Adaptive  \\ RiGL   \\ LAMP  \\SuRP (ours)}
& \centered{$87.48$ \\ $87.24$ \\ $87.30$  \\ $87.63$ \\ $87.54$ \\$\mathbf{91.37}$ }
&\centered{$86.97$ \\$86.70$ \\ $87.00$ \\ $87.49$  \\ $87.12$  \\$\mathbf{90.44}$}
& \centered{ $86.29$ \\ $86.09$  \\  $86.27$ \\ $86.83$  \\$86.56$  \\ $\mathbf{89.00}$ }
& \centered{$85.02$ \\ $84.53$ \\ $85.00$ \\ $85.84$ \\ $85.64$ \\  $\mathbf{88.87}$}
& \centered{$83.15$ \\ $82.05$ \\ $83.23$ \\ $84.08$ \\ $84.18$  \\ $\mathbf{87.05}$ }
& \centered{$80.52$ \\ $77.19$ \\ $80.40$ \\ $81.76$  \\  $81.56$  \\  $\mathbf{83.98}$ }
& \centered{$76.28$ \\  $64.24$ \\  $76.40$ \\ $78.70$  \\ $78.63$ \\ $\mathbf{79.00}$  }
%& \centered{ $70.69$ \\ $47.97$  \\  $69.31$ \\  $74.40$ \\$74.20$  \\ $\mathbf{74.86}$ }
& \centered{$47.47$ \\ $20.45$ \\  $52.06$  \\  $66.42$ \\ $67.01$  \\ $\mathbf{70.64}$}
& \centered{$12.02$ \\ $13.35$ \\ $20.19$ \\ $50.90$\\ $51.24$\\ $\mathbf{54.22}$}
\\ \midrule
&  \textbf{Pruning Ratio:}  & $94.50\%$   & $95.60 \%$      & $96.48\%$   & $97.18 \%$    & $97.75\%$  & $98.20\%$    & $98.56\%$   & $99.08\%$  & $99.26\%$  \\ \midrule
\centered{DenseNet-121} &\centered{Global  \\Uniform  \\ Adaptive  \\ RiGL   \\ LAMP  \\SuRP (ours)}
& \centered{$90.16$ \\ $90.24$  \\ $90.25$  \\ $90.21$  \\ $90.89$ \\$\mathbf{91.42}$ }
&\centered{$89.52$  \\$89.50$   \\ $89.70$  \\ $89.79$  \\ $90.11$  \\$\mathbf{90.75}$}
& \centered{ $88.83$\\ $88.44$  \\  $89.03$ \\ $88.92$  \\$89.72$  \\ $\mathbf{90.30}$ }
& \centered{$88.00$ \\ $87.94$  \\ $88.22$  \\ $88.20$  \\ $89.12$ \\  $\mathbf{89.62}$}
& \centered{$86.85$ \\ $86.83$  \\ $87.40$  \\ $87.25$  \\ $88.39$  \\ $\mathbf{88.77}$ }
& \centered{$85.32$ \\ $85.00$  \\ $86.26$  \\ $86.22$  \\  $87.75$  \\  $\mathbf{88.06}$ }
& \centered{$77.68$ \\  $82.16$ \\  $84.55$ \\ $84.11$  \\ $86.53$ \\ $\mathbf{86.71}$  }
%& \centered{ $45.30$ \\ $70.13$ \\  $81.87$ \\ $81.82$  \\$85.13$  \\ $\mathbf{85.34}$ }
& \centered{$49.65$ \\ $66.46$  \\  $69.25$ \\ $59.06$  \\ $82.92$  \\ $\mathbf{83.18}$}
& \centered{$20.96$ \\ $48.71$  \\ $58.91$  \\ $59.07$  \\ $79.23$\\ $\mathbf{79.45}$}
\\ \midrule
&  \textbf{Pruning Ratio:}  & $59.00\%$   & $73.80 \%$      & $83.20\%$   & $89.30 \%$    & $93.13\%$  & $95.60\%$    & $97.18\%$   & $98.20\%$   & $99.26\%$ \\ 
\midrule
\centered{EfficientNet-B0} &\centered{Global \\Uniform  \\ Adaptive  \\ RiGL  \\ LAMP  \\SuRP (ours)}
& \centered{$89.66$ \\ $88.99$  \\ $89.18$  \\ $89.54$  \\ $89.52$ \\$\mathbf{90.96}$ }
&\centered{$89.55$  \\$88.26$   \\ $88.03$  \\ $90.09$  \\ $89.95$  \\$\mathbf{90.94}$}
& \centered{$88.80$ \\ $86.48$  \\  $86.71$ \\ $90.01$  \\$89.97$  \\ $\mathbf{90.89}$ }
& \centered{$87.64$ \\ $83.40$  \\ $84.16$  \\ $89.62$  \\ $90.21$ \\  $\mathbf{90.75}$}
& \centered{$84.36$ \\ $23.65$  \\ $36.64$  \\ $88.82$  \\ $89.91$  \\ $\mathbf{90.31}$ }
& \centered{$79.25$ \\ $10.83$  \\ $10.45$  \\ $87.08$  \\  $89.79$  \\  $\mathbf{90.08}$ }
& \centered{$11.09$ \\  $10.00$ \\  $10.00$ \\ $84.72$  \\ $89.30$ \\ $\mathbf{89.88}$  }
& \centered{$10.62$ \\ $10.00$ \\  $10.19$ \\ $81.53$  \\$88.51$  \\ $\mathbf{89.02}$ }
%5& \centered{$10.00$ \\ $10.00$  \\  $10.00$ \\ $51.31$  \\ $86.79$  \\ $\mathbf{87.80}$}
& \centered{$10.00$ \\ $10.00$  \\ $10.00$  \\ $13.40$  \\ $65.76$\\ $\mathbf{70.76}$}
\\
\bottomrule
\\
\end{tabular}
}
\label{tab:experiment_cifar}
\vspace{-2mm}
\end{table*}

\begin{thm} An iteration of SuRP is zero-rate optimal if $\lim_{n\rightarrow\infty}$ $\frac{\log 2\beta_n}{\log n(n-1)}=0$ holds.
\label{thm_zero_rate}
\end{thm}
%\noa{we can move the proof to appendix if it is hard to fit our draft in 8 pages}
%The proof of Theorem~\ref{thm_zero_rate} is given in Appendix~\ref{thm_zero_rate_proof_appendix}.
\begin{proof}
In an iteration of SuRP, where $R_n = \frac{\log n(n-1)}{n}$ and %$D_n = \frac{2}{n \lambda}\log \frac{n}{2\beta_n}$, we have
\begin{align*}
\frac{D_n} {R_n} & = -\frac{\frac{2}{\lambda}\log \frac{n}{2\beta_n}}{\log n(n-1)} \\
& = -\frac{1}{\lambda}\frac{\log n^2}{\log n(n-1)} + \frac{1}{\lambda}\frac{2\log 2\beta_n}{\log n(n-1)}.
\end{align*}
If $\lim_{n\rightarrow\infty}$ $\frac{\log 2\beta_n}{\log n(n-1)}=0$, it is clear that $\frac{D_n}{R_n}$ converges to $D'(0) = -\frac{1}{\lambda}$ as $n$ increases. Therefore, SuRP is zero-rate optimal under the condition that $\lim_{n\rightarrow\infty}$ $\frac{\log 2\beta_n}{\log n(n-1)}=0$.
\end{proof}
In Section~\ref{sec:surp}, we choose $\beta_n=\log n$ to keep the probability in Eq.~\ref{eq:prob} small. With this choice of $\beta_n$, $\lim_{n\rightarrow\infty}$ $\frac{\log 2\beta_n}{\log n(n-1)}=0$ holds. Therefore, from Theorem~\ref{thm_zero_rate}, our implementation of SuRP is indeed zero-rate optimal. 

\begin{rem}
In pure information-theoretic compression setting (main concern is not NN compression),
similar zero-rate optimal schemes were proposed for Gaussian source under mean squared error \citep{venkataramanan2014lossy, no2016rateless}.%, which is also successively refinable. They iteratively applied zero-rate optimal schemes and achieved the rate-distortion limit using a special property of Gaussian random variables.
\end{rem}
\section{Experiments}
\label{experiments}

In this section, we empirically investigate the performance of SuRP compared to recent pruning strategies in terms of accuracy-sparsity tradeoff. % and (2) improvements in the communication efficiency of federated learning (FL) when gradients are compressed with SuRP. %We emphasize that the main contribution of our paper is to bridge the gap between data compression and NN compression through a rate-distortion theoretic analysis of NN compression problem. 
We emphasize that the main contribution of our paper is to provide an information-theoretic justification for pruning. SuRP is designed solely to show that an algorithm derived with an information-theoretic approach indeed outputs a pruned model, as suggested by our findings. This also provides theoretical support for the recent success of pruning strategies.     

%For ease of implementation, instead of reconstructing two weights larger than $\frac{1}{\lambda} \log{\frac{n}{2 \beta}}$ in magnitude at each iteration, we reconstruct one weight larger than $\frac{1}{\lambda} \log{\frac{n}{\beta}}$. This is equivalent to the algorithm described in Section~\ref{method2_SuRP}, and the zero-rate optimality still holds (see Appendix\ref{algorithms_exponential_appendix} for details). 
%\noa{we can separately mention the FL setting (MNIST Caffe) since we haven't mentioned any of FL after introduction}
For our NN compression experiments, we consider two image datasets: CIFAR-10 \citep{krizhevsky2009learning} and ImageNet \citep{imagenet_cvpr09}. For CIFAR-10, we use four architectures: VGG-16 \citep{vgg}, ResNet-20 \citep{he2016deep}, DenseNet-121 \citep{iandola2014densenet}, and EfficientNet-B0 \citep{tan2019efficientnet}. %DenseNet-121 \citep{iandola2014densenet}, and EfficientNet-B0 \citep{tan2019efficientnet}. 
For ImageNet, we use ResNet-50 \citep{he2016deep, NEURIPS2019_9015}. We give additional details on model architectures and hyperparameters in Appendix~\ref{hyperparams_appendix}. We present experimental results averaged over 3-5 runs (see Appendix~\ref{experiments_app} for complete results). 

%We consider three image datasets: CIFAR-10 \cite{krizhevsky2009learning} and ImageNet \cite{imagenet_cvpr09} for NN compression and MNIST \cite{lecun2010mnist} for federated learning experiments. For MNIST, we use LeNet-5-Caffe. For CIFAR-10, we use four architectures: VGG-16 \cite{vgg}, ResNet-20 \cite{he2016deep}, DenseNet-121 \cite{iandola2014densenet}, and EfficientNet-B0 \cite{tan2019efficientnet}. For ImageNet, we use ResNet-50 \cite{he2016deep, NEURIPS2019_9015}. We give additional details on model architectures and hyperparameters in Appendix~\ref{hyperparams_appendix}. We present experimental results averaged over 3-5 runs (see Appendix~\ref{experiments_app} for complete results). 

\textbf{NN Compression/Pruning:}
In Tables~\ref{tab:experiment_cifar} and~\ref{tab:experiment_imagenet}, we compare our scheme with the recent pruning papers. We apply iterative pruning, meaning that we apply SuRP in repeating cycles (see Appendix~\ref{experiments_app} for details). As baselines, we consider Global \citep{morcos2019one}, Uniform \citep{zhu2017prune}, and Adaptive \citep{stateofsparsity} pruning techniques and LAMP \citep{lee2020deeper}. Additionally, we include comparisons to recent works on weight rewinding and dynamic sparsity, in particular SNIP \citep{snip}, DSR \citep{DSR2019}, SNFS \citep{SNFS2019}, and RiGL \citep{evci2020rigging}. 

We present the performance of pruned VGG-16, ResNet-20, and DenseNet-121 architectures on CIFAR-10 in Table~\ref{tab:experiment_cifar} and ResNet-50 on ImageNet in Table~\ref{tab:experiment_imagenet} . As can be seen from Table~\ref{tab:experiment_cifar}, SuRP outperforms prior work in all sparsity levels. From Table~\ref{tab:experiment_imagenet}, SuRP and Adaptive pruning perform similarly (with $\pm 0.06 \%$ difference), and they both outperform other baselines. 

\setlength{\tabcolsep}{3pt}
\begin{table}[!h]
%\vspace{-5mm}
\centering
\caption{Accuracy of ResNet-50 on ImageNet ($3$ runs). %Adaptive \citep{stateofsparsity}, SNIP \citep{snip}, DSR \citep{DSR2019}  ,  SNFS \citep{SNFS2019},  RiGL \citep{evci2020rigging} , LAMP \citep{lee2020deeper}. %Results are averaged over three runs.
}
\resizebox{0.65\columnwidth}{!}{
\begin{tabular}{lcc}
\toprule
\textbf{Pruning Ratio:}             & $80\%$   & $90\%$ \\ \midrule
Adaptive    & $75.60$  & $73.90$ \\
SNIP                    & $72.00$  & $67.20$\\ 
DSR                  & $73.30$  & $71.60$ \\
SNFS                & $74.90$  & $72.90$\\ 
RiGL         & $74.60$  & $72.00$\\ 
LAMP         & $74.96$  & $73.22$\\ 
SuRP (ours)                          & $\mathbf{75.54}$      & $\mathbf{73.95}$\\ \bottomrule
\end{tabular}}
\label{tab:experiment_imagenet}
\vspace{-2mm}
%\end{table}
\end{table}

We provide a comparison for lower pruning ratios in Appendix~\ref{experiments_app}.
%\setlength{\tabcolsep}{3pt}
%\begin{table*}[!h]
%%\vspace{-5mm}
%\centering
%\caption{Accuracy of ResNet-50 on ImageNet. Results are %averaged over three runs.}
%\resizebox{0.82\textwidth}{!}{
%\begin{tabular}{cccccccc}
%\toprule
%\textbf{Pruning Ratio}             &  Adaptive %\cite{stateofsparsity}    & SNIP \cite{snip} & DSR %\cite{DSR2019}  & SNFS \cite{SNFS2019}   & RiGL %\cite{evci2020rigging}  & LAMP \cite{lee2020deeper}  & SuRP %(ours)   \\ \midrule
% $80\%$                         & $75.60$  & $72.00$ & $73.30$ %& $74.90$ & $74.60$ & $74.96$ & $\mathbf{75.54}$ \\
% $90\%$                         & $73.90$  & $67.20$ & $71.60$ %& $72.90$ & $72.00$ & $73.22$ & $\mathbf{73.95}$ \\ \bottomrule
%\end{tabular}}
%\label{tab:experiment_imagenet}
%\vspace{-2mm}
%%\end{table}
%

%\input{sections/06-discussion}
\section{Discussion and Conclusion}
\label{conclusion}

In this work, we connected two lines of research, namely, data compression and NN compression. We investigated the theoretical tradeoff between the compression ratio and output perturbation of NN models, and found out that the rate-distortion theoretic formulation introduces a theoretical foundation for pruning. Guided by this, we developed a NN compression algorithm that outputs a pruned model and outperforms prior work. 

We note that our algorithm SuRP has an additional advantage in optimizing the bitrate of the model thanks to the rate-distortion theoretic basis of our approach. In particular, the decoder has only access to a list of indices, and these indices represent the whole (compressed) model -- more efficiently than describing the precise values of surviving weights. However, our current implementation does not exploit this efficiency to the full extent due to retraining steps after each pruning iteration. Like many, we will also look for ways to prune NN models without a retraining step afterward. That way, SuRP can be improved to provide a better accuracy-bitrate tradeoff, together with the already demonstrated sparsity-accuracy gain. We give more details on this and share experimental results in Appendices~\ref{bit_rate_appendix} and~\ref{experiments_app}. Finally, to give an idea about the bitrate efficiency of SuRP, we apply it for compressing gradients in a federated learning setting. Since the compressed gradients are not exposed to fine-tuning (like retraining in pruning), SuRP provides a substantial improvement on the bitrate compared to prior work. We elaborate more on this in the next paragraph.

%\setlength{\tabcolsep}{7pt}
%\begin{figure}
%\parbox{.37\linewidth}{
%\resizebox{0.37\textwidth}{!}{
%\begin{tabular}{lcc}
%\toprule
%\textbf{Pruning Ratio:}             & $80\%$   & $90\%$  \\ %\midrule
%Adaptive \cite{stateofsparsity}   & $75.60$  & $73.90$   \\
%SNIP \cite{snip}                   & $72.00$  & $67.20$ \\ 
%DSR \cite{DSR2019}                 & $73.30$  & $71.60$ \\
%SNFS \cite{SNFS2019}               & $74.90$  & $72.90$\\ 
%RiGL \cite{evci2020rigging}        & $74.60$  & $72.00$\\ 
%LAMP \cite{lee2020deeper}        & $74.96$  & $73.22$\\ 
%SuRP (ours)                          & $\mathbf{75.54}$      & %$\mathbf{73.93}$    \\ \bottomrule
%\\
%\end{tabular}}

%}
%\parbox{.60\linewidth}{
%\resizebox{0.60\textwidth}{!}{
%        \includegraphics[width=.60\textwidth]{figures/imagenet_p%runing.png}%}
%        \includegraphics[width=.60\textwidth]{figures/imagenet_d%ownstream_test_acc.png}%}
%    }}
%\caption{\textbf{(left)} Accuracy of ResNet-50 on ImageNet. %Results are averaged over three runs. \textbf{(right)} Sparsity %and accuracy of the reconstructed ResNet-50 on ImageNet during %one cycle of SuRP. Iterations correspond to the iterations %running inside SuRP.}\label{tab:imagenet}
%\vspace{-5mm}
%\end{figure}

\textbf{Compression for Federated Learning (FL):} FL is a distributed training setting where edge devices are responsible for doing local training and sending local gradients to a central server \citep{kairouz2019advances}. Given the resource limitations of edge devices, gradient communication is a significant bottleneck in FL, and gradient compression is crucial \citep{federated2}. We show in Appendix~\ref{federated_learning_appendix} that Laplacian distribution is a good fit for NN gradients. Therefore, SuRP is applicable to this problem as well. Our preliminary experiments with LeNet-5-Caffe on MNIST \citep{lecun2010mnist} compare SuRP with DGC \citep{lin2017deep} and rTop-k \citep{barnes2020rtop}. We compute the communication budget for prior work by assuming a naive encoding with $k(\log n+32)$ bits ($n$ is the model size)
since no other method is provided.
%\noa{I think using bold here is bit risky since we may not want to emphasize the FL result too much.
%Also, I think we can mention FL results in Discussion and Future work section}
With the same sparsity ratio $99.9 \%$, DGC achieves $98.5\%$ accuracy with \textbf{$\mathbf{2.05}$KB} of budget, rTop-$k$ achieves $99.1 \%$ accuracy with \textbf{$\mathbf{2.05}$KB} of budget, and SuRP achieves $99.1\%$ accuracy with \textbf{$\mathbf{218}$B}  of budget. Thus, SuRP provides $10\times$ times improvement in the gradients' compression rate while achieving the same accuracy as rTop-$k$.

\textbf{Limitations and Broader Impact:} When we evaluated our strategy, we only considered accuracy as a metric. However, compression might have an impact on other properties of the model as well, such as fairness, as pointed out by \cite{hooker2020characterising}. %Like many NN compression works, our distortion function does not address the potential disproportionate effects of compression on different subgroups of data. 
We agree that this issue deserves more attention from the community.

The codebase for this work is open-sourced
at \url{https://github.com/BerivanIsik/SuRP}.

 %In particular, we explained the compressibility of NN models via rate-distortion theory. Although our initial goal was to understand the theoretical tradeoff between the compression ratio and output perturbation, we found out that the rate-distortion theoretic formulation of the problem also introduces a theoretical foundation for pruning. \
%noa{(how about ``We investigated the theoretical tradeoff between the compression ratio and output perturbation, and we found out that the rat-distortion theoretic formulation ...'')} 

%\berivan{a paragraph to emphasize coarsest to finest approach and similarity to standard image compression techniques. here and intro} 
%\textbf{Limitations:} One limitation of the proposed approach is i.i.d.\ source assumption. While previous score-based pruning approaches make the same assumption implicitly, empirical results indicate that NN weights are correlated, hence not independent. In future work, we plan to formalize the problem without i.i.d.\ assumption while also hypothesizing about the source distribution more rigorously. 

\section{Acknowledgement}
This work was supported in part by a Sony Stanford Graduate Fellowship, a National Science Foundation (NSF) award, a Meta (formerly Facebook) resarch award, 
and a National Research Foundation of Korea (NRF) grant funded by the Korea government (MSIT) (No.\ 2021R1F1A105956711).

\newpage
\bibliography{aistats}

\begin{thebibliography}{}

\bibitem[Aji and Heafield, 2017]{Aji}
Aji, A.~F. and Heafield, K. (2017).
\newblock Sparse communication for distributed gradient descent.
\newblock {\em arXiv preprint arXiv:1704.05021}.

\bibitem[Ball{\'e} et~al., 2016]{balle2016end}
Ball{\'e}, J., Laparra, V., and Simoncelli, E.~P. (2016).
\newblock End-to-end optimized image compression.
\newblock {\em arXiv preprint arXiv:1611.01704}.

\bibitem[Banner et~al., 2018]{scalableQuant}
Banner, R., Hubara, I., Hoffer, E., and Soudry, D. (2018).
\newblock Scalable methods for 8-bit training of neural networks.
\newblock In {\em Advances in neural information processing systems}, pages
  5145--5153.

\bibitem[Barnes et~al., 2020]{barnes2020rtop}
Barnes, L.~P., Inan, H.~A., Isik, B., and {\"O}zg{\"u}r, A. (2020).
\newblock rtop-k: A statistical estimation approach to distributed sgd.
\newblock {\em IEEE Journal on Selected Areas in Information Theory},
  1(3):897--907.

\bibitem[Berger, 2003]{berger2003rate}
Berger, T. (2003).
\newblock Rate-distortion theory.
\newblock {\em Wiley Encyclopedia of Telecommunications}.

\bibitem[Blalock et~al., 2020]{blalock2020state}
Blalock, D., Ortiz, J. J.~G., Frankle, J., and Guttag, J. (2020).
\newblock What is the state of neural network pruning?
\newblock {\em arXiv preprint arXiv:2003.03033}.

\bibitem[Carreira-Perpinan and Idelbayev, 2018]{carreira2018learning}
Carreira-Perpinan, M.~A. and Idelbayev, Y. (2018).
\newblock ``learning-compression” algorithms for neural net pruning.
\newblock In {\em Proceedings of the IEEE Conference on Computer Vision and
  Pattern Recognition}, pages 8532--8541.

\bibitem[Chen et~al., 2021]{chen2021long}
Chen, T., Zhang, Z., Liu, S., Chang, S., and Wang, Z. (2021).
\newblock Long live the lottery: The existence of winning tickets in lifelong
  learning.
\newblock In {\em International Conference on Learning Representations, 2021a.
  URL https://openreview. net/forum}.

\bibitem[Choi et~al., 2020]{choi2020universal}
Choi, Y., El-Khamy, M., and Lee, J. (2020).
\newblock Universal deep neural network compression.
\newblock {\em IEEE Journal of Selected Topics in Signal Processing}.

\bibitem[Cover and Thomas, 2006]{elements_of_it}
Cover, T.~M. and Thomas, J.~A. (2006).
\newblock {\em Elements of Information Theory (Wiley Series in
  Telecommunications and Signal Processing)}.
\newblock Wiley-Interscience, USA.

\bibitem[Cun et~al., 1990]{OBD}
Cun, Y.~L., Denker, J.~S., and Solla, S.~A. (1990).
\newblock {\em Optimal Brain Damage}, page 598–605.
\newblock Morgan Kaufmann Publishers Inc., San Francisco, CA, USA.

\bibitem[Dai et~al., 2018]{compress_var_info}
Dai, B., Zhu, C., Guo, B., and Wipf, D. (2018).
\newblock Compressing neural networks using the variational information
  bottleneck.
\newblock In Dy, J. and Krause, A., editors, {\em Proceedings of the 35th
  International Conference on Machine Learning}, volume~80 of {\em Proceedings
  of Machine Learning Research}, pages 1135--1144. PMLR.

\bibitem[Deng et~al., 2009]{imagenet_cvpr09}
Deng, J., Dong, W., Socher, R., Li, L.-J., Li, K., and Fei-Fei, L. (2009).
\newblock {ImageNet: A Large-Scale Hierarchical Image Database}.
\newblock In {\em CVPR09}.

\bibitem[Dettmers and Zettlemoyer, 2019]{SNFS2019}
Dettmers, T. and Zettlemoyer, L. (2019).
\newblock Sparse networks from scratch: Faster training without losing
  performance.
\newblock {\em arXiv preprint arXiv:1907.04840}.

\bibitem[Elsen et~al., 2020]{elsen2020fast}
Elsen, E., Dukhan, M., Gale, T., and Simonyan, K. (2020).
\newblock Fast sparse convnets.
\newblock In {\em Proceedings of the IEEE/CVF conference on computer vision and
  pattern recognition}, pages 14629--14638.

\bibitem[Equitz and Cover, 1991]{equitz1991successive}
Equitz, W.~H. and Cover, T.~M. (1991).
\newblock Successive refinement of information.
\newblock {\em IEEE Transactions on Information Theory}, 37(2):269--275.

\bibitem[Evci et~al., 2020]{evci2020rigging}
Evci, U., Gale, T., Menick, J., Castro, P.~S., and Elsen, E. (2020).
\newblock Rigging the lottery: Making all tickets winners.
\newblock In {\em International Conference on Machine Learning}, pages
  2943--2952. PMLR.

\bibitem[Federici et~al., 2017]{federici2017improved}
Federici, M., Ullrich, K., and Welling, M. (2017).
\newblock Improved bayesian compression.
\newblock {\em arXiv preprint arXiv:1711.06494}.

\bibitem[Frankle and Carbin, 2019]{frankle2018lottery}
Frankle, J. and Carbin, M. (2019).
\newblock The lottery ticket hypothesis: Finding sparse, trainable neural
  networks.
\newblock {\em International Conference on Learning Representations (ICLR)}.

\bibitem[Gale et~al., 2019]{stateofsparsity}
Gale, T., Elsen, E., and Hooker, S. (2019).
\newblock The state of sparsity in deep neural networks.
\newblock {\em arXiv preprint arXiv:1902.09574}.

\bibitem[Gallager and Van~Voorhis, 1975]{gallager1975optimal}
Gallager, R. and Van~Voorhis, D. (1975).
\newblock Optimal source codes for geometrically distributed integer alphabets
  (corresp.).
\newblock {\em IEEE Transactions on Information theory}, 21(2):228--230.

\bibitem[Gao et~al., 2019]{gao2019rate}
Gao, W., Liu, Y.-H., Wang, C., and Oh, S. (2019).
\newblock Rate distortion for model compression: From theory to practice.
\newblock In {\em International Conference on Machine Learning}, pages
  2102--2111. PMLR.

\bibitem[Golomb, 1966]{golomb1966run}
Golomb, S. (1966).
\newblock Run-length encodings (corresp.).
\newblock {\em IEEE transactions on information theory}, 12(3):399--401.

\bibitem[Gr{\"u}nwald and Grunwald, 2007]{grunwald2007minimum}
Gr{\"u}nwald, P.~D. and Grunwald, A. (2007).
\newblock {\em The minimum description length principle}.
\newblock MIT press.

\bibitem[Guo et~al., 2016]{guo2016dynamic}
Guo, Y., Yao, A., and Chen, Y. (2016).
\newblock Dynamic network surgery for efficient dnns.
\newblock In {\em Advances in neural information processing systems}, pages
  1379--1387.

\bibitem[Han et~al., 2016]{deep_compression}
Han, S., Mao, H., and Dally, W.~J. (2016).
\newblock Deep compression: Compressing deep neural networks with pruning,
  trained quantization and huffman coding.
\newblock {\em International Conference on Learning Representations (ICLR)}.

\bibitem[Hassibi et~al., 1993]{OBS}
Hassibi, B., Stork, D.~G., Wolff, G., and Watanabe, T. (1993).
\newblock Optimal brain surgeon: Extensions and performance comparisons.
\newblock In {\em Proceedings of the 6th International Conference on Neural
  Information Processing Systems}, NIPS'93, page 263–270, San Francisco, CA,
  USA.

\bibitem[Havasi et~al., 2019]{havasi2018minimal}
Havasi, M., Peharz, R., and Hernández-Lobato, J.~M. (2019).
\newblock Minimal random code learning: Getting bits back from compressed model
  parameters.
\newblock In {\em International Conference on Learning Representations (ICLR)}.

\bibitem[He et~al., 2016]{he2016deep}
He, K., Zhang, X., Ren, S., and Sun, J. (2016).
\newblock Deep residual learning for image recognition.
\newblock In {\em Proceedings of the IEEE conference on computer vision and
  pattern recognition}, pages 770--778.

\bibitem[Hinton et~al., 2015]{distillation}
Hinton, G., Vinyals, O., and Dean, J. (2015).
\newblock Distilling the knowledge in a neural network.
\newblock In {\em NIPS Deep Learning and Representation Learning Workshop}.

\bibitem[Hooker et~al., 2020]{hooker2020characterising}
Hooker, S., Moorosi, N., Clark, G., Bengio, S., and Denton, E. (2020).
\newblock Characterising bias in compressed models.
\newblock {\em arXiv preprint arXiv:2010.03058}.

\bibitem[Huffman, 1952]{huffman1952method}
Huffman, D.~A. (1952).
\newblock A method for the construction of minimum-redundancy codes.
\newblock {\em Proceedings of the IRE}, 40(9):1098--1101.

\bibitem[Iandola et~al., 2014]{iandola2014densenet}
Iandola, F., Moskewicz, M., Karayev, S., Girshick, R., Darrell, T., and
  Keutzer, K. (2014).
\newblock Densenet: Implementing efficient convnet descriptor pyramids.
\newblock {\em arXiv preprint arXiv:1404.1869}.

\bibitem[Idelbayev and Carreira-Perpinan, 2020]{Idelbayev_low_rank3}
Idelbayev, Y. and Carreira-Perpinan, M.~A. (2020).
\newblock Low-rank compression of neural nets: Learning the rank of each layer.
\newblock In {\em Proceedings of the IEEE/CVF Conference on Computer Vision and
  Pattern Recognition (CVPR)}.

\bibitem[Idelbayev et~al., 2021]{idelbayev2021optimal}
Idelbayev, Y., Molchanov, P., Shen, M., Yin, H., Carreira-Perpinan, M.~A., and
  Alvarez, J.~M. (2021).
\newblock Optimal quantization using scaled codebook.
\newblock In {\em Proc. of the 2021 IEEE Computer Society Conf. Computer Vision
  and Pattern Recognition (CVPR’21), Virtual}.

\bibitem[Ioannou et~al., 2015]{low_rank2}
Ioannou, Y., Robertson, D., Shotton, J., Cipolla, R., and Criminisi, A. (2015).
\newblock Training cnns with low-rank filters for efficient image
  classification.
\newblock {\em arXiv preprint arXiv:1511.06744}.

\bibitem[Isik et~al., 2021]{isiknoisy}
Isik, B., Choi, K., Zheng, X., Weissman, T., Ermon, S., Wong, H. S.~P., and
  Alaghi, A. (2021).
\newblock Neural network compression for noisy storage devices.
\newblock {\em arXiv preprint arXiv:2102.07725}.

\bibitem[Jacob et~al., 2018]{jacob2018quantization}
Jacob, B., Kligys, S., Chen, B., Zhu, M., Tang, M., Howard, A., Adam, H., and
  Kalenichenko, D. (2018).
\newblock Quantization and training of neural networks for efficient
  integer-arithmetic-only inference.
\newblock In {\em Proceedings of the IEEE Conference on Computer Vision and
  Pattern Recognition}, pages 2704--2713.

\bibitem[Jung et~al., 2019]{jung2019learning}
Jung, S., Son, C., Lee, S., Son, J., Han, J.-J., Kwak, Y., Hwang, S.~J., and
  Choi, C. (2019).
\newblock Learning to quantize deep networks by optimizing quantization
  intervals with task loss.
\newblock In {\em Proceedings of the IEEE/CVF Conference on Computer Vision and
  Pattern Recognition}, pages 4350--4359.

\bibitem[Kairouz et~al., 2019]{kairouz2019advances}
Kairouz, P., McMahan, H.~B., Avent, B., Bellet, A., Bennis, M., Bhagoji, A.~N.,
  Bonawitz, K., Charles, Z., Cormode, G., Cummings, R., et~al. (2019).
\newblock Advances and open problems in federated learning.
\newblock {\em arXiv preprint arXiv:1912.04977}.

\bibitem[Konečný et~al., 2016]{federated2}
Konečný, J., McMahan, H.~B., Yu, F.~X., Richtarik, P., Suresh, A.~T., and
  Bacon, D. (2016).
\newblock Federated learning: Strategies for improving communication
  efficiency.
\newblock In {\em NIPS Workshop on Private Multi-Party Machine Learning}.

\bibitem[Koshelev, 1980]{koshelev1980hierarchical}
Koshelev, V.~N. (1980).
\newblock Hierarchical coding of discrete sources.
\newblock {\em Problemy peredachi informatsii}, 16(3):31--49.

\bibitem[Krizhevsky et~al., 2009]{krizhevsky2009learning}
Krizhevsky, A., Hinton, G., et~al. (2009).
\newblock Learning multiple layers of features from tiny images.

\bibitem[LeCun et~al., 1998]{lecun1998gradient}
LeCun, Y., Bottou, L., Bengio, Y., and Haffner, P. (1998).
\newblock Gradient-based learning applied to document recognition.
\newblock {\em Proceedings of the IEEE}, 86(11):2278--2324.

\bibitem[LeCun et~al., 2010]{lecun2010mnist}
LeCun, Y., Cortes, C., and Burges, C. (2010).
\newblock Mnist handwritten digit database.

\bibitem[Lee et~al., 2021]{lee2020deeper}
Lee, J., Park, S., Mo, S., Ahn, S., and Shin, J. (2021).
\newblock Layer-adaptive sparsity for the magnitude-based pruning.
\newblock {\em International Conference on Learning Representations}.

\bibitem[Lee et~al., 2018]{snip}
Lee, N., Ajanthan, T., and Torr, P.~H. (2018).
\newblock Snip: Single-shot network pruning based on connection sensitivity.
\newblock {\em arXiv preprint arXiv:1810.02340}.

\bibitem[Lewis and Knowles, 1992]{lewis1992image}
Lewis, A.~S. and Knowles, G. (1992).
\newblock Image compression using the 2-d wavelet transform.
\newblock {\em IEEE Transactions on image Processing}, 1(2):244--250.

\bibitem[Li et~al., 2016]{quant_1}
Li, F., Zhang, B., and Liu, B. (2016).
\newblock Ternary weight networks.
\newblock {\em arXiv preprint arXiv:1605.04711}.

\bibitem[Lin et~al., 2019]{lin2019towards}
Lin, S., Ji, R., Yan, C., Zhang, B., Cao, L., Ye, Q., Huang, F., and Doermann,
  D. (2019).
\newblock Towards optimal structured cnn pruning via generative adversarial
  learning.
\newblock In {\em Proceedings of the IEEE/CVF Conference on Computer Vision and
  Pattern Recognition}, pages 2790--2799.

\bibitem[Lin et~al., 2017]{lin2017deep}
Lin, Y., Han, S., Mao, H., Wang, Y., and Dally, W.~J. (2017).
\newblock Deep gradient compression: Reducing the communication bandwidth for
  distributed training.
\newblock {\em International Conference on Learning Representations (ICLR)}.

\bibitem[Liu et~al., 2018]{liu2018rethinking}
Liu, Z., Sun, M., Zhou, T., Huang, G., and Darrell, T. (2018).
\newblock Rethinking the value of network pruning.
\newblock {\em arXiv preprint arXiv:1810.05270}.

\bibitem[Louizos et~al., 2017a]{louizos2017bayesian}
Louizos, C., Ullrich, K., and Welling, M. (2017a).
\newblock Bayesian compression for deep learning.
\newblock {\em arXiv preprint arXiv:1705.08665}.

\bibitem[Louizos et~al., 2017b]{louizos2017learning}
Louizos, C., Welling, M., and Kingma, D.~P. (2017b).
\newblock Learning sparse neural networks through $ l\_0 $ regularization.
\newblock {\em arXiv preprint arXiv:1712.01312}.

\bibitem[McMahan et~al., 2017]{federated0}
McMahan, H.~B., Moore, E., Ramage, D., Hampson, S., and y~Arcas, B.~A. (2017).
\newblock Communication-efficient learning of deep networks from decentralized
  data.
\newblock In {\em AISTATS}.

\bibitem[Mocanu et~al., 2018]{mocanu2018scalable}
Mocanu, D.~C., Mocanu, E., Stone, P., Nguyen, P.~H., Gibescu, M., and Liotta,
  A. (2018).
\newblock Scalable training of artificial neural networks with adaptive sparse
  connectivity inspired by network science.
\newblock {\em Nature communications}, 9(1):1--12.

\bibitem[Molchanov et~al., 2017]{molchanov2017variational}
Molchanov, D., Ashukha, A., and Vetrov, D. (2017).
\newblock Variational dropout sparsifies deep neural networks.
\newblock In {\em International Conference on Machine Learning}, pages
  2498--2507. PMLR.

\bibitem[Molchanov et~al., 2016]{molchanov2016pruning}
Molchanov, P., Tyree, S., Karras, T., Aila, T., and Kautz, J. (2016).
\newblock Pruning convolutional neural networks for resource efficient
  inference.
\newblock {\em arXiv preprint arXiv:1611.06440}.

\bibitem[Morcos et~al., 2019]{morcos2019one}
Morcos, A.~S., Yu, H., Paganini, M., and Tian, Y. (2019).
\newblock One ticket to win them all: generalizing lottery ticket
  initializations across datasets and optimizers.
\newblock {\em arXiv preprint arXiv:1906.02773}.

\bibitem[Mostafa and Wang, 2019]{DSR2019}
Mostafa, H. and Wang, X. (2019).
\newblock Parameter efficient training of deep convolutional neural networks by
  dynamic sparse reparameterization.
\newblock In {\em International Conference on Machine Learning}, pages
  4646--4655. PMLR.

\bibitem[No et~al., 2016]{no2016strong}
No, A., Ingber, A., and Weissman, T. (2016).
\newblock Strong successive refinability and rate-distortion-complexity
  tradeoff.
\newblock {\em IEEE Transactions on Information Theory}, 62(6):3618--3635.

\bibitem[No and Weissman, 2016]{no2016rateless}
No, A. and Weissman, T. (2016).
\newblock Rateless lossy compression via the extremes.
\newblock {\em IEEE transactions on information theory}, 62(10):5484--5495.

\bibitem[Oktay et~al., 2019]{oktay2019scalable}
Oktay, D., Ball{\'e}, J., Singh, S., and Shrivastava, A. (2019).
\newblock Scalable model compression by entropy penalized reparameterization.
\newblock In {\em International Conference on Learning Representations (ICLR)}.

\bibitem[Park et~al., 2020]{park2020lookahead}
Park, S., Lee, J., Mo, S., and Shin, J. (2020).
\newblock Lookahead: A far-sighted alternative of magnitude-based pruning.
\newblock {\em International Conference on Learning Representations (ICLR)}.

\bibitem[Paszke et~al., 2019]{NEURIPS2019_9015}
Paszke, A., Gross, S., Massa, F., Lerer, A., Bradbury, J., Chanan, G., Killeen,
  T., Lin, Z., Gimelshein, N., Antiga, L., Desmaison, A., Kopf, A., Yang, E.,
  DeVito, Z., Raison, M., Tejani, A., Chilamkurthy, S., Steiner, B., Fang, L.,
  Bai, J., and Chintala, S. (2019).
\newblock Pytorch: An imperative style, high-performance deep learning library.
\newblock In Wallach, H., Larochelle, H., Beygelzimer, A., d\textquotesingle
  Alch\'{e}-Buc, F., Fox, E., and Garnett, R., editors, {\em Advances in Neural
  Information Processing Systems 32}, pages 8024--8035. Curran Associates, Inc.

\bibitem[Peng et~al., 2019]{peng2019collaborative}
Peng, H., Wu, J., Chen, S., and Huang, J. (2019).
\newblock Collaborative channel pruning for deep networks.
\newblock In {\em International Conference on Machine Learning}, pages
  5113--5122. PMLR.

\bibitem[Polino et~al., 2018]{polino2018model}
Polino, A., Pascanu, R., and Alistarh, D. (2018).
\newblock Model compression via distillation and quantization.
\newblock {\em arXiv preprint arXiv:1802.05668}.

\bibitem[Rabbani, 2002]{rabbani2002jpeg2000}
Rabbani, M. (2002).
\newblock Jpeg2000: Image compression fundamentals, standards and practice.
\newblock {\em Journal of Electronic Imaging}, 11(2):286.

\bibitem[Renda et~al., 2020]{renda2020comparing}
Renda, A., Frankle, J., and Carbin, M. (2020).
\newblock Comparing fine-tuning and rewinding in neural network pruning.
\newblock In {\em International Conference on Learning Representations}.

\bibitem[Sainath et~al., 2013]{low_rank1}
Sainath, T.~N., Kingsbury, B., Sindhwani, V., Arisoy, E., and Ramabhadran, B.
  (2013).
\newblock Low-rank matrix factorization for deep neural network training with
  high-dimensional output targets.
\newblock In {\em 2013 IEEE international conference on acoustics, speech and
  signal processing}, pages 6655--6659. IEEE.

\bibitem[Salomon, 2004]{salomon2004data}
Salomon, D. (2004).
\newblock {\em Data compression: the complete reference}.
\newblock Springer Science \& Business Media.

\bibitem[Shannon, 1948]{shannon2001mathematical}
Shannon, C.~E. (1948).
\newblock A mathematical theory of communication.
\newblock {\em The Bell system technical journal}, 27(3):379--423.

\bibitem[Shannon, 1959]{shannon1959coding}
Shannon, C.~E. (1959).
\newblock Coding theorems for a discrete source with a fidelity criterion.
\newblock {\em IRE Nat. Conv. Rec}, 4(142-163):1.

\bibitem[Simonyan and Zisserman, 2014]{vgg}
Simonyan, K. and Zisserman, A. (2014).
\newblock Very deep convolutional networks for large-scale image recognition.
\newblock {\em arXiv preprint arXiv:1409.1556}.

\bibitem[Stock et~al., 2021]{QuantNoise}
Stock, P., Fan, A., Graham, B., Grave, E., Gribonval, R., Jegou, H., and
  Joulin, A. (2021).
\newblock Training with quantization noise for extreme model compression.
\newblock In {\em International Conference on Learning Representations}.

\bibitem[Tan and Le, 2019]{tan2019efficientnet}
Tan, M. and Le, Q. (2019).
\newblock Efficientnet: Rethinking model scaling for convolutional neural
  networks.
\newblock In {\em International Conference on Machine Learning}, pages
  6105--6114. PMLR.

\bibitem[Ullrich et~al., 2017]{ullrich2017soft}
Ullrich, K., Meeds, E., and Welling, M. (2017).
\newblock Soft weight-sharing for neural network compression.
\newblock {\em arXiv preprint arXiv:1702.04008}.

\bibitem[Venkataramanan et~al., 2014]{venkataramanan2014lossy}
Venkataramanan, R., Sarkar, T., and Tatikonda, S. (2014).
\newblock Lossy compression via sparse linear regression: Computationally
  efficient encoding and decoding.
\newblock {\em IEEE transactions on information theory}, 60(6):3265--3278.

\bibitem[Verdu, 1996]{verdu1996exponential}
Verdu, S. (1996).
\newblock The exponential distribution in information theory.
\newblock {\em Problemy peredachi informatsii}, 32(1):100--111.

\bibitem[Wang et~al., 2018]{Wang}
Wang, H., Sievert, S., Liu, S., Charles, Z.~B., Papailiopoulos, D.~S., and
  Wright, S. (2018).
\newblock Atomo: Communication-efficient learning via atomic sparsification.
\newblock In {\em NeurIPS}.

\bibitem[Wang et~al., 2019a]{wang2019private}
Wang, J., Bao, W., Sun, L., Zhu, X., Cao, B., and Philip, S.~Y. (2019a).
\newblock Private model compression via knowledge distillation.
\newblock In {\em Proceedings of the AAAI Conference on Artificial
  Intelligence}, volume~33, pages 1190--1197.

\bibitem[Wang et~al., 2019b]{quant_3}
Wang, K., Liu, Z., Lin, Y., Lin, J., and Han, S. (2019b).
\newblock Haq: Hardware-aware automated quantization with mixed precision.
\newblock In {\em Proceedings of the IEEE/CVF Conference on Computer Vision and
  Pattern Recognition}, pages 8612--8620.

\bibitem[Wangni et~al., 2018]{Wangni}
Wangni, J., Wang, J., Liu, J., and Zhang, T. (2018).
\newblock Gradient sparsification for communication-efficient distributed
  optimization.
\newblock In Bengio, S., Wallach, H., Larochelle, H., Grauman, K.,
  Cesa-Bianchi, N., and Garnett, R., editors, {\em Advances in Neural
  Information Processing Systems 31}, pages 1299--1309. Curran Associates, Inc.

\bibitem[{Wiedemann} et~al., 2020]{DeepCABAC}
{Wiedemann}, S., {Kirchhoffer}, H., {Matlage}, S., {Haase}, P., {Marban}, A.,
  {Marinč}, T., {Neumann}, D., {Nguyen}, T., {Schwarz}, H., {Wiegand}, T.,
  {Marpe}, D., and {Samek}, W. (2020).
\newblock Deepcabac: A universal compression algorithm for deep neural
  networks.
\newblock {\em IEEE Journal of Selected Topics in Signal Processing},
  14(4):700--714.

\bibitem[Xiao et~al., 2019]{xiao2019autoprune}
Xiao, X., Wang, Z., and Rajasekaran, S. (2019).
\newblock Autoprune: Automatic network pruning by regularizing auxiliary
  parameters.
\newblock {\em Advances in neural information processing systems}, 32.

\bibitem[Young et~al., 2020]{young2020transform}
Young, S.~I., Zhe, W., Taubman, D., and Girod, B. (2020).
\newblock Transform quantization for cnn compression.
\newblock {\em arXiv preprint arXiv:2009.01174}.

\bibitem[Yu et~al., 2018]{yu2018nisp}
Yu, R., Li, A., Chen, C.-F., Lai, J.-H., Morariu, V.~I., Han, X., Gao, M., Lin,
  C.-Y., and Davis, L.~S. (2018).
\newblock Nisp: Pruning networks using neuron importance score propagation.
\newblock In {\em Proceedings of the IEEE Conference on Computer Vision and
  Pattern Recognition}, pages 9194--9203.

\bibitem[Zhao et~al., 2019]{zhao2019variational}
Zhao, C., Ni, B., Zhang, J., Zhao, Q., Zhang, W., and Tian, Q. (2019).
\newblock Variational convolutional neural network pruning.
\newblock In {\em Proceedings of the IEEE/CVF Conference on Computer Vision and
  Pattern Recognition}, pages 2780--2789.

\bibitem[Zhe et~al., 2021]{zhe2021rate}
Zhe, W., Lin, J., Aly, M.~S., Young, S., Chandrasekhar, V., and Girod, B.
  (2021).
\newblock Rate-distortion optimized coding for efficient cnn compression.
\newblock In {\em 2021 Data Compression Conference (DCC)}, pages 253--262.
  IEEE.

\bibitem[Zhu and Gupta, 2017]{zhu2017prune}
Zhu, M. and Gupta, S. (2017).
\newblock To prune, or not to prune: exploring the efficacy of pruning for
  model compression.
\newblock {\em arXiv preprint arXiv:1710.01878}.

\end{thebibliography}
\onecolumn \makesupplementtitle
%\aistatstitle{Appendix}
\section*{}
\label{appendix}
\renewcommand{\thesubsection}{\Alph{subsection}}

\subsection{Proof of Theorem~\ref{thm:l1bound}}
\label{distortion_appendix} 
In this section, we provide the proof of Theorem~\ref{thm:l1bound}.
The fully connected $d$-layer NN model with 1-Lipschitz activations $\sigma(\cdot)$ is given by
\begin{align*}
f(\bx; \bw) = \bw^{(d)}\sigma(\bw^{(d-1)}\sigma(\cdots \bw^{(2)}\sigma(\bw^{(1)}\bx))).
\end{align*}
We let $\bw^{(1:i)} = \{\bw^{(1)}, \ldots, \bw^{(i)} \}$ for $1\leq i\leq d$ where $\bw^{(1:d)} = \bw$.
Furthermore, we define the first $i$ layer of the network by
\begin{align*}
f(\bx; \bw^{(1:i)}) = \bw^{(i)}\sigma(\bw^{(i-1)}\sigma(\cdots \bw^{(2)}\sigma(\bw^{(1)}\bx))).
\end{align*}

Then, the output perturbation is bounded by
\begin{align}
& \|f(\bx; \bw^{(1:d)}) - f(\bx; \bwh^{(1:d)})\|_1\nonumber\\
=& \|\bw^{(d)} \sigma(f(\bx; \bw^{(1:d-1)})) - \bwh^{(d)}\sigma(f(\bx; \bwh^{(1:d-1)}))\|_1\nonumber\\
\leq& \|\bw^{(d)} \sigma(f(\bx; \bw^{(1:d-1)})) - \bwh^{(d)}\sigma(f(\bx; \bw^{(1:d-1)}))\|_1\nonumber\\
& + \|\bwh^{(d)} \sigma(f(\bx; \bw^{(1:d-1)})) - \bwh^{(d)}\sigma(f(\bx; \bwh^{(1:d-1)}))\|_1\label{eq:triangle}\\
\leq& \|\bw^{(d)} - \bwh^{(d)}\|_1 \cdot \| \sigma(f(\bx; \bw^{(1:d-1)}))\|_1 
+ \|\bwh^{(d)}\|_1  \cdot \| \sigma(f(\bx; \bw^{(1:d-1)})) - \sigma(f(\bx; \bwh^{(1:d-1)}))\|_1\label{eq:l1-norm product}\\
\leq& \|\bw^{(d)} - \bwh^{(d)}\|_1 \cdot \| f(\bx; \bw^{(1:d-1)})\|_1 + \|\bwh^{(d)}\|_1 \cdot  \| f(\bx; \bw^{(1:d-1)}) - f(\bx; \bwh^{(1:d-1)})\|_1\label{eq:lipshitz}\\
\leq& \|\bw^{(d)} - \bwh^{(d)}\|_1 \cdot  \prod_{l=1}^{d-1} \| \bw^{(l)}\|_1 \cdot \|\bx \|_1+ \|\bwh^{(d)}\|_1 \cdot \| f(\bx; \bw^{(1:d-1)}) - f(\bx; \bwh^{(1:d-1)})\|_1\label{eq:bound of f}\\
\leq& \|\bw^{(d)} - \bwh^{(d)}\|_1 \cdot  \prod_{l=1}^{d-1} \| \bw^{(l)}\|_1 + \|\bwh^{(d)}\|_1 \cdot \| f(\bx; \bw^{(1:d-1)}) - f(\bx; \bwh^{(1:d-1)})\|_1\label{eq:norm of x}\\
\leq& \|\bw^{(d)} - \bwh^{(d)}\|_1 \cdot \prod_{l=1}^{d-1} \| \bw^{(l)}\|_1 + \|\bw^{(d)}\|_1 \cdot \| f(\bx; \bw^{(1:d-1)}) - f(\bx; \bwh^{(1:d-1)})\|_1\label{eq:w leq what}
\end{align}
where Eq.~\ref{eq:triangle} is due to triangle inequality, and Eq.~\ref{eq:l1-norm product} holds from the property of $\ell_1$-norm (and induced norm).
Eq.~\ref{eq:lipshitz} is from 1-Lipshitzness of activation $\sigma(\cdot)$, i.e., $\|\sigma(\bx)\|_1 \leq \|\bx\|_1$.
Eq.~\ref{eq:bound of f} holds from the following lemma.
\begin{lem}
For all $1\leq i\leq d$, we have
$\|f(\bx;\bw^{(1:i)})\|_1\leq \prod_{j=1}^i \|\bw^{(j)}\|_1 \cdot \|\bx\|_1$.
\end{lem}
\begin{proof}
From the property of $\ell_1$-norm, we have
\begin{align}
\|f(\bx;\bw^{(1:i)})\|_1\leq& \|\bw^{(i)}\|_1 \cdot \|\sigma(f(\bx;\bw^{(1:i-1)}))\|_1\\
\leq& \|\bw^{(i)}\|_1 \cdot \|f(\bx;\bw^{(1:i-1)})\|_1
\end{align}
where the last inequality is due to 1-Lipshitzness of $\sigma$.
Then, we can keep applying the same inequality, which concludes the proof.
\end{proof}
Eq.~\ref{eq:norm of x} follows by the constraint $\|\bx\|_1 \leq 1$ in Theorem~\ref{thm:l1bound}. Finally Eq.~\ref{eq:w leq what} is from the assumption $\|\bwh^{(l)}\|_1\leq \|\bw^{(l)}\|_1$ for all $1\leq l\leq d$.

Thus, we have
\begin{align}
&\left(\prod_{l=1}^d \frac{1}{\|\bw^{(l)}\|_1}\right) \|f(\bx; \bw^{(1:d)}) - f(\bx; \bwh^{(1:d)})\|_1\nonumber\\
&\leq \frac{\|\bw^{(d)} - \bwh^{(d)}\|_1}{\|\bw^{(d)}\|_1}
+\left(\prod_{l=1}^{d-1} \frac{1}{\|\bw^{(l)}\|_1}\right) \|f(\bx; \bw^{(1:d-1)}) - f(\bx; \bwh^{(1:d-1)})\|_1.
\end{align}
We can repeat the same procedure, and get
\begin{align}
&\left(\prod_{l=1}^d \frac{1}{\|\bw^{(l)}\|_1}\right) \|f(\bx; \bw^{(1:d)}) - f(\bx; \bwh^{(1:d)})\|_1 \leq \sum_{l=1}^d \frac{\|\bw^{(l)} - \bwh^{(l)}\|_1}{\|\bw^{(l)}\|_1}.
\end{align}
This completes the proof.

\subsection{Modified Theorem 1}
\label{sec:modifiedTheorem1}

In this section, we provide a symmetric version of Theorem~\ref{thm:l1bound},
which essentially implies the same upper bound on the output perturbation without requiring the additional condition of $\|\bw\|_1 \geq \|\bwh\|_1$.

\begin{thm}\label{thm:modifiedTheorem1}
Suppose $f(\cdot;\bw)$ is a fully-connected NN model with $d$ layers and 1-Lipschitz activations $\sigma(\cdot)$ such that $\sigma(0)=0$, e.g., ReLU. Let $\bwh$ be the reconstructed weights (after compression) where all layers are subject to compression. Then, we have the following bound on the output perturbation:
\begin{align}
\begin{aligned}
\sup_{\|x\|_1\leq 1} \|f(\bx, \bw) - f(\bx, \bwh)\|_1
\leq \left(\sum_{l=1}^d \frac{\|\bw^{(l)} - \bwh^{(l)}\|_1}{\max\{\|\bwh^{(l)}\|_1, \|\bw^{(l)}\|_1\}}\right)
\left(\prod_{k=1}^d \max\{\|\bwh^{(k)}\|_1, \|\bw^{(k)}\|_1\}\right).
\label{eq:modifiedl1bound}
\end{aligned}
\end{align}
%\noa{or we can say
%\begin{align}
%\begin{aligned}
%\sup_{\|x\|_1\leq 1} \frac{ \|f(\bx, \bw) - f(\bx, %\bwh)\|_1}{\prod_{k=1}^d \|\bw^{(k)}\|_1}
%\\ \leq \left(\sum_{l=1}^d \frac{\|\bw^{(l)} - %\bwh^{(l)}\|_1}{\|\bw^{(l)}\|_1}\right) 
%\label{eq:l1bound}
%\end{aligned}
%\end{align}
%which implies ``normalized output perturbation'' is bounded by %``normalized wight distortion''.
%Then, we do not have to explain why we ignore the last term in %Eq.~\ref{eq:l1bound}.
%}
\end{thm}

By rearranging the terms in Eq.~\ref{eq:modifiedl1bound}, we get the following relation:
\begin{align}
    \begin{aligned}
     \sup_{\|x\|_1\leq 1} \frac{\|f(\bx, \bw) - f(\bx, \bwh)\|_1}{ \prod_{k=1}^d \max\{\|\bwh^{(k)}\|_1, \|\bw^{(k)}\|_1\} }
\leq \left(\sum_{l=1}^d \frac{\|\bw^{(l)} - \bwh^{(l)}\|_1}{\max\{\|\bwh^{(l)}\|_1, \|\bw^{(l)}\|_1\}}\right), \label{eq:symmetric bound}
    \end{aligned}
\end{align}
which implies that the normalized output perturbation is bounded by the normalized weight differences.
With the additional condition of $\|\bw\|_1 \geq \|\bwh\|_1$, we can simply recover Theorem~\ref{thm:l1bound} from Theorem~\ref{thm:modifiedTheorem1}.
\begin{align}
     \sup_{\|x\|_1\leq 1} \frac{\|f(\bx, \bw) - f(\bx, \bwh)\|_1}{ \prod_{k=1}^d \max\{\|\bwh^{(k)}\|_1, \|\bw^{(k)}\|_1\} }
\leq& \left(\sum_{l=1}^d \frac{\|\bw^{(l)} - \bwh^{(l)}\|_1}{\max\{\|\bwh^{(l)}\|_1, \|\bw^{(l)}\|_1\}}\right)\\
\leq& \sum_{l=1}^d \frac{\|\bw^{(l)} - \bwh^{(l)}\|_1}{\|\bw^{(l)}\|_1},   
\end{align}
which is compatible with the rest of our results.
The proof of Theorem~\ref{thm:modifiedTheorem1} is almost identical to the proof of Theorem~\ref{thm:l1bound}.
\begin{proof}[Proof of Theorem~\ref{thm:modifiedTheorem1}]
Since Eq.~\ref{eq:norm of x} still holds without the additional condition $\|\bw\|_1\geq \|\bwh\|_1$,
\begin{align}
& \|f(\bx; \bw^{(1:d)}) - f(\bx; \bwh^{(1:d)})\|_1\nonumber\\
\leq& \|\bw^{(d)} - \bwh^{(d)}\|_1 \cdot  \prod_{l=1}^{d-1} \|\bw^{(l)}\|_1 
    + \|\bwh^{(d)}\|_1 \cdot \| f(\bx; \bw^{(1:d-1)}) - f(\bx; \bwh^{(1:d-1)})\|_1\label{eq:norm of x modified}\\
\leq& \|\bw^{(d)} - \bwh^{(d)}\|_1 \cdot \prod_{l=1}^{d-1} \max\{\| \bw^{(l)}\|_1, \|\bwh^{(l)}\|_1\} 
    + \max\{\|\bw^{(d)}\|_1, \|\bwh^{(d)}\|_1\} \cdot \| f(\bx; \bw^{(1:d-1)}) - f(\bx; \bwh^{(1:d-1)})\|_1, \label{eq:w leq what modified}
\end{align}
which implies
\begin{align}
& \frac{\|f(\bx; \bw^{(1:d)}) - f(\bx; \bwh^{(1:d)})\|_1}{\prod_{l=1}^{d} \max\{\| \bw^{(l)}\|_1, \|\bwh^{(l)}\|_1\}} \nonumber\\
\leq& \frac{\|\bw^{(d)} - \bwh^{(d)}\|_1}{\max\{\|\bw^{(d)}\|_1, \|\bwh^{(d)}\|_1\}}
    + \frac{\| f(\bx; \bw^{(1:d-1)}) - f(\bx; \bwh^{(1:d-1)})\|_1}{\prod_{l=1}^{d-1} \max\{\| \bw^{(l)}\|_1, \|\bwh^{(l)}\|_1\}}.
    \label{eq:symmetric bound step}
\end{align}
Similar to the proof of Theorem~\ref{thm:l1bound}, we can recursively apply the above inequality to obtain Eq.~\ref{eq:symmetric bound}.
\end{proof}

\subsection{Density Estimation for Neural Network Parameters without Normalization}
\label{density_appendix}
   \begin{figure*}[h!]
        \centering %
        \includegraphics[width=.51\textwidth]{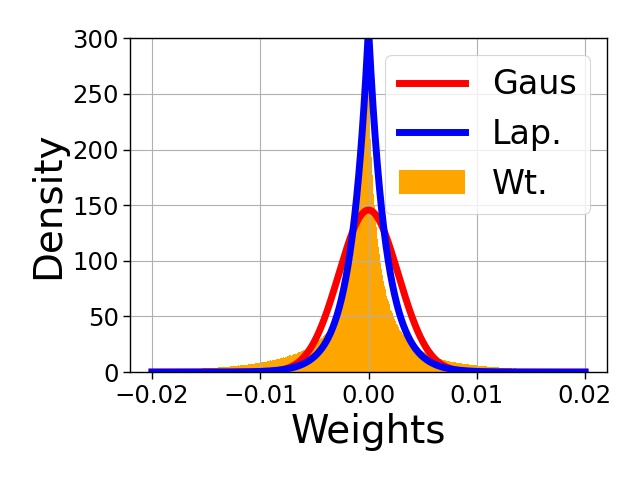}
    \caption{Weight Density of ResNet-18 (trained on CIFAR-10) before normalization.
    }\label{fig:density_resnet18}
   \end{figure*}
In Section~\ref{method1_rate_distortion}, we have justified our assumption of Laplacian distribution over normalized NN weights through density plots for three distinct architectures. We have also emphasized that Laplacian would be a good fit for unnormalized NN weights as well. We give the density plots of unnormalized weights of ResNet-18 in Figure~\ref{fig:density_resnet18} to justify our claim empirically. This claim implies that SuRP can also be applied to NNs without normalization and it would achieve rate-distortion theoretic optimal performance for reconstructing the NN weights back. However, recall from Theorem~\ref{thm:l1bound} that $\ell_1$ distortion of normalized weights upper bounds the output perturbation. Since we care more about maintaining the outputs rather than the weights themselves, we have applied SuRP after the normalization. 

We have additionally observed that weights of layers closer to the input tend to follow a Gaussian-like distribution.
In contrast, the weights of layers closer to the output behave like Laplacian random variables.
Since the last layers have larger number of parameters in the architectures used in this work, we see a Laplacian distribution over the weights globally. Therefore, different pruning strategies might be necessary for layers with Gaussian and Laplacian behaviour for a layer by layer pruning approach.

\subsection{Proof of Lemma~\ref{rd_lap}}
\label{lemma_proof_appendix}
In this section, we briefly describe the proof outline of Lemma~\ref{rd_lap} in Section~\ref{method1_rate_distortion}, which is provided in \citep{berger2003rate}. Consider the Laplacian source $U\sim P_{U}$ with parameter $\lambda$, and the target distortion $D$ satisfies $0\leq D\leq 1/\lambda$.
Then,
\begin{align*}
R(D) =& \min_{\mathbb{E}[d(U, \Uh)]\leq D} I(U;\Uh)\\
=& \inf_{\mathbb{E}[|U-\Uh|]\leq D} I(U;\Uh).
\end{align*}
Let $Q$ be another conditional distribution where $Q_{U|\Uh}(u|\uh) = \frac{1}{D}e^{-|u-\uh|/D}$.
Then, 
\begin{align}
I(U;\Uh) =& D_{KL}(P_{U|\Uh} \| P_U|P_{\Uh})\nonumber\\
=& D_{KL}(P_{U|\Uh} \|Q_{U|\Uh} | P_{\Uh})
+ \mathbb{E}_{P_{U, \Uh}} \left[\log\frac{Q_{U|\Uh}(U|\Uh)}{P_{U}(U)}\right]\nonumber\\
\geq& \mathbb{E}_{P_{U, \Uh}} \left[\log\frac{Q_{U|\Uh}(U|\Uh)}{P_{U}(U)}\right]\label{app:KL}\\
=& - \log(\lambda D)  - \frac{1}{D}\mathbb{E}[|U-\Uh|] + \lambda \mathbb{E}[|U|]\nonumber\\
\geq& -\log(\lambda D) \label{app:l1distance}
\end{align}
where Eq.~\ref{app:KL} is due to nonnegativity of KL divergence, and Eq.~\ref{app:l1distance} is because
$\mathbb{E}[|U|] = \frac{1}{\lambda}$ and $\mathbb{E}[|U-\Uh|]\leq D$.
This implies that $R(D) \geq -\log (\lambda D)$.
We note that we followed a technique inspired by Verdu's proof for rate-distortion function of exponential source \citep{verdu1996exponential}. The same lower bound can also be achieved via Shannon lower bound (SLB) \citep{shannon1959coding}.

On the other hand, we need to show that the lower bound $R(D)\geq -\log (\lambda D)$ is indeed tight.
Let $V$ be a mixture of point measure and Laplacian random variable, where the probability density function is given by
\begin{align*}
P_{V}(v) = \lambda^2 D^2 \cdot\delta(v) + (1-\lambda^2 D^2)\cdot \frac{\lambda}{2} e^{-\lambda |v|}.
\end{align*}
We further let $N$ be a Laplacian random variable with parameter $1/D$, where $V$ and $N$ are independent.
Then, the Laplace transform of $P_{V}$ and $P_{N}$ are given by
\begin{align*}
\mathbb{E}[e^{-sV}] =& \lambda^2 D^2 + (1-\lambda^2 D^2) \frac{\lambda^2}{\lambda^2+s^2}\\
\mathbb{E}[e^{-sN}] =& \frac{1/D^2}{1/D^2+s^2}.
\end{align*}
Consider the sum of two random variables $V+N$.
Since they are independent, Laplace transform of the density of $V+N$ is a product of the above two terms.
\begin{align*}
\mathbb{E}[e^{-s(V+N)}] =& \mathbb{E}[e^{-sV}] \cdot \mathbb{E}[e^{-sN}]\\
=& \frac{\lambda^2}{\lambda^2+s^2}.
\end{align*}
Since it coincides with the Laplace transform of $P_{U}$, we conclude that $U \stackrel{(d)}{=} V+N$.
Thus, by letting $U=V+N$, we obtain the conditional distribution $Q_{U|V}(u|v) = \frac{1}{D} e^{-|u-v|/D}$.
It is clear that $Q_{U|V}$ satisfies the equality conditions in Eq.~\ref{app:KL} and Eq.~\ref{app:l1distance},
and therefore it achieves the lower bound $I(U;\Uh) = -\log (\lambda D)$ with $\Uh=V$.

To sum, the optimal rate-distortion tradeoff is $R(D) = -\log (\lambda D)$ and it can be achieved with a reconstruction that follows
\begin{align}
\begin{aligned}
P_{V}(v) = \lambda^2 D^2 \cdot\delta(v) + (1-\lambda^2 D^2)\cdot \frac{\lambda}{2} e^{-\lambda |v|}.
\end{aligned}
\label{marginal_app}
\end{align}

\subsection{Algorithms}
\label{algorithms_exponential_appendix}
    \begin{algorithm}[!h]
        {\bf Hyperparameters:} $\beta$ \\
        {\bf Inputs:} weights $w_1, \dots, w_n$ in $d$ layers\\
        %{\bf Inputs:} Laplacian sequence $u_1, \dots, u_n$ with $n$ symbols\\
        {\bf Output:} reconstructed weights $w^{recon}_1, \dots, w^{recon}_n$ \\
    %\vspace{-.2in}
            \begin{algorithmic}
            \STATE{\underline{\textbf{Normalization:}}}
            \FOR{$l=1,\ldots,d$} 
            \STATE{$u^{(l)}\gets \frac{w^{(l)}}{\|w^{(l)} \|_1}$}
            \ENDFOR \\
            \STATE
            %\STATE{\underline{\textbf{Compression:}}}
            \STATE{$(u^{recon}_1, \dots, u^{recon}_n) \gets 0$}
            \STATE{$\lambda \gets \mathsf{ParamEst}((u_1,\dots,u_n))$}
            \STATE{$\mathsf{Encoder}$ sends $\lambda$ to the $\mathsf{Decoder}$}. 
            \FOR{$t=1,\ldots,L$} 
            \STATE{\hspace{.15in}\underline{\textbf{Encoder:}}}
            \STATE{$m_{max} \gets$ (indices of the components in $(u_1,\dots,u_n)$ that are larger than $\frac{1}{\lambda} \cdot \log{\frac{n}{2 \beta}}$.)}
            \STATE{$m_{min} \gets$ (indices of the components in $(u_1,\dots,u_n)$ that are smaller than $-\frac{1}{\lambda} \cdot \log{\frac{n}{2 \beta}}$.)}
            \IF{$m_{max}$ or $m_{min}$ is empty}
            \STATE{$\lambda \gets \mathsf{ParamEst}((u_1,\dots,u_n))$}
            \STATE{sends $\lambda$ to the $\mathsf{Decoder}$}.
            \ENDIF
            \STATE $m_1 \gets$ (a random index from $m_{max}$)
            \STATE $m_{-1} \gets$ (a random index from $m_{min}$)
            \STATE {sends $m_1$ and $m_{-1}$ to the Decoder.}
            \STATE {$u_{m_1} = u_{m_1} - \frac{1}{\lambda} \cdot \log{\frac{n}{2 \beta}}$}
            \STATE {$u_{m_{-1}} = u_{m_{-1}} + \frac{1}{\lambda} \cdot \log{\frac{n}{2 \beta}}$}
            \STATE  {$\lambda\gets \frac{n}{n-2\log{\frac{n}{2 \beta}}} \cdot\lambda$}\\
            \hspace{.15in} \underline{\textbf{Decoder:}}
            \STATE {receives $m_1$ and $m_{-1}$ from the Encoder.}
            \STATE {$u^{recon}_{m_1} = u^{recon}_{m_1} + \frac{1}{\lambda} \cdot \log{\frac{n}{2 \beta}}$}
            \STATE {$u^{recon}_{m_{-1}} = u^{recon}_{m_{-1}} - \frac{1}{\lambda} \cdot \log{\frac{n}{2 \beta}}$}
            \STATE  {$\lambda\gets \frac{n}{n-2\log{\frac{n}{2 \beta}}} \cdot\lambda$}
            \ENDFOR \\
            \STATE {$w^{recon}_1, \dots, w^{recon}_n \gets$ (denormalize $u^{recon}_1, \dots, u^{recon}_n$.)}
            \STATE
            \STATE $\mathsf{ParamEst}((u_1,\dots,u_n)):$
            \STATE \hspace{.15in}{$1/ \lambda \gets$ mean of $(|u_1|, \dots, |u_n|)$}
            %\STATE \hspace{.15in}{$\lambda \gets \alpha \cdot \lambda$}
            \STATE \hspace{.15in}{\bf return} $\lambda$  
            \end{algorithmic}
        \caption{SuRP}
        \label{algorithm_laplacian_appendix}
    % \vspace{-1mm}
    \end{algorithm}

    \begin{algorithm}[!h]
        {\bf Hyperparameters:} $\beta$ \\
        {\bf Inputs:} weights $w_1, \dots, w_n$ in $d$ layers\\
        %{\bf Inputs:} Laplacian sequence $u_1, \dots, u_n$ with $n$ symbols\\
        {\bf Output:} reconstructed weights $w^{recon}_1, \dots, w^{recon}_n$ \\
    %\vspace{-.2in}
            \begin{algorithmic}
            \STATE{\underline{\textbf{Normalization:}}}
            \FOR{$l=1,\ldots,d$} 
            \STATE{$u^{(l)}\gets \frac{|w^{(l)|}}{\|w^{(l)} \|_1}$}
            \ENDFOR \\
            \STATE
            \STATE{$(u^{recon}_1, \dots, u^{recon}_n) \gets 0$}
            \STATE{$\lambda \gets \mathsf{ParamEst}((u_1,\dots,u_n))$}
            \STATE{$\mathsf{Encoder}$ sends $\lambda$ to the $\mathsf{Decoder}$}. 
            \FOR{$t=1,\ldots,L$} 
            \STATE{\hspace{.15in}\underline{\textbf{Encoder:}}}
            \STATE{$m_{inds} \gets$ (indices of the components in $(u_1,\dots,u_n)$ that are larger than $\frac{1}{\lambda} \cdot \log{\frac{n}{\beta}}$.)}
            \IF{$m_{inds}$ is empty}
            \STATE{$\lambda \gets \mathsf{ParamEst}((u_1,\dots,u_n))$}
            \STATE{sends $\lambda$ to the $\mathsf{Decoder}$}.
            \ENDIF
            \STATE $m \gets$ (a random index from $m_{inds}$)
            \STATE {sends $m$ to the Decoder.}
            \STATE {$u_m = u_m - \frac{1}{\lambda} \cdot \log{\frac{n}{\beta}}$}
            \STATE {$\lambda\gets \frac{n}{n-\log{\frac{n}{\beta}}} \cdot\lambda$} \\
            \hspace{.15in} \underline{\textbf{Decoder:}}
            \STATE {receives $m$ from the Encoder.}
            \STATE {$u^{recon}_m = u^{recon}_m + \frac{1}{\lambda} \cdot \log{\frac{n}{\beta}}$}
            \STATE {$\lambda\gets \frac{n}{n-\log{\frac{n}{\beta}}} \cdot\lambda$}
            \ENDFOR \\
            \STATE {$w^{recon}_1, \dots, w^{recon}_n \gets$ (denormalize $u^{recon}_1, \dots, u^{recon}_n$ and add sign bits.)}
            \STATE
            \STATE $\mathsf{ParamEst}((u_1,\dots,u_n)):$
            \STATE \hspace{.15in}{$1/ \lambda \gets$ mean of $(u_1, \dots, u_n)$}
            %\STATE \hspace{.15in}{$\lambda \gets \alpha \cdot \lambda$}
            \STATE \hspace{.15in}{\bf return} $\lambda$  
            \end{algorithmic}
        \caption{SuRP-modified}
        \label{algorithm_exponential_appendix}
    % \vspace{-1mm}
    \end{algorithm}

We give the algorithm described in Section~\ref{method2_SuRP} in Algorithm~\ref{algorithm_laplacian_appendix}. For the experiments in Section~\ref{experiments}, we slightly modified Algorithm~\ref{algorithm_laplacian_appendix} and used Algorithm~\ref{algorithm_exponential_appendix}. 

As mentioned in Section~\ref{experiments}, these two algorithms are equivalent except the fact that Algorithm~\ref{algorithm_exponential_appendix} applies the same compression scheme after taking the absolute value of the normalized weights. Furthermore, Algorithm~\ref{algorithm_exponential_appendix} is rate-distortion theoretic optimal too. To see this, it is enough to follow the same steps in Sections~\ref{method1_rate_distortion} and~\ref{method2_SuRP} for exponential source instead of Laplacian source since the magnitude of Laplacian source sequence follows an exponential distribution. We now give the rate-distortion function for exponential source (magnitude of normalized weights). We consider i.i.d. exponential source sequence $u_1, \dots, u_n$ with distribution $f_{exp}(u;\lambda) = \lambda e^{-\lambda u}$ for $u \geq 0$, reconstruction $v_1, \dots, v_n$, and one-sided $\ell_1$ distortion given by:
\begin{align*}
    d(u, v) = \begin{cases} u - v, & \text{if} \ u \geq v \\
                    \infty,  & \text{otherwise.}\end{cases}
\end{align*}    
Then, the rate-distortion function is given in Lemma~\ref{lem:rd_exp}:

\begin{lem}{\citep{verdu1996exponential}} The rate-distortion function for an exponential source with one-sided distortion is given by
\begin{align}
\begin{aligned}
    R(D) = \begin{cases} -\log(\lambda D), & \ 0 \leq D \leq \frac{1}{\lambda} \\
    0, &  D > \frac{1}{\lambda}\end{cases}
\end{aligned}
\label{eq_rd_exp}
\end{align}
with the following optimal conditional probability distribution that achieves the minimum mutual information
\begin{align}
\begin{aligned}
    f_{\bU|\bV}(u|v) = \begin{cases}\frac{1}{D} e^{-(u-v)/D},& \mbox{ $u \geq v \geq 0$}\\
    0, & \mbox{ otherwise}.
    \end{cases}
\end{aligned}
\label{eq_cond_dist_exp}
\end{align}
Moreover, the marginal distribution of $\bV$ is as follows
\begin{align}
\begin{aligned}
f_{\bV}(v) = \lambda D \cdot \delta (v) + (1-\lambda D) \cdot \lambda e^{-\lambda v}
\end{aligned}
\label{eq:marginal_exp}
\end{align}
where $\delta(v)$ is a Dirac measure at 0.
\label{lem:rd_exp}
\end{lem}
%\vspace*{-1.85\baselineskip}

Proof of Lemma~\ref{lem:rd_exp} can be found in \citep{verdu1996exponential}. It is clear to see from Eq.s~\ref{eq_cond_dist_exp} and~\ref{eq:marginal_exp} that exponential source has the same nice properties as Laplacian:
\begin{enumerate}
    \item It suggests pruning as an essential step in a good compression algorithm since Eq.~\ref{eq:marginal_exp} is a sparse distribution.
    \item It is successively refinable, allowing for a both practical and rate-distortion theoretic optimal algorithm (see Algorithm~\ref{algorithm_exponential_appendix}). 
\end{enumerate}

Following the same steps in Section~\ref{method2_SuRP}, it can be proven that Algorithm~\ref{algorithm_exponential_appendix}, which we used in our experiments, is zero-rate optimal with $\beta = \log{n}$.

\subsubsection{Effect of $\beta$}
\label{appedix_beta}
Table~\ref{tab:experiment_beta_neurips} shows the effect of the hyperparameter $\beta$ on the model accuracy, the number of SuRP iterations needed to achieve the desired sparsity, and the number of required refreshment for the Laplacian parameter $\lambda$. We perform one-shot pruning experiments without retraining, i.e., apply SuRP once, with different $\beta$ values as shown in the table. We can consider the accuracy as a measure of distortion; and 'the number of iterations' and 'the number of parameter  refreshments' as a measure of rate. More concretely, we first fix a target sparsity for all the experiments. If the number of iterations to achieve this sparsity is large, then the compression amount is small since higher number of indices represents the same model. During one running of SuRP, there might be a need for re-estimating the parameter $\lambda$ of the underlying Laplacian distribution, as we stated in Section~\ref{sec:surp}. This is an undesirable situation since this requires the encoder to send the re-estimated parameter $\lambda$ to the decoder; we call this a "refreshment". The numbers in this table verify our theoretical analysis that as $\beta$ increases, the number of refreshments becomes negligible compared to the total number of iterations (see Eq.~\ref{eq:prob}). However, for very large $\beta$ such as $\beta=(\log n)^2$, zero-rate optimality is not as strong as $\beta = \log n$ (see Theorem~\ref{thm_zero_rate}). This can also be seen from the table since the number of iterations needed is significantly larger for $\beta=(\log n)^2$, indicating that the bit-rate is large and we do not compress the model much. Since the accuracy is similar across different $\beta$ values, we can conclude that $c \cdot \log n$ is indeed a reasonable choice for $\beta$ as it balances the two factors (number of iterations and number of refreshments) that contribute to the rate.

\setlength{\tabcolsep}{3pt}
\begin{table*}[!h]
%\vspace{-5mm}
\centering
\caption{Effect of the choice of $\beta$ on the model accuracy, the number of SuRP iterations needed to achieve the desired sparsity, and the number of required refreshment for the Laplacian parameter $\lambda$. The experiments are run with VGG-16 model on CIFAR-10 dataset. The sparsity is $95 \%$ in all cases. Note that these experiments do not involve multiple pruning steps or retraining, that is why the accuracy is slightly smaller than the numbers in Table~\ref{tab:experiment_cifar}. 
}
\begin{tabular}{lccccc}
\toprule
\textbf{$\beta$ \ :}                   & $\sqrt{\log n}$   & $1/2 \cdot \log n$ & $\log n$ & $2 \cdot \log n$ & $(\log n)^2$ \\ \midrule
Accuracy                            & $89.2\%$            & $90.02\%$            & $90.02\%$        & $90.00\%$           & $90.00\%$\\
Required Number of Iterations       & $1.1M$            & $1.2M$             & $1.2M$           & $1.3M$             & $20M$\\ 
Number of Parameter Refreshment     & $20K$             & $500$              & $22$             & $20$              & $10$ \\ \bottomrule
\\
\end{tabular}

\label{tab:experiment_beta_neurips}
\vspace{-2mm}
\end{table*}

\subsection{Visualization of SuRP}
\label{succ_pruning_app}%

Figure~\ref{fig:no_retrain} shows the decreasing $\ell_1$ distortion and sparsity, and increasing accuracy of the reconstructed model through the iterations (running inside SuRP). We note that SuRP is applied only once in Figure~\ref{fig:no_retrain} and iterations correspond to the iterations running inside SuRP. However, as stated in Section~\ref{experiments}, we adopted iterative pruning approach, where each pruning iteration corresponds to running SuRP one time. After each pruning iteration, SuRP outputs a sparse model, and a retraining procedure is applied to the sparse model. When we do iterative pruning, we apply SuRP several times by increasing the target sparsity ratio every time. For instance, let us say we want to prune $90\%$ of the parameters in the first pruning iteration. Then, as shown in Figure~\ref{fig:phase_1}, SuRP stops once the sparsity ratio drops to $90\%$. Before starting the next iteration (next round of SuRP), we retrain the sparse model by excluding the pruned parameters (as proposed in \citep{deep_compression}). In the next iteration, as shown in Figure~\ref{fig:phase_2}, the sparsity can never be lower than $90\%$ no matter how long we run the algorithm since $90\%$ of the parameters are already pruned in the previous iteration of the pruning. As we typically desire a higher sparsity ratio in the later iterations, we need to stop SuRP at the target sparsity (which is higher than $90\%$).   

\begin{figure*}[h!]
        \centering %
        \subfigure[$\ell_1$ Distortion.]{\includegraphics[width=.31\textwidth]{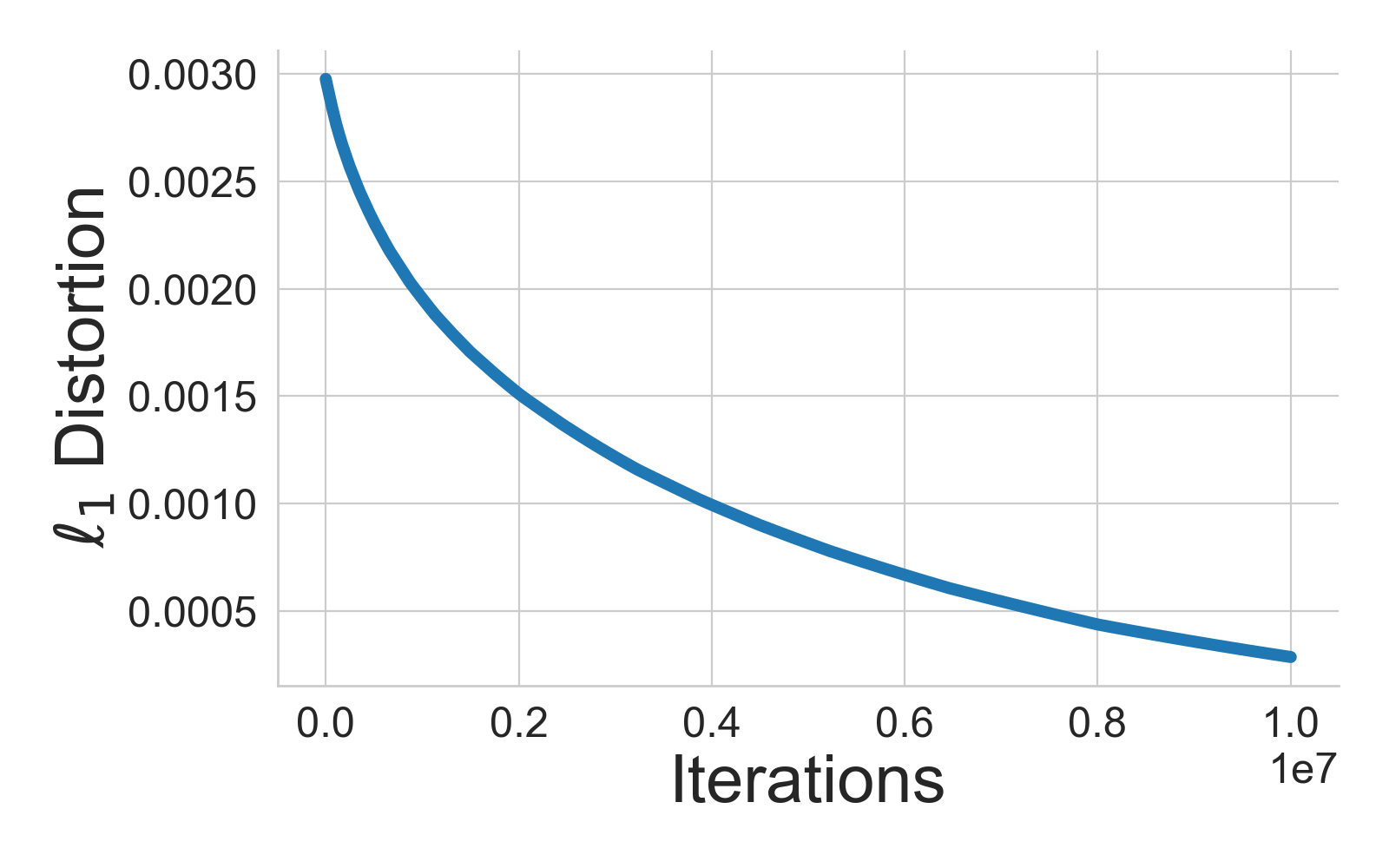}}
        \subfigure[Sparsity.]{\includegraphics[width=.31\textwidth]{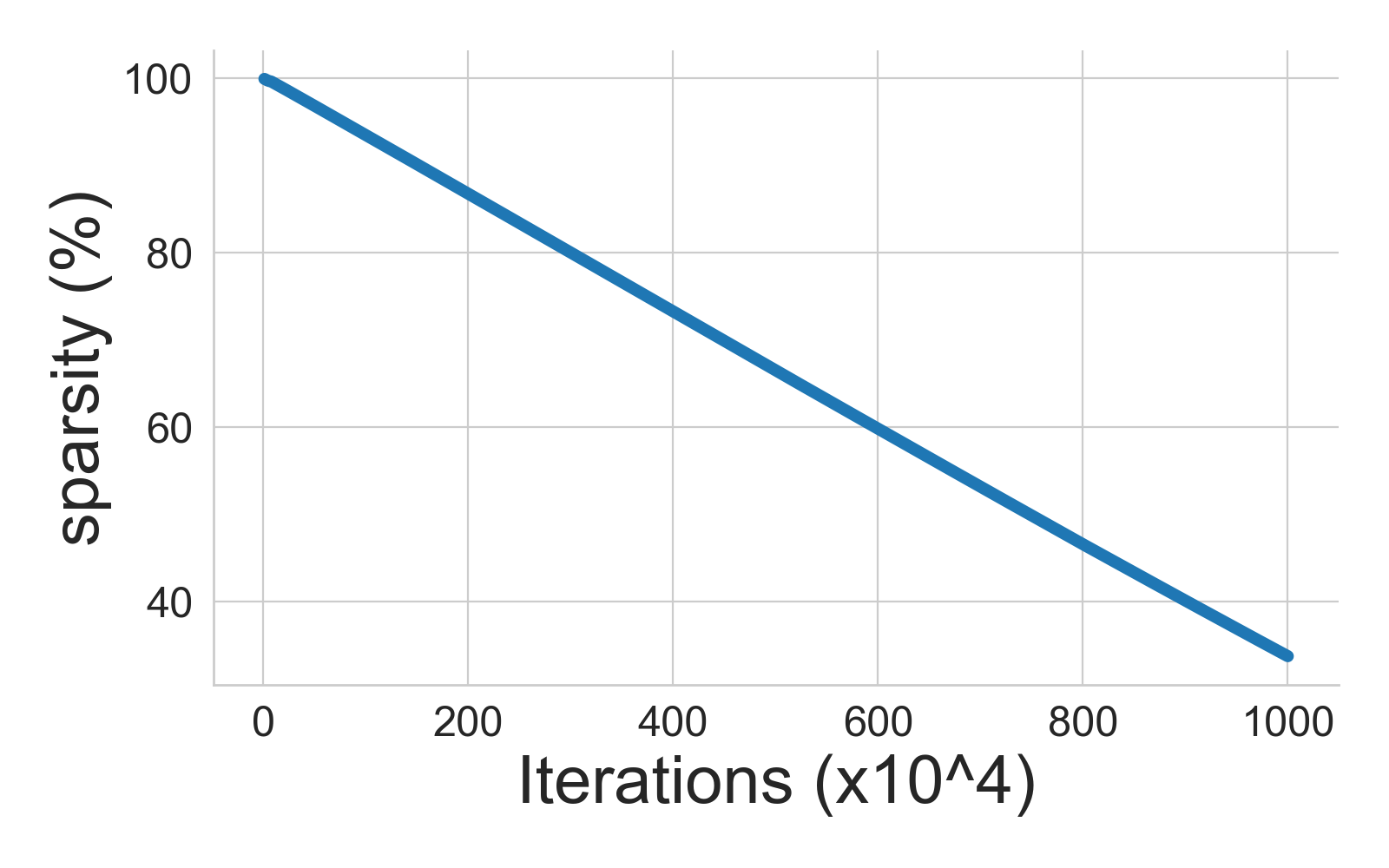}}
        \subfigure[Accuracy.]{\includegraphics[width=.31\textwidth]{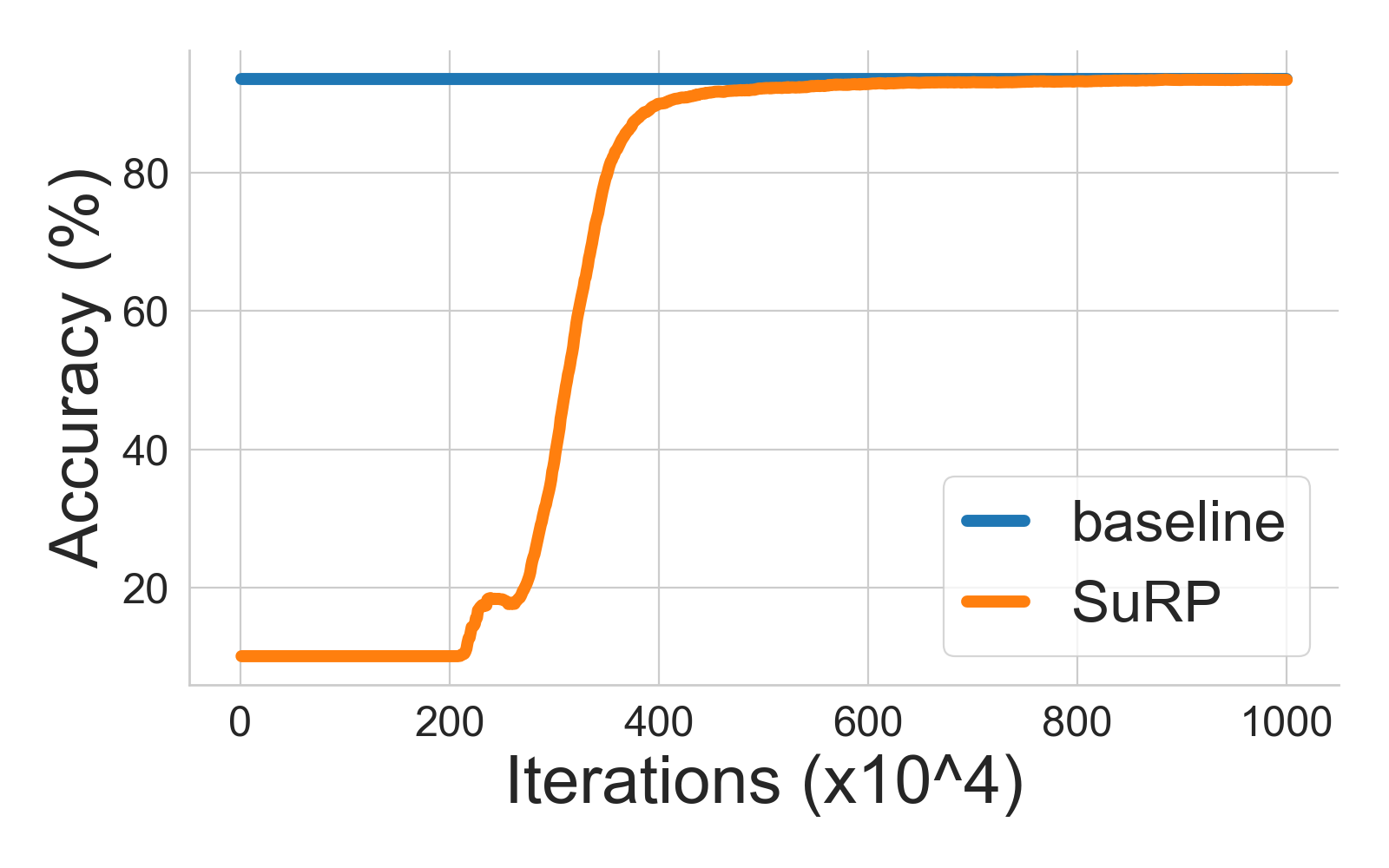}}
    \caption{(a) Average $\ell_1$ distortion, (b) sparsity and (c) accuracy of the reconstructed VGG-16 when SuRP is applied once (no iterative pruning). Baseline: fully-trained model without compression.}\label{fig:no_retrain}
   \end{figure*}
   
 \begin{figure*}[h!]
        \centering %
                \subfigure[$\ell_1$ Distortion.]{\includegraphics[width=.31\textwidth]{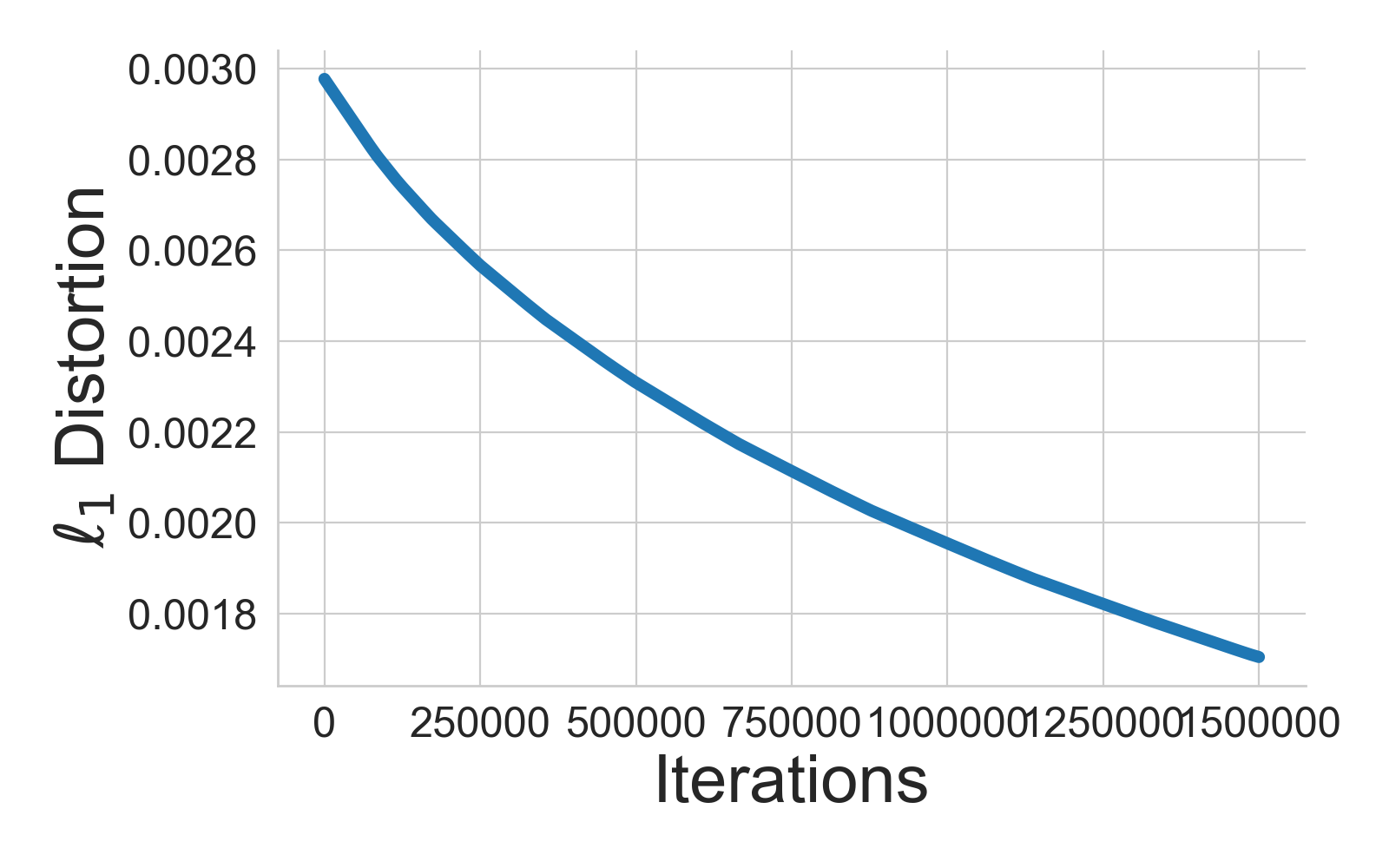}}
        \subfigure[Sparsity.]{\includegraphics[width=.31\textwidth]{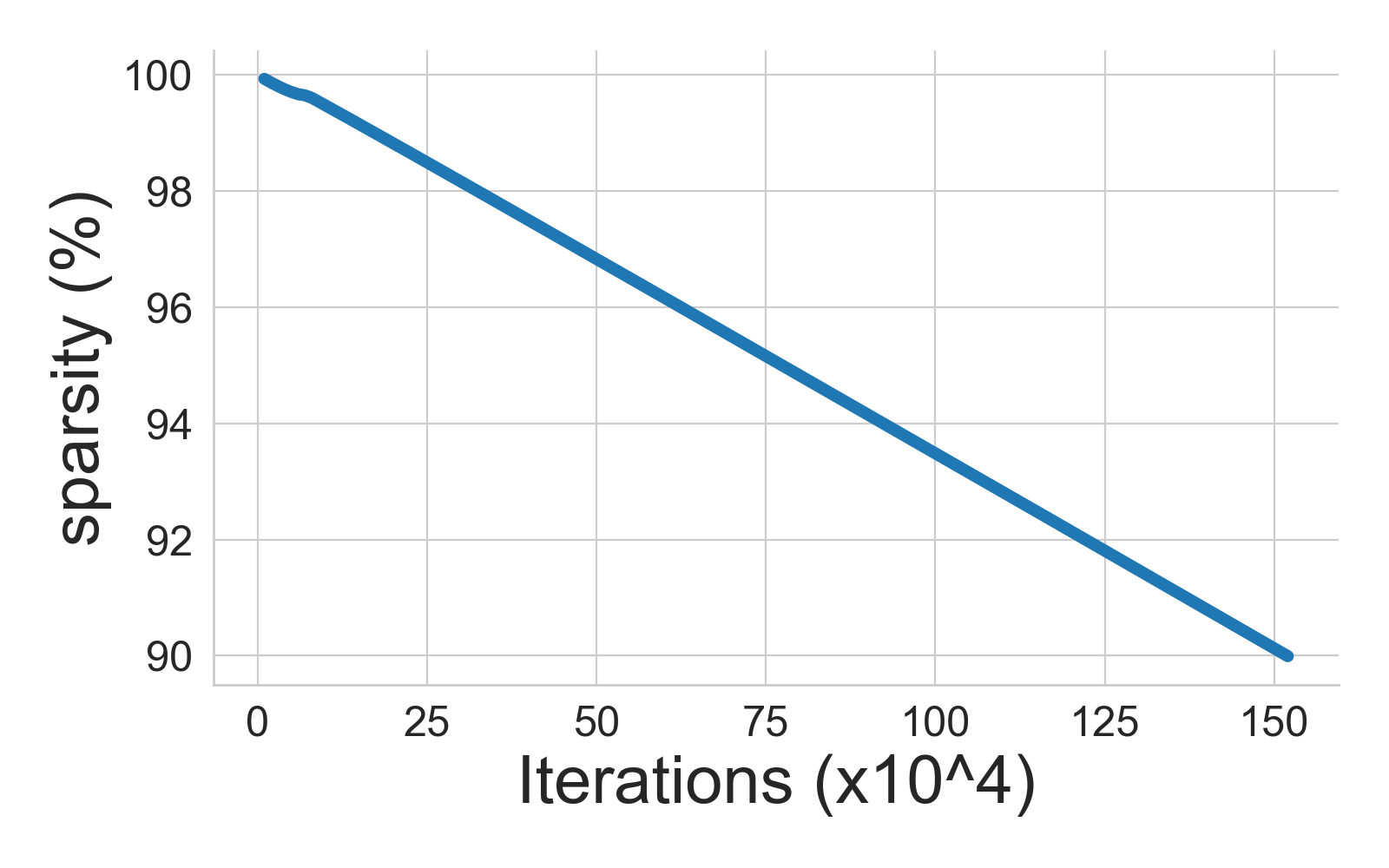}}
        \subfigure[Accuracy.]{\includegraphics[width=.31\textwidth]{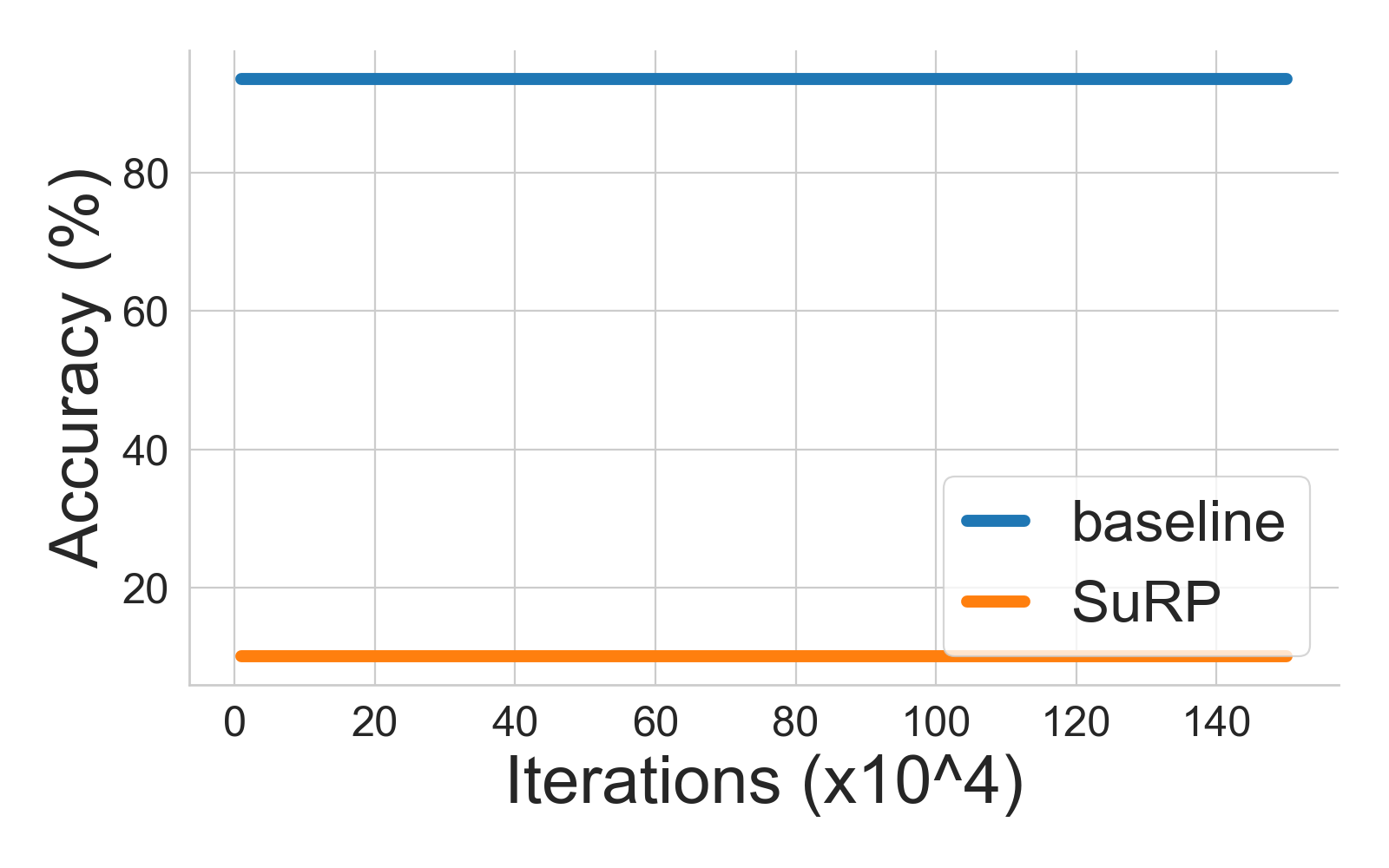}}
    \caption{(a) Average $\ell_1$ distortion, (b) sparsity and (c) accuracy of the reconstructed VGG-16 during SuRP (first iteration of the iterative pruning). SuRP stops at the desired sparsity $90\%$. Baseline: fully-trained model without compression.}\label{fig:phase_1}
   \end{figure*}
   
  \begin{figure*}[h!]
        \centering %
        \subfigure[$\ell_1$ Distortion.]{\includegraphics[width=.31\textwidth]{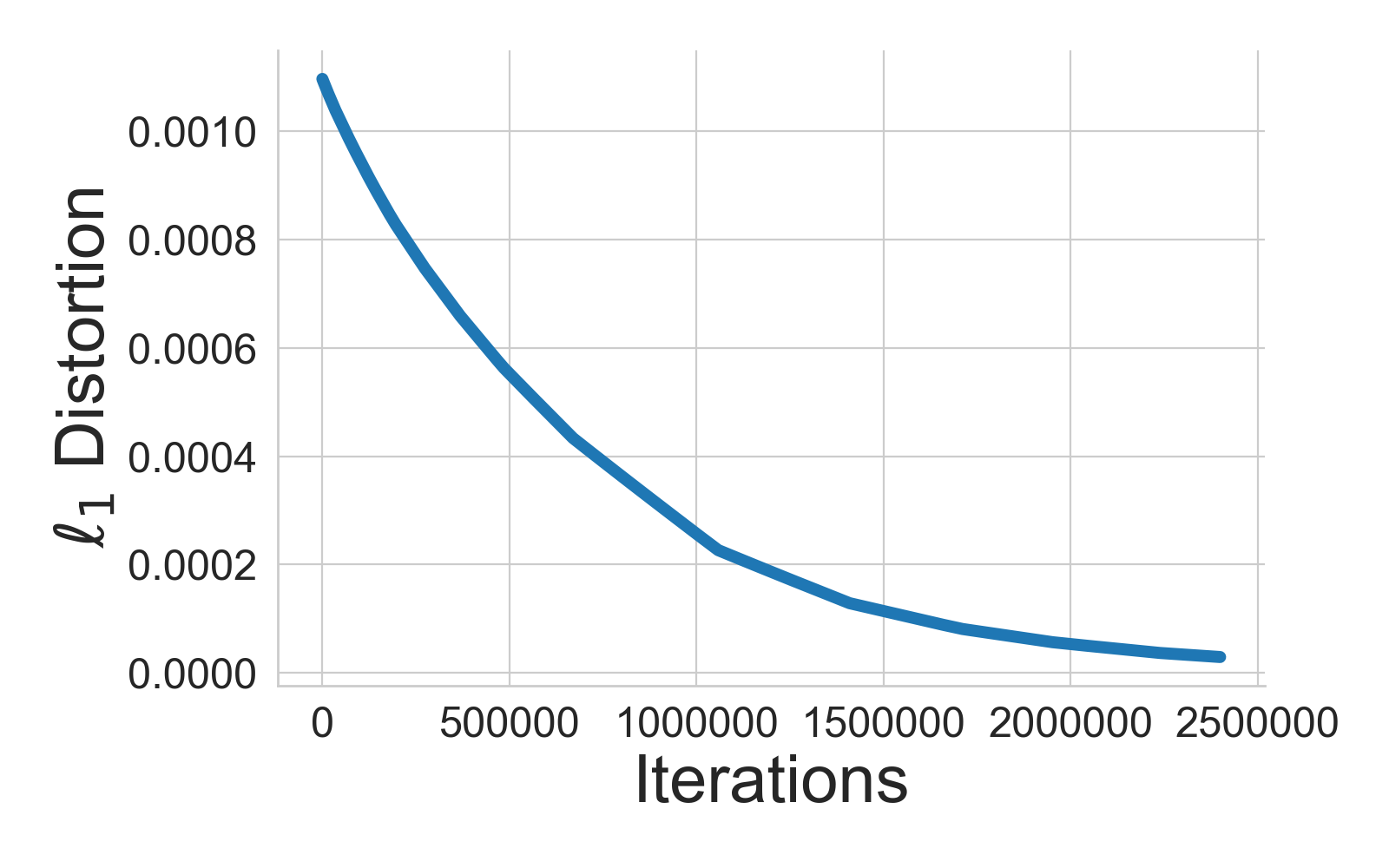}}
        \subfigure[Sparsity.]{\includegraphics[width=.31\textwidth]{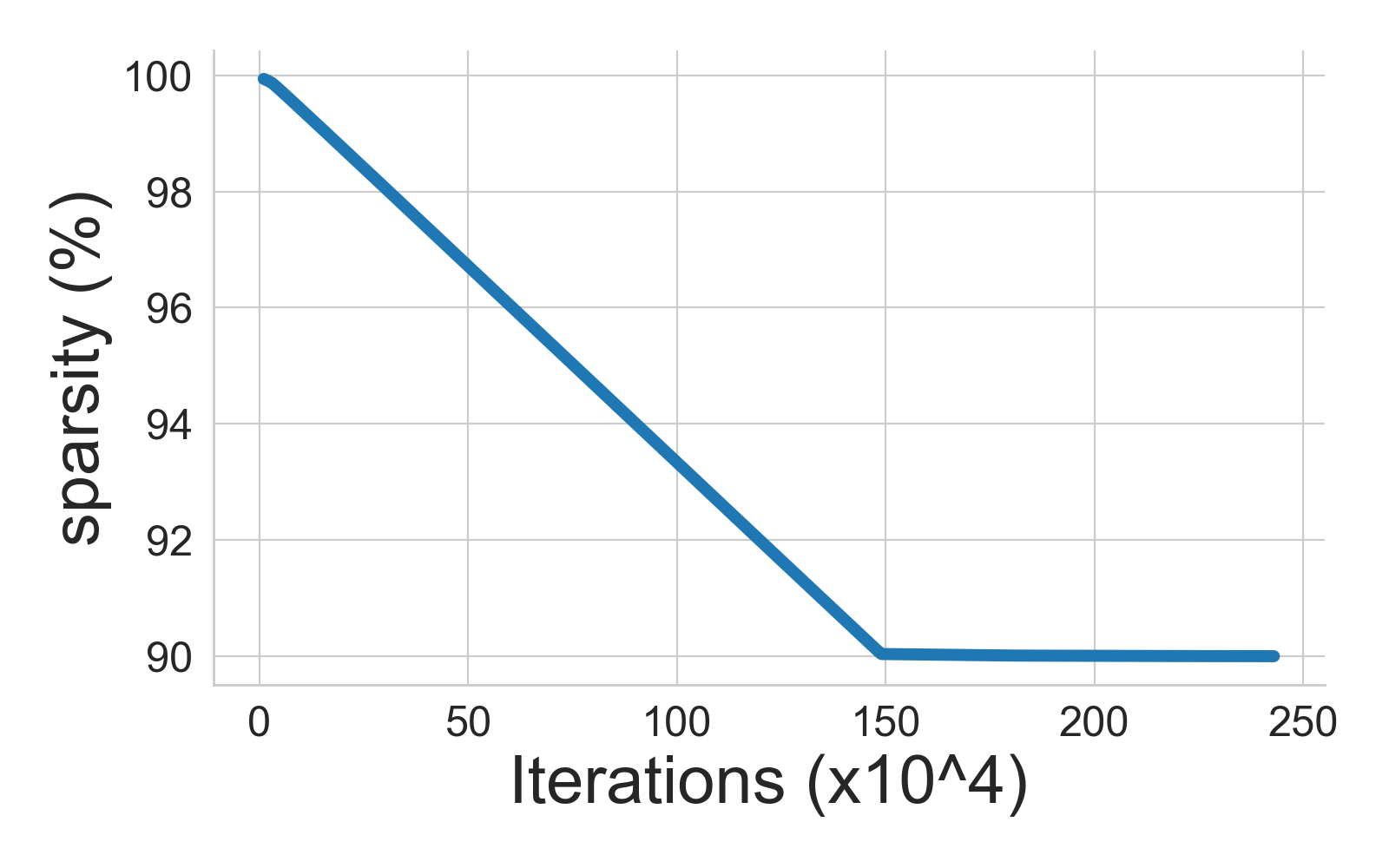}}
        \subfigure[Accuracy.]{\includegraphics[width=.31\textwidth]{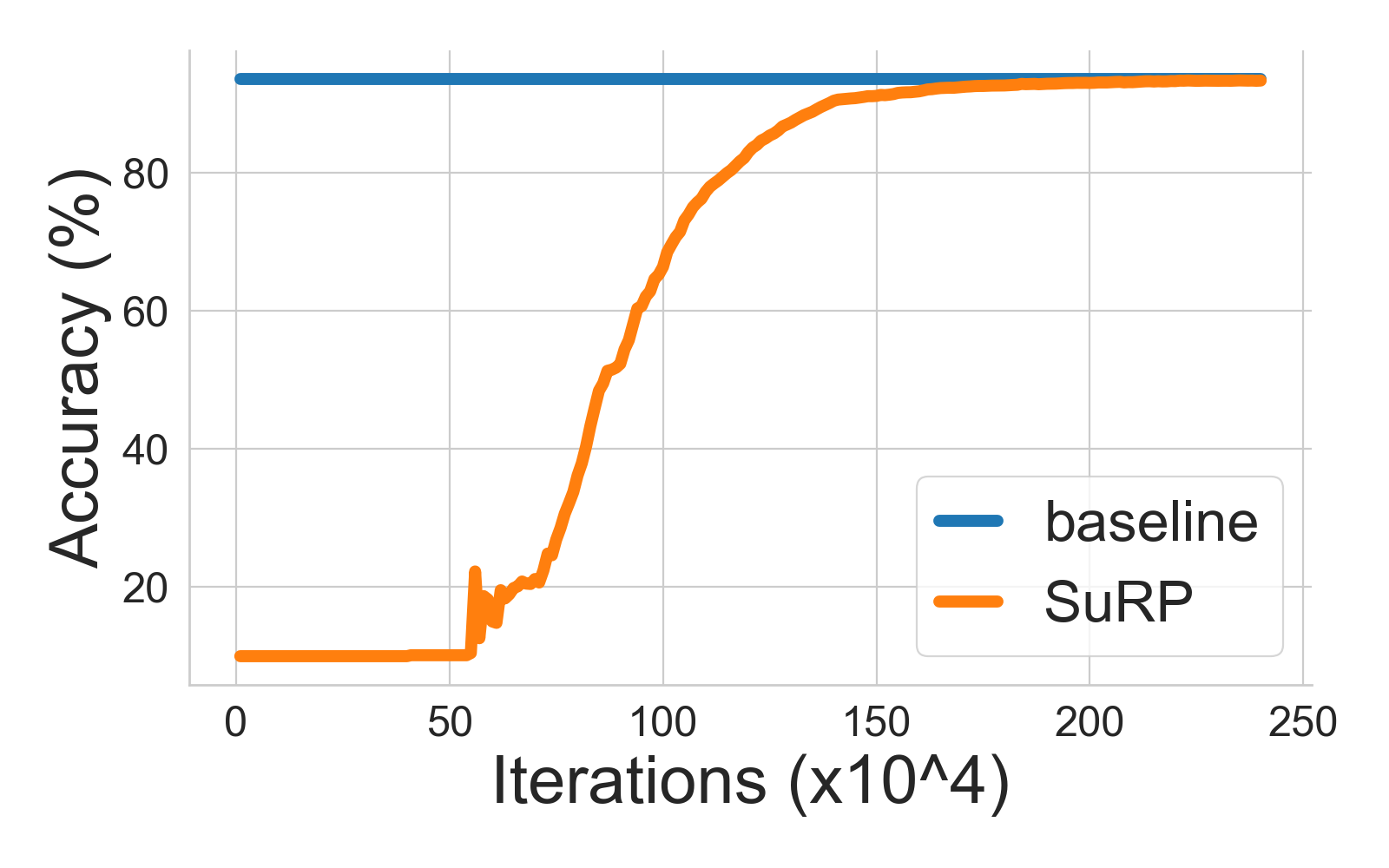}}
    \caption{(a) Average $\ell_1$ distortion, (b) sparsity and (c) accuracy of the reconstructed VGG-16 during SuRP (after the first iteration of the iterative pruning). The sparsity cannot be lower than $90\%$ and SuRP must stop at the desired sparsity (which is higher than $90\%$). Baseline: fully-trained model without compression. }\label{fig:phase_2}
   \end{figure*}

\subsection{Proof of Theorem~\ref{thm_zero_rate}}
\label{thm_zero_rate_proof_appendix}
In this section, we provide the proof of Theorem~\ref{thm_zero_rate}. In an iteration of SuRP, where $R_n = \frac{\log n(n-1)}{n}$ and $D_n = \frac{2}{n \lambda}\log \frac{n}{2\beta_n}$, we have
\begin{align*}
\frac{D_n} {R_n} & = -\frac{\frac{2}{\lambda}\log \frac{n}{2\beta_n}}{\log n(n-1)} \\
& = -\frac{1}{\lambda}\frac{\log n^2}{\log n(n-1)} + \frac{1}{\lambda}\frac{2\log 2\beta_n}{\log n(n-1)}.
\end{align*}
If $\lim_{n\rightarrow\infty}$ $\frac{\log 2\beta_n}{\log n(n-1)}=0$, it is clear that $\frac{D_n}{R_n}$ converges to $D'(0) = -\frac{1}{\lambda}$ as $n$ increases. Therefore, SuRP is zero-rate optimal under the condition that $\lim_{n\rightarrow\infty}$ $\frac{\log 2\beta_n}{\log n(n-1)}=0$.

\subsection{Optimizing the Bit Rate}
\label{bit_rate_appendix}
In this section, we highlight a useful byproduct of SuRP as a way to minimize the bit rate of the pruned model. Recall that SuRP requires transmitting two indices $i,j \in\{1,\dots,n\}$ from the encoder to the decoder for each iteration. This means that SuRP automatically gives the integer (indices are integers from $1, \dots, n$) representation of the model. Therefore, without dealing with floating points, i.e., precise values of the weights, we can reconstruct the model back using these indices. In order to further optimize this, we need a lossless compression scheme, namely entropy coding, to represent these indices as binary sequences. In information theory, the optimal entropy coding method can be found when the source distribution is known in advance \citep{huffman1952method}. Although there are universal codes that encode any source regardless of the distribution, they are preferable only when the source distribution is unknown since an entropy coding that matches the source distribution is always better than a universal code. Fortunately, our coding scheme for Laplacian (also for exponential) source induces a well-defined distribution that allows us to choose an optimal entropy coding method. Notice that randomly picking two indices $i,j$ from $\{k: \bU^{(t)}_k \geq \frac{1}{\lambda_{t}} \cdot \log{\frac{n}{2\beta}}\}$ and $\{k: \bU^{(t)}_k \leq - \frac{1}{\lambda_{t}} \cdot \log{\frac{n}{2\beta}}\}$ is equivalent to; (1) first randomly permuting $\bU^{(t)}$, and (2) selecting the minimum indices $i,j$ from $\{k: \bU^{(t)}_k \geq \frac{1}{\lambda_{t}} \cdot \log{\frac{n}{2\beta}}\}$ and $\{k: \bU^{(t)}_k \leq - \frac{1}{\lambda_{t}} \cdot \log{\frac{n}{2\beta}}\}$. The second approach induces a geometric  distribution under the i.i.d. assumption on the indices where small indices are always more probable to be selected. For geometric sources, there are two standard entropy coding methods: unary coding and Golomb coding \citep{golomb1966run, gallager1975optimal}. In our additional experiments in Appendix~\ref{experiments_app}, for comparing SuRP with other works on accuracy-bit rate tradeoff, we use Golomb coding. Now, we give more details on both methods.

\paragraph{Unary Coding.} Unary coding is a prefix-free code that is optimally efficient for the following geometric distribution:
\begin{align}
\begin{aligned}
    P_B(b)= 2^{-b}
\end{aligned}
\label{geo_2}
\end{align}
where $b$ is a positive integer. In the simplest term, unary coding encodes an integer $b$ with single $1$ followed by $b-1$ consecutive $0$'s. For instance, $72$ would be uniquely encoded as $100000010$. In our problem, indices follow the distribution in Eq.~\ref{geo_2} only when the fraction of normalized weights larger than $1/\lambda_t \cdot \log{\frac{n}{2 \beta}}$ in magnitude is exactly equal to $1/2$. Since this is not the case in every iteration, unary coding is not the optimal entropy coding method for indices in SuRP. 
\paragraph{Golomb Coding.} Golomb coding is an optimal prefix-free code for any geometric source, i.e., it is more general than unary coding. The construction of Golomb codes can be found in \citep{golomb1966run}. In our additional experiments in Appendix~\ref{experiments_app}, we implemented Golomb coding to represent NN models as binary arrays.

\subsection{Compression for Federated Learning}
\label{federated_learning_appendix}

\begin{figure*}[h!]
        \centering %
        \includegraphics[width=.31\textwidth]{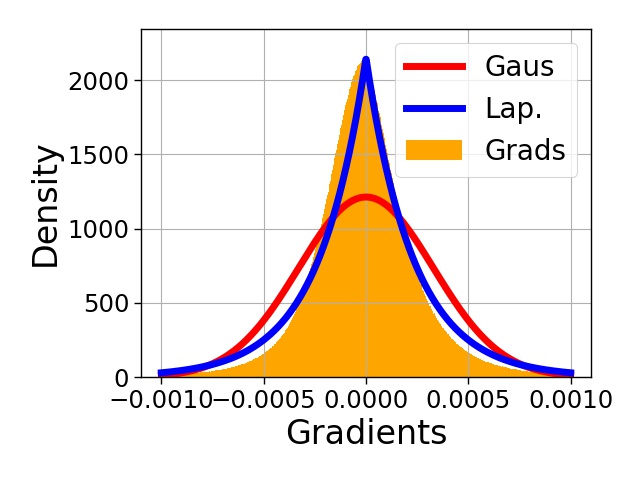}
    \caption{Density of gradients of ResNet-50 trained on ImageNet. We present only the gradients from the late stages of training since we use a pretrained ResNet-50.}\label{fig:density_grad_imagenet_app}
   \end{figure*}
   
\begin{figure*}[h!]
        \centering %
        \subfigure[Gradients at Early Stage.]{\includegraphics[width=.31\textwidth]{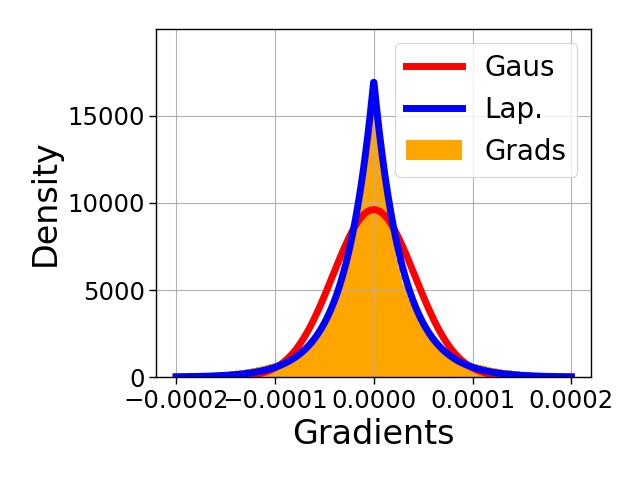}}
        \subfigure[Gradients at Mid Stage.]{\includegraphics[width=.31\textwidth]{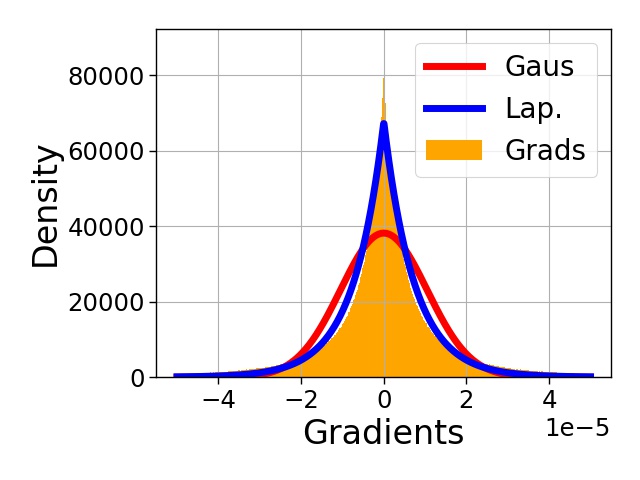}}
        \subfigure[Gradients at Late Stage.]{\includegraphics[width=.31\textwidth]{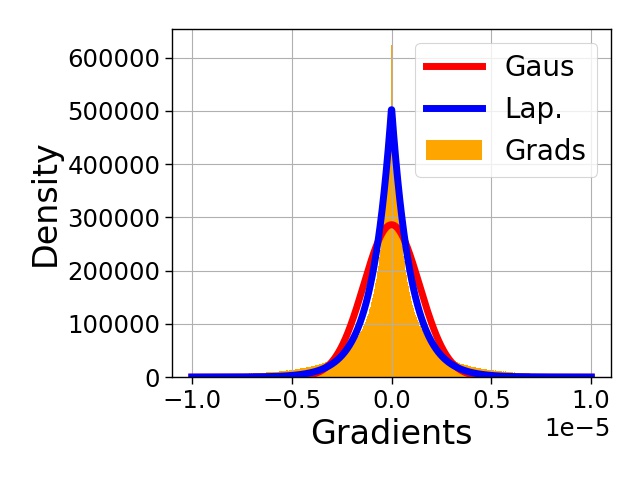}}
    \caption{Density of gradients of ResNet-18 trained on CIFAR-10 during (a) early stages of training (epoch 32), (b) middle stages of training (epoch 155), (c) late stages of training (epoch 336).}\label{fig:density_grad_resnet18_app}
   \end{figure*}   
   
\begin{figure*}[h!]
        \centering %
        \subfigure[Gradients at Early Stage.]{\includegraphics[width=.31\textwidth]{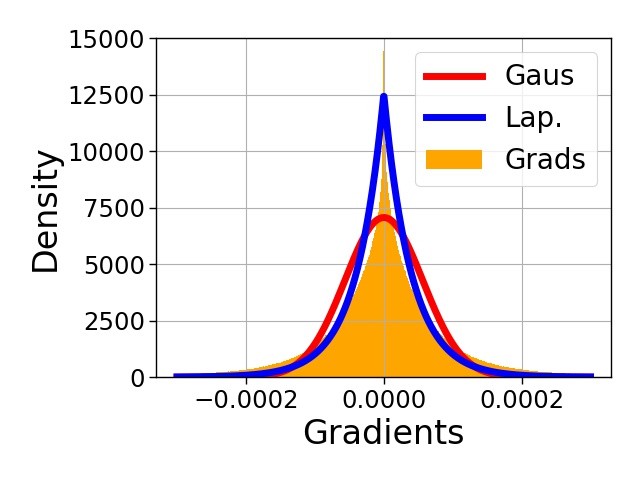}}
        \subfigure[Gradients at Mid Stage.]{\includegraphics[width=.31\textwidth]{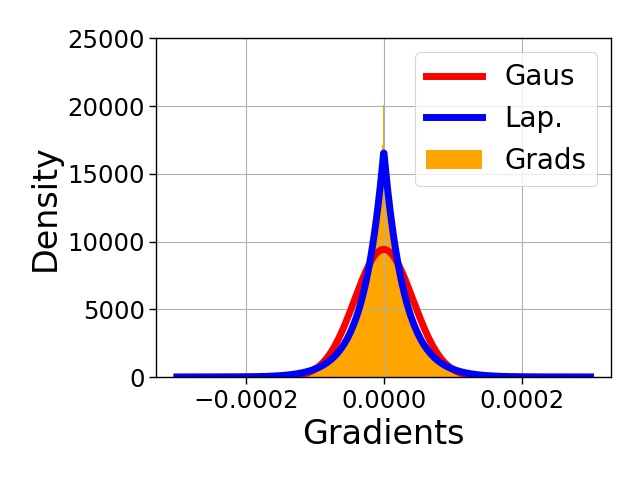}}
        \subfigure[Gradients at Late Stage.]{\includegraphics[width=.31\textwidth]{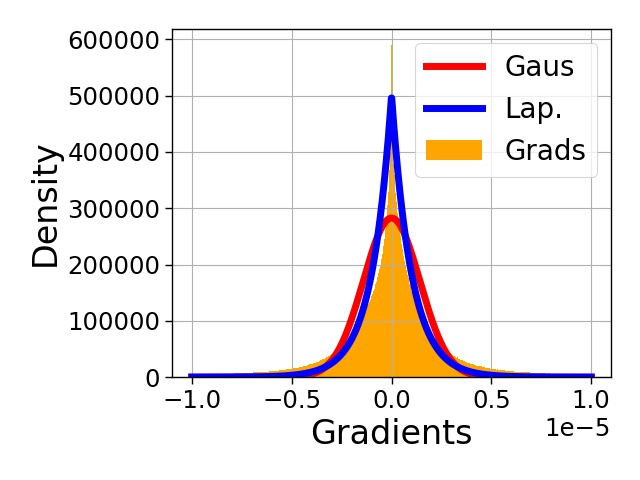}}
    \caption{Density of gradients of small VGG-16 trained on CIFAR-10 during (a) early stages of training (epoch 16), (b) middle stages of training (epoch 100), (c) late stages of training (epoch 170). }\label{fig:density_grad_vgg16_app}
   \end{figure*}

In Section~\ref{experiments}, we have applied SuRP to compress gradients in federated learning. In Figures~\ref{fig:density_grad_imagenet_app}, \ref{fig:density_grad_resnet18_app}, and \ref{fig:density_grad_vgg16_app}, we justify that Laplacian is a good fit for gradients of ResNet-50 trained on ImageNet, ResNet-18 trained on CIFAR-10, and VGG-16 trained on CIFAR-10. Since we need to compress the gradients before each communication round of federated training, SuRP requires the gradients to follow a Laplacian distribution throughout the learning process. In other words, although the parameter of the Laplacian distribution might change, we must be able to fit a Laplacian distribution to the gradients in every round.
We provide the density estimation of gradients in early-, mid-, and late-stages of training in Figures~\ref{fig:density_grad_imagenet_app}, \ref{fig:density_grad_resnet18_app}, and \ref{fig:density_grad_vgg16_app} to show that Laplacian distribution is a good fit for gradients starting from the early stages of training till the training ends. Therefore, we can apply SuRP to compress gradients at every communication round. In our experiments, we update the parameter of the Laplacian distribution ($\lambda$) at every communication round.  

Among other gradient sparsification methods for federated learning \citep{Aji, lin2017deep,federated0, Wang, Wangni}, SuRP is most similar to rTop-$k$ \citep{barnes2020rtop}, in that they also communicate a random subset of the large gradients. Different from our work, they approach the communication-efficient federated learning problem from a distributed statistical estimation point of view. By modeling the gradients with a sparse Bernoulli distribution, they show that the optimal compression strategy for each user (device) is to communicate a random $k/r$ fraction of the $r$ largest gradients. In contrast, we study the gradient compression problem from an information-theoretic approach and assume Laplacian distribution over the gradients. With this assumption, we conclude that each user must communicate the parameter of underlying Laplacian distribution of the local gradients and a list of indices that are randomly selected among the gradients larger than $\frac{1}{\lambda_t} \cdot \log{\frac{n}{2 \beta}}$ in magnitude at iteration $t$. Since the threshold $\frac{1}{\lambda_t} \cdot \log{\frac{n}{2 \beta}}$ is decreasing at each iteration, SuRP assigns a higher probability for larger (in magnitude) gradients to be selected, whereas rTop-$k$ picks the gradients uniformly random from the large gradients.

\subsection{Additional Experimental Details}
\label{hyperparams_appendix}

We conducted our experiments on NVIDIA Titan X (MNIST and CIFAR-10) and NVIDIA Titan Xp (ImageNet) GPUs on an internal cluster server. We used 1 GPU for MNIST and CIFAR-10 experiments and 2 GPUs for ImageNet experiments. We set the target sparsity of each SuRP round so that at each pruning iteration, $20\%$ of the surviving parameters will be pruned, e.g., sparsity schedule is as follows $20\%, 36\%, 48.8\%, 59.04\%, \dots$.

\subsubsection{MNIST:} 
We provide the architectural details and hyperparameters for LeNet-5 Caffe in Table~\ref{table:lenet-5_arch} \citep{lecun1998gradient}. We use a batch size of 100 and train for 100 epochs, early stopping at the best accuracy on validation set. We use the Adam optimizer with learning rate $=0.001$, and $\beta_1=0.9, \beta_2=0.999$ with weight decay $=5e^{-4}$.

\begin{table}[h!]
\centering
\caption{LeNet-5 Caffe convolutional architecture.}
\begin{tabular}{c|c}
\hline
\textbf{Name}& \textbf{Component}\\
\hline
conv1 & [$5 \times 5$ conv, 20 filters, stride 1], ReLU, $2 \times 2$ max pool \\
\hline
conv2 & [$5 \times 5$ conv, 50 filters, stride 1], ReLU, $2 \times 2$ max pool \\
\hline
Linear & Linear $800 \rightarrow 500$, ReLU \\
\hline
Output Layer & Linear $500 \rightarrow 10$ \\
\hline
\end{tabular}
\label{table:lenet-5_arch}
\end{table}

\subsubsection{CIFAR-10:} We provide the architectural details and hyperparameters for the ResNet-20 \citep{he2016deep}  and (small) VGG-16 \citep{vgg} in Tables~\ref{table:resnet_20_arch} and ~\ref{table:vgg16_arch}, respectively. For both ResNet-20 and VGG-16, we use a batch size of 128, we train ResNet-20 for 350 epochs and VGG-16 for 200 epochs, early stopping at the best accuracy on validation set. We use SGD with learning rate $=0.1$, momentum $=0.9$, and weight decay $=5e^{-4}$. We note that VGG-16 architecture is a smaller version of the original VGG architecture in \citep{vgg}. We retrain both models for $20$ epochs at the end of each pruning iteration.

\begin{table}[h!]
\centering
\caption{Slim ResNet-20 architecture.}
\begin{tabular}{c|c}
\hline
\textbf{Name}& \textbf{Component}\\
\hline
conv1 & $3\times3$ conv, 16 filters. stride 1, BatchNorm \\
\hline
Residual Block 1 & 
$
\begin{bmatrix}
    3 \times 3 \text{ conv, } 16 \text{ filters} \\
    3 \times 3 \text{ conv, } 16 \text{ filters}
\end{bmatrix}
\times 2$ \\
Residual Block 2 & 
$
\begin{bmatrix}
    3 \times 3 \text{ conv, } 32 \text{ filters} \\
    3 \times 3 \text{ conv, } 32 \text{ filters}
\end{bmatrix}
\times 2$ \\
\hline
Residual Block 3 & $
\begin{bmatrix}
    3 \times 3 \text{ conv, } 64 \text{ filters} \\
    3 \times 3 \text{ conv, } 64 \text{ filters}
\end{bmatrix}
\times 2$ \\
\hline
Output Layer & $7 \times 7$ average pool stride 1, fully-connected, softmax \\
\hline
\end{tabular}
\label{table:resnet_20_arch}
\end{table}

\begin{table}[h!]
\centering
\caption{VGG-16 architecture.}
\begin{tabular}{c|c}
\hline
\textbf{Name}& \textbf{Component}\\
\hline
conv1-2 & [$3\times3$ conv, 64 filters. stride 1, BatchNorm, ReLU] $\times 2$  \\
\hline
max pool & $2\times2$, stride 2  \\
\hline
conv3-4 & [$3\times3$ conv, 128 filters. stride 1, BatchNorm, ReLU] $\times 2$  \\
\hline
max pool & $2\times2$, stride 2  \\
\hline
conv5-7 & [$3\times3$ conv, 256 filters. stride 1, BatchNorm, ReLU] $\times 3$  \\
\hline
max pool & $2\times2$, stride 2  \\
\hline
conv8-10 & [$3\times3$ conv, 512 filters. stride 1, BatchNorm, ReLU] $\times 3$  \\
\hline
max pool & $2\times2$, stride 2  \\
\hline
conv11-13 & [$3\times3$ conv, 512 filters. stride 1, BatchNorm, ReLU] $\times 3$  \\
\hline
max pool & $2\times2$, stride 2  \\
\hline
Output Layer & $1 \times 1$ average pool stride 1, fully-connected, softmax \\
\hline
\end{tabular}
\label{table:vgg16_arch}
\end{table}

\subsubsection{ImageNet:}
We provide the architectural details and hyperparameters for the ResNet-50 used in our experiments in Table~\ref{table:resnet_arch3}. We use the pretrained ResNet-50 from PyTorch (\texttt{https://github.com/pytorch/vision/blob/master/torchvision/models/resnet.py}), with a batch size of 64. At the end of each pruning iteration, we retrain the model for 15 epochs. We use SGD with learning rate $=0.001$, momentum $=0.9$ and weight decay $=5e^{-4}$.

\begin{table}[h!]
\centering
\caption{ResNet-50 architecture.}
\begin{tabular}{c|c}
\hline
\textbf{Name}& \textbf{Component}\\
\hline
conv1 & $3\times3$ conv, 64 filters. stride 1, BatchNorm \\
\hline
Residual Block 1 & 
$
\begin{bmatrix}
    1 \times 1 \text{ conv, } 64 \text{ filters} \\
    3 \times 3 \text{ conv, } 64 \text{ filters} \\
    1 \times 1 \text{ conv, } 256 \text{ filters}
\end{bmatrix}
\times 3$ \\
\hline
Residual Block 2 & 
$
\begin{bmatrix}
    1 \times 1 \text{ conv, } 128 \text{ filters} \\
    3 \times 3 \text{ conv, } 128 \text{ filters} \\
    1 \times 1 \text{ conv, } 512 \text{ filters}
\end{bmatrix}
\times 4$ \\
\hline
Residual Block 3 & $
\begin{bmatrix}
    1 \times 1 \text{ conv, } 256 \text{ filters} \\
    3 \times 3 \text{ conv, } 256 \text{ filters} \\
    1 \times 1 \text{ conv, } 1024 \text{ filters}
\end{bmatrix}
\times 6$ \\
\hline
Residual Block 4 & $
\begin{bmatrix}
    1 \times 1 \text{ conv, } 512 \text{ filters} \\
    3 \times 3 \text{ conv, } 512 \text{ filters} \\
    1 \times 1 \text{ conv, } 2048 \text{ filters}
\end{bmatrix}
\times 3$ \\
\hline
Output Layer & $4 \times 4$ average pool stride 1, fully-connected, softmax \\
\hline
\end{tabular}
\label{table:resnet_arch3}
\end{table}

\subsection{Additional Experimental Results}
\label{experiments_app}
We give a more detailed version of Table~\ref{tab:experiment_cifar} in Tables~\ref{tab:experiment_cifar_vgg16_app},~\ref{tab:experiment_cifar_resnet20_app},~\ref{tab:experiment_cifar_densenet_app} and~\ref{tab:experiment_cifar_efficientnet_app} and a more detailed version of Table~\ref{tab:experiment_imagenet} in Table~\ref{tab:experiment_imagenet_app} with confidence intervals included in SuRP results.

\setlength{\tabcolsep}{3pt}
\begin{table*}[!h]
%\vspace{-5mm}
\centering
\caption{Accuracy of VGG-16 on CIFAR-10. Results are averaged over five runs. 
}
\resizebox{\textwidth}{!}{
\begin{tabular}{lccccccccccc}
\toprule
 \textbf{Pruning Ratio:}  & $93.12\%$   & $95.60\%$      & $97.19\%$    & $98.20\%$    & $98.85\%$  & $99.26\%$    & $99.53\%$   & $99.70\%$   & $99.81\%$ & $99.88\%$ \\ \midrule
Global \citep{morcos2019one} & $91.30$     & $90.80$        & $89.28$      & $85.55$      & $81.56$    & $54.58$      &$41.91$      & $31.93$     & $21.87$   & $11.72$  \\
Uniform \citep{zhu2017prune}                   & $91.47$     & $90.78$        & $88.61$      & $84.17$      & $55.68$    & $38.51$      &$26.41$      & $16.75$     & $11.58$   & $9.95$  \\ 
Adaptive \citep{stateofsparsity}    & $91.54$     & $91.20$        & $90.16$      & $89.44$      & $87.85$    & $86.53$      &$84.84$      & $82.41$     & $74.54$   & $24.46$  \\
RiGL \citep{evci2020rigging}          & $92.34$     & $91.99$        & $91.66$      & $91.15$      & $90.55$    & $89.51$      &$88.21$      & $86.73$     & $84.85$   & $81.50$  \\ 
LAMP \citep{lee2020deeper}                     & $92.24$     & $92.06$        & $91.71$      & $91.66$      & $91.07$    & $90.49$      &$89.64$      & $88.75$     & $87.07$   & $84.90$  \\
SuRP (ours)                & $\textbf{92.55} \pm 0.19$     & $\mathbf{92.13} \pm 0.20$        & $\mathbf{91.95} \pm 0.21$      & $\mathbf{91.72} \pm 0.28$      & $\mathbf{91.21} \pm 0.24$    & $\mathbf{90.73} \pm 0.21$      &$\mathbf{90.65} \pm 0.27$      & $\mathbf{89.70} \pm 0.32$     & $\mathbf{87.28} \pm 0.32$   & $\mathbf{85.04} \pm 0.35$ 
\\
\bottomrule
\\
\end{tabular}
}

\label{tab:experiment_cifar_vgg16_app}
\vspace{-2mm}
\end{table*}

\setlength{\tabcolsep}{3pt}
\begin{table*}[!h]
%\vspace{-5mm}
\centering
\caption{Accuracy of ResNet-20 on CIFAR-10. Results are averaged over five runs.
}
\resizebox{\textwidth}{!}{
\begin{tabular}{lccccccccccc}
\toprule
\textbf{Pruning Ratio:}  & $79.03\%$   & $86.58 \%$      & $91.41\%$   & $94.50 \%$    & $96.48\%$  & $97.75\%$    & $98.56\%$   & $99.08\%$   & $99.41\%$  & $99.62\%$ \\ \midrule
Global \citep{morcos2019one}                    &$87.48$      &$86.97$          & $86.29$     & $85.02$       & $83.15$    & $80.52$      & $76.28$     & $70.69$     & $47.47$    & $12.02$\\
Uniform \citep{zhu2017prune}                   &$87.24$      &$86.70$          & $86.09$     & $84.53$       & $82.05$    & $77.19$      & $64.24$     & $47.97$     & $20.45$    & $13.35$\\% 
Adaptive \citep{stateofsparsity}         &$87.30$      &$87.00$          & $86.27$     & $85.00$       & $83.23$    & $80.40$      & $76.40$     & $69.31$     & $52.06$    & $20.19$\\ 
RiGL \citep{evci2020rigging}          &$87.63$      &$87.49$          & $86.83$     & $85.84$       & $84.08$    & $81.76$      & $78.70$     & $74.40$     & $66.42$    & $50.90$\\
LAMP \citep{lee2020deeper}                     &$87.54$      &$87.12$          & $86.56$     & $85.64$       & $84.18$    & $81.56$      & $78.63$     & $74.20$     & $67.01$    & $51.24$\\
SuRP (ours)                &$\mathbf{91.37} \pm 0.24$      &$\mathbf{90.44} \pm 0.26$          & $\mathbf{89.00} \pm 0.21$     & $\mathbf{88.87} \pm 0.26$       & $\mathbf{87.05} \pm 0.28$    & $\mathbf{83.98} \pm 0.20$      & $\mathbf{79.00} \pm 0.34$     & $\mathbf{74.86} \pm 0.29$     & $\mathbf{70.64} \pm 0.38$    & $\mathbf{54.22} \pm 0.42$\\
\bottomrule
\\
\end{tabular}
}

\label{tab:experiment_cifar_resnet20_app}
\vspace{-2mm}
\end{table*}

\setlength{\tabcolsep}{3pt}
\begin{table*}[!h]
%\vspace{-5mm}
\centering
\caption{DenseNet-121 on CIFAR-10. Results are averaged over five runs.
}
\resizebox{\textwidth}{!}{
\begin{tabular}{lcccccccccc}
\toprule
  \textbf{Pruning Ratio:}  & $94.50\%$   & $95.60 \%$      & $96.48\%$   & $97.18 \%$    & $97.75\%$  & $98.20\%$    & $98.56\%$ & $98.85\%$   & $99.08\%$  & $99.26\%$  \\ \midrule
 Global \citep{morcos2019one} & $90.16$ & $89.52$ &  $88.83$  & $88.00$ & $86.85$ & $85.32$ & $77.68$ &  $45.30$ & $49.65$ & $20.96$ \\
 Unif. \citep{zhu2017prune}  &  $90.24$ &  $89.50$   &  $88.44$ & $87.94$ & $86.83$ &  $85.00$ & $82.16$ & $70.13$ & $66.46$ & $48.71$  \\
 Adap. \citep{stateofsparsity} & $90.25$  & $89.70$ & $89.03$  & $88.22$  & $87.40$  & $86.26$ & $84.55$ & $81.87$  & $69.25$ & $58.91$ \\
  RiGL \citep{evci2020rigging} & $90.21$ & $89.79$ &  $88.92$ & $88.20$  &  $87.25$ & $86.22$  & $84.11$ & $81.82$ & $59.06$ &  $59.07$ \\
  LAMP \citep{lee2020deeper} & $90.89$ & $90.11$  & $89.72$ & $89.12$ & $88.39$  & $87.75$ & $86.53$ & $85.13$ &  $82.92$ & $79.23$ \\
  SuRP (ours) & $\mathbf{91.42}\pm 0.11$ & $\mathbf{90.75}\pm 0.08$ & $\mathbf{90.30} \pm 0.20$ & $\mathbf{89.62} \pm 0.17$ & $\mathbf{88.77} \pm 0.08$ & $\mathbf{88.06} \pm 0.22$  & $\mathbf{86.71} \pm 0.15$  & $\mathbf{85.34} \pm 0.27$  & $\mathbf{83.18} \pm 0.24$ & $\mathbf{79.45} \pm 0.36$
\\
\bottomrule
\\
\end{tabular}
}
\label{tab:experiment_cifar_densenet_app}
\vspace{-2mm}
\end{table*}

\setlength{\tabcolsep}{3pt}
\begin{table*}[!h]
%\vspace{-5mm}
\centering
\caption{EfficientNet-B0 on CIFAR-10. Results are averaged over five runs.
}
\resizebox{\textwidth}{!}{
\begin{tabular}{lcccccccccc}
\toprule

\textbf{Pruning Ratio:}  & $59.00\%$   & $73.80 \%$      & $83.20\%$   & $89.30 \%$    & $93.13\%$  & $95.60\%$    & $97.18\%$   & $98.20\%$   & $98.85\%$  & $99.26\%$ \\ 
\midrule
Global \citep{morcos2019one} & $89.66$ & $89.55$ & $88.80$ & $87.64$ & $84.36$ & $79.25$ & $11.09$ & $10.62$ & $10.00$ & $10.00$ \\
Uniform \citep{zhu2017prune} & $88.99$ & $88.26$ & $86.48$ & $83.40$ &  $23.65$ & $10.83$ & $10.00$ & $10.00$ & $10.00$ & $10.00$ \\
Adaptive \citep{stateofsparsity} & $89.18$ & $88.03$ & $86.71$ & $84.16$ & $36.64$  &  $10.45$ & $10.00$ & $10.19$ & $10.00$ & $10.00$ \\
RiGL \citep{evci2020rigging} & $89.54$ &  $90.09$ & $90.01$ & $89.62$ & $88.82$ & $87.08$ & $84.72$ & $81.53$ & $51.31$ & $13.40$  \\
LAMP \citep{lee2020deeper} & $89.52$ & $89.95$ & $89.97$  & $90.21$ & $89.91$ & $89.79$ & $89.30$ & $88.51$ & $86.79$ & $65.76$ \\
SuRP (ours) & $\mathbf{90.96} \pm 0.10$ & $\mathbf{90.94} \pm 0.12$ & $\mathbf{90.89} \pm 0.12$ & $\mathbf{90.75} \pm 0.16$ & $\mathbf{90.31} \pm 0.21$ & $\mathbf{90.08} \pm 0.20$ &  $\mathbf{89.88} \pm 0.27$ & $\mathbf{89.02} \pm 0.38$ & $\mathbf{87.80} \pm 0.0.36$ & $\mathbf{70.76} \pm 0.52$
\\
\bottomrule
\\
\end{tabular}
}
\label{tab:experiment_cifar_efficientnet_app}
\vspace{-2mm}
\end{table*}

\setlength{\tabcolsep}{3pt}
\begin{table*}[!h]
%\vspace{-5mm}
\centering
\caption{Accuracy of ResNet-50 on ImageNet. Results are averaged over three runs.
}
\begin{tabular}{lcc}
\toprule
\textbf{Pruning Ratio:}             & $80\%$   & $90\%$  \\ \midrule
Adaptive \citep{stateofsparsity}   & $75.60$  & $73.90$   \\
SNIP \citep{snip}                   & $72.00$  & $67.20$ \\ 
DSR \citep{DSR2019}                 & $73.30$  & $71.60$ \\
SNFS \citep{SNFS2019}               & $74.90$  & $72.90$\\ 
RiGL \citep{evci2020rigging}        & $74.60$  & $72.00$\\ 
SuRP (ours)                          & $\mathbf{75.54} \pm 0.05$      & $\mathbf{73.93} \pm 0.04$    \\ \bottomrule
\\
\end{tabular}
\label{tab:experiment_imagenet_app}
\vspace{-2mm}
\end{table*}

 We also provide a comparison between SuRP and LAMP at lower pruning rates in Table~\ref{tab:lower_pruning}. 
 
\setlength{\tabcolsep}{3pt}
\begin{table}[!h]
\centering
\caption{Additional Results with Low Pruning Ratios.
}
\begin{tabular}{clcccccc}
\toprule
 \textbf{Pruning Ratio:} &  & $20\%$   & $36\%$      & $49\%$    & $59\%$    & $67\%$    & $79\%$\\ \midrule
\centered{VGG-16}& \centered{ LAMP  \citep{lee2020deeper}  \\ SuRP (ours) }
& \centered{ $93.12$ \\ $\textbf{93.72}$} 
& \centered{ $93.08$   \\ $\mathbf{93.75}$}
& \centered{ $93.05$ \\ $\mathbf{93.72}$ }
&\centered{ $92.89$ \\ $\mathbf{93.63}$ }
&\centered{ $92.81$ \\ $\mathbf{93.64}$}
& \centered{ $92.75$ \\$\mathbf{93.56}$ }
\\ \midrule
%&  \textbf{Pruning Ratio:}  & $20\%$   & $36\%$      & $49\%$    & $59\%$    & $67\%$    & $79\%$ \\ \midrule
\centered{ResNet-20} &\centered{ LAMP \citep{lee2020deeper}  \\SuRP (ours)}
& \centered{$89.12$ \\$\mathbf{92.47}$ }
&\centered{$88.81$  \\$\mathbf{92.43}$}
& \centered{$88.67$  \\ $\mathbf{92.29}$ }
& \centered{$88.27$ \\  $\mathbf{92.30}$}
& \centered{$87.95$  \\ $\mathbf{91.98}$}
& \centered{$87.54$  \\  $\mathbf{91.37}$}
\\
\bottomrule
\\
\end{tabular}
\label{tab:lower_pruning}
\end{table}

Additionally, we provide accuracy-bit rate comparisons between SuRP and relevant baselines such as Deep Comp. \citep{deep_compression}, DeepCABAC \citep{DeepCABAC}, DNS \citep{guo2016dynamic}, and SWS \citep{ullrich2017soft} in Table~\ref{tab:comparison_app}. It is seen from Table~\ref{tab:comparison_app} that SuRP outperforms the baselines both in terms of accuracy-sparsity and accuracy-bit rate tradeoffs. 

\setlength{\tabcolsep}{3pt}
\begin{table*}[h!]
%\vspace{-5mm}
\centering
\caption{Comparison of SuRP with other pruning strategies in terms of accuracy, sparsity and size (bit rate). 
}
%\resizebox{\columnwidth}{!}{
\begin{tabular}{cccccc}
\toprule
Model              & Original         & Method        & Sparsity                      & Comp.       &Comp.\\ 
(Original size)    & Acc. ($\%$)     &               &  $\frac{|w =0|}{|w|} (\%)$       & Size           & Acc. ($\%$)\\ \midrule
 \centered{
    LeNet-5-Caffe  \\
    MNIST\\
    (1.72 MB)}     &\centered{99.14} &\centered{Deep Comp. \citep{deep_compression}\\
                                        DNS \citep{guo2016dynamic} \\
                                        SWS \citep{ullrich2017soft}\\
                                        DeepCABAC \citep{DeepCABAC}\\
                                        SuRP (ours) \\
                                        SuRP (ours)}    & \centered{92.0 \\
                                                                    99.1 \\ 
                                                                    99.5 \\ 
                                                                    98.1 \\ 
                                                                    99.2 \\
                                                                    99.3}           & \centered{44 KB ($\times 39$)\\
                                                                                                16 KB ($\times 107$) \\
                                                                                                11 KB ($\times 156$) \\
                                                                                                12 KB ($\times 143$) \\
                                                                                                \textbf{7 KB ($\times 246$)}\\
                                                                                                \textbf{5 KB ($\times 344$)}}     & \centered{99.3\\
                                                                                                                                                99.1 \\
                                                                                                                                                99.0 \\
                                                                                                                                                99.1 \\
                                                                                                                                                99.3 ($\pm$ 0.0) \\
                                                                                                                                                98.2 ($\pm$ 0.1)}  \\ \midrule
\centered{ResNet-18 \\
CIFAR-10  \\
(44.70 MB) }        &\centered{95.60 }       &\centered{SuRP (ours) \\
                                                        SuRP (ours) \\
                                                        SuRP (ours)}    &\centered{90.0 \\
                                                                                    95.0 \\
                                                                                    97.0}        &\centered{3.1 MB ($\times 15$) \\
                                                                                                            1.1 MB ($\times 42$) \\
                                                                                                            \textbf{875 KB ($\times 53$)}}              &\centered{ 95.1 ($\pm$ 0.0)\\
                                                                                         92.2 ($\pm$ 0.1) \\
                                                                                        90.0 ($\pm$ 0.2) }\\ \midrule
\centered{Small VGG-16 \\
CIFAR-10 \\
(58.91 MB)}                &\centered{93.60}        &\centered{DeepCABAC \citep{DeepCABAC} \\
                                                             SuRP (ours)\\
                                                             SuRP (ours)}     &\centered{92.4 \\
                                                                                          95.0 \\
                                                                                          90.0 }   &\centered{956 KB ($\times 61$) \\
                                                                                                        \textbf{1.1 MB ($\times 54$)}\\
                                                                                                        3.0 MB  ($\times 20$)}                      &\centered{91.0 \\
                                                                                                                                                            92.4 ($\pm$ 0.1)\\
                                                                                                                                                            93.5 ($\pm$ 0.1)} \\ \midrule
\centered{ResNet-50 \\
        ImageNet \\
        (102.23 MB)  }         & \centered{76.60}     & \centered{Deep Comp. \citep{deep_compression} \\
                                                            DeepCABAC \citep{DeepCABAC}\\
                                                            SuRP (ours)}         & \centered{71.0 \\
                                                                                          74.6 \\
                                                                                          71.0}                   & \centered{6.1 MB ($\times 16$) \\
                                                                                                                                6.1 MB ($\times 16$) \\
                                                                                                                                 \textbf{6.1 MB ($\times 16$)}}           & \centered{76.1 \\
                                                                                                                                                                                74.1 \\
                                                                                                                                             
                                                                                                                                                                                76.4 ($\pm$ 0.0)} \\
\bottomrule
\\
\end{tabular}

\label{tab:comparison_app}
\end{table*}

\end{document}